\documentclass[10pt]{article}

\usepackage[utf8]{inputenc}
\usepackage[T1]{fontenc}

\usepackage{epsf}
\usepackage{amsmath}

\allowdisplaybreaks

\usepackage[showframe=false]{geometry}
\usepackage{changepage}

\usepackage{epsfig}
\usepackage{amssymb}

\usepackage{amsthm}
\usepackage{setspace}
\usepackage{cite}
\usepackage{mcite}

\usepackage{algorithmic}  % This package provides an algorithmic environment fo describing algorithms.
\usepackage{algorithm}

\usepackage{shadow}
\usepackage{fancybox}
\usepackage{fancyhdr}

\usepackage{color}
\usepackage[usenames,dvipsnames,svgnames,table]{xcolor}
\newcommand{\bl}[1]{\textcolor{blue}{#1}}
\newcommand{\red}[1]{\textcolor{red}{#1}}

\definecolor{mypurple}{rgb}{.4,.0,.5}

\usepackage[hyphens]{url}

\usepackage[colorlinks=true,
            linkcolor=black,
            urlcolor=blue,
            citecolor=purple]{hyperref}

\usepackage{breakurl}
\def\y{{\bf y}}

\def\x{{\bf x}}

\def\x{{\mathbf x}}

\def\u{{\bf u}}

\def\x{{\bf x}}
\def\y{{\bf y}}
\def\z{{\bf z}}
\def\q{{\bf q}}
\def\m{{\bf m}}

\def\b{{\bf b}}
\def\c{{\bf c}}
\def\d{{\bf d}}

\def\h{{\bf h}}

\def\cH{{\mathcal H}}

\def\be{\begin{equation}}
\def\ee{\end{equation}}
\def\ba{\left[\begin{array}}
\def\ea{\end{array}\right]}

\def\u{{\bf u}}

\def\x{{\bf x}}
\def\y{{\bf y}}
\def\z{{\bf z}}
\def\q{{\bf q}}

\def\b{{\bf b}}
\def\c{{\bf c}}
\def\d{{\bf d}}

\def\p{{\bf p}}

\def\1{{\bf 1}}

\def\G{{\bf G}}

\def\0{{\bf 0}}

\def\erf{\mbox{erf}}
\def\erfc{\mbox{erfc}}

\def\mR{{\mathbb R}}
\def\mN{{\mathbb N}}
\def\mE{{\mathbb E}}
\def\mS{{\mathbb S}}
\def\mB{{\mathbb B}}
\def\mP{{\mathbb P}}

\def\lp{\left (}
\def\rp{\right )}

\def\y{{\bf y}}

\def\x{{\bf x}}

\def\x{{\mathbf x}}

\def\u{{\bf u}}

\def\x{{\bf x}}
\def\y{{\bf y}}
\def\z{{\bf z}}
\def\q{{\bf q}}

\def\b{{\bf b}}
\def\c{{\bf c}}
\def\d{{\bf d}}

\def\h{{\bf h}}

\def\cH{{\cal H}}

\def\be{\begin{equation}}
\def\ee{\end{equation}}
\def\ba{\left[\begin{array}}
\def\ea{\end{array}\right]}

\def\u{{\bf u}}

\def\x{{\bf x}}
\def\y{{\bf y}}
\def\z{{\bf z}}
\def\q{{\bf q}}

\def\b{{\bf b}}
\def\c{{\bf c}}
\def\d{{\bf d}}

\def\p{{\bf p}}

\def\({\left (}
\def\){\right )}

\def\1{{\bf 1}}
\def\m{{\bf m}}
\def\q{{\bf q}}

\def\G{{\bf G}}

\def\0{{\bf 0}}

\def\cX{{\mathcal X}}
\def\cY{{\mathcal Y}}

\usepackage{xcolor}
\usepackage{color}

\definecolor{darkgreen}{rgb}{0, 0.4,0}

\definecolor{purplebrown}{rgb}{0.5,0.1,0.6}

\definecolor{ultclupcol}{rgb}{0.1,0.5,0.5}

\definecolor{mytrycolor}{rgb}{0.5,0.7,0.2}

\definecolor{ultclupcola}{rgb}{.5,0,.5}

\definecolor{shadebrown}{rgb}{0.1,0.1,0.9}
\definecolor{lightblue}{rgb}{0.2,0,1}

\usepackage{fancybox}
\usepackage{graphicx}
\usepackage{epstopdf}
\usepackage{epsfig}
\usepackage{wrapfig}
\usepackage{subfigure}

\usepackage{xcolor}
\usepackage{tcolorbox}
\newtcbox{\xmybox}{on line,
arc=7pt,
before upper={\rule[-3pt]{0pt}{10pt}},boxrule=0pt,
boxsep=0pt,left=6pt,right=6pt,top=0pt,bottom=0pt,enhanced, coltext=blue, colback=white!10!yellow}

\newtcbox{\xmyboxa}{on line,
arc=7pt,
before upper={\rule[-3pt]{0pt}{10pt}},boxrule=0pt,
boxsep=0pt,left=6pt,right=6pt,top=0pt,bottom=0pt,enhanced, colback=white!10!yellow}

\newtcbox{\xmyboxb}{on line,
arc=7pt,
before upper={\rule[-3pt]{0pt}{10pt}},boxrule=1pt,colframe=darkgreen!100!blue,
boxsep=0pt,left=6pt,right=6pt,top=0pt,bottom=0pt,enhanced, colback=white!10!yellow}

\newtcbox{\xmyboxc}{on line,
arc=7pt,
before upper={\rule[-3pt]{0pt}{10pt}},boxrule=.7pt,colframe=blue!100!blue,
boxsep=0pt,left=6pt,right=6pt,top=0pt,bottom=0pt,enhanced, coltext=blue, colback=white!10!yellow}

\newtcbox{\xmytboxa}{on line,
arc=7pt,
before upper={\rule[-3pt]{0pt}{10pt}},boxrule=.0pt,colframe=pink!50!yellow,
boxsep=0pt,left=6pt,right=6pt,top=0pt,bottom=0pt,enhanced, coltext=white, colback=blue!40!red}

\newtcbox{\xmytboxb}{on line,
arc=7pt,
before upper={\rule[-3pt]{0pt}{10pt}},boxrule=.0pt,colframe=pink!50!yellow,
boxsep=0pt,left=6pt,right=6pt,top=0pt,bottom=0pt,enhanced, coltext=white, colback=white!40!green}

\makeatletter
\newcommand\subsubsubsection{\@startsection{paragraph}{4}{\z@}{-2.5ex\@plus -1ex \@minus -.25ex}{1.25ex \@plus .25ex}{\normalfont\normalsize\bfseries}}
\newcommand\subsubsubsubsection{\@startsection{subparagraph}{5}{\z@}{-2.5ex\@plus -1ex \@minus -.25ex}{1.25ex \@plus .25ex}{\normalfont\normalsize\bfseries}}
\makeatother

\newtheorem{theorem}{Theorem}

\newtheorem{corollary}{Corollary}
\newtheorem{conjecture}{Conjecture}

\begin{document}

\begin{singlespace}

\title {Parametric RDT approach to computational gap of \emph{symmetric} binary perceptron %A tight variant of Gordon's escape through a mesh theorem
%\footnote{ This work was supported in
%part.}
}
\author{
\textsc{Mihailo Stojnic
\footnote{e-mail: {\tt flatoyer@gmail.com}}
%\footnote{e-mail: {\tt flatoyer@gmail.com}; Independent researcher, West Lafayette, IN 47906, USA}
}}
\date{}
\maketitle

%%%%%%%%%%%%%%%%%%%%%%%%%%%%%%%%%%%%%%%%%%%%%%%%%%%%%%%%%%%%%%%%%%%%%%%%%%%%%%%%
%%%%%%%%%%%%%%%%%%%%%%%%%%%%%%%%%%%%%%%%%%%%%%%%%%%%%%%%%%%%%%%%%%%%%%%%%%%%%%%%
\centerline{{\bf Abstract}} \vspace*{0.1in}
%%%%%%%%%%%%%%%%%%%%%%%%%%%%%%%%%%%%%%%%%%%%%%%%%%%%%%%%%%%%%%%%%%%%%%%%%%%%%%%%
%%%%%%%%%%%%%%%%%%%%%%%%%%%%%%%%%%%%%%%%%%%%%%%%%%%%%%%%%%%%%%%%%%%%%%%%%%%%%%%%

We study potential presence of statistical-computational gaps (SCG) in symmetric binary perceptrons (SBP) via a parametric utilization of  \emph{fully lifted random duality theory} (fl-RDT) \cite{Stojnicflrdt23}.  A structural change from decreasingly to arbitrarily ordered $\c$-sequence (a key fl-RDT parametric component) is observed  on the second lifting level and associated with \emph{satisfiability} ($\alpha_c$) -- \emph{algorithmic} ($\alpha_a$) constraints density threshold change thereby suggesting a potential existence of a nonzero computational gap $SCG=\alpha_c-\alpha_a$. The second level estimate is shown to match the theoretical $\alpha_c$ whereas the $r\rightarrow \infty$ level one is proposed to correspond to $\alpha_a$. For example, for the canonical SBP ($\kappa=1$ margin) we obtain $\alpha_c\approx 1.8159$ on the second and  $\alpha_a\approx 1.6021$ (with converging tendency towards $\sim 1.59$ range) on the seventh level. Our propositions remarkably well concur with recent literature: (\textbf{\emph{i}}) 
in \cite{Bald20} local entropy replica approach predicts $\alpha_{LE}\approx 1.58$ as the onset of clustering defragmentaion  (presumed driving force behind locally improving algorithms failures); \textbf{\emph{(ii)}} in $\alpha\rightarrow 0$ regime we obtain on the third lifting level $\kappa\approx 1.2385\sqrt{\frac{\alpha_a}{-\log\lp \alpha_a \rp}}$ which qualitatively matches overlap gap property (OGP) based predictions of \cite{GamKizPerXu22} and identically matches local entropy based predictions of \cite{BarbAKZ23}; \textbf{\emph{(iii)}} $\c$-sequence ordering change phenomenology mirrors the one observed in asymmetric binary perceptron (ABP) in \cite{Stojnicalgbp25} and the negative Hopfield model in  \cite{Stojniccluphop25}; and \textbf{\emph{(iv)}} as in \cite{Stojnicalgbp25,Stojniccluphop25}, we here design a CLuP based algorithm whose practical performance closely matches proposed theoretical predictions.

\vspace*{0.25in} \noindent {\bf Index Terms: Symmetric binary perceptrons; fl-RDT; Statistical-computational gaps}.

\end{singlespace}

%%%%%%%%%%%%%%%%%%%%%%%%%%%%%%%%%%%%%%%%%%%%%%%%%%%%%%%%%%%%%%%%%
%%%%%%%%%%%%%%%%%%%%%%%%%%%%%%%%%%%%%%%%%%%%%%%%%%%%%%%%%%%%%%%%%
%%%%%%%%%%%%%%%%%%%%%%%%%%%%%%%%%%%%%%%%%%%%%%%%%%%%%%%%%%%%%%%%%
%%%%%%%%%%%%%%%%%%%%%%%%%%%%%%%%%%%%%%%%%%%%%%%%%%%%%%%%%%%%%%%%%
%%%%%%%%%%%%%%%%%%%%%%%%%%%%%%%%%%%%%%%%%%%%%%%%%%%%%%%%%%%%%%%%%
%%%%%%%%%%%%%%%%%%%%%%%%%%%%%%%%%%%%%%%%%%%%%%%%%%%%%%%%%%%%%%%%%
\section{Introduction}
\label{sec:back}
%%%%%%%%%%%%%%%%%%%%%%%%%%%%%%%%%%%%%%%%%%%%%%%%%%%%%%%%%%%%%%%%%
%%%%%%%%%%%%%%%%%%%%%%%%%%%%%%%%%%%%%%%%%%%%%%%%%%%%%%%%%%%%%%%%%
%%%%%%%%%%%%%%%%%%%%%%%%%%%%%%%%%%%%%%%%%%%%%%%%%%%%%%%%%%%%%%%%%
%%%%%%%%%%%%%%%%%%%%%%%%%%%%%%%%%%%%%%%%%%%%%%%%%%%%%%%%%%%%%%%%%
%%%%%%%%%%%%%%%%%%%%%%%%%%%%%%%%%%%%%%%%%%%%%%%%%%%%%%%%%%%%%%%%%

Tremendous AI progress witnessed over the last two decades is in large part powered by machine learning (ML) and neural networks (NN) theoretical and algorithmic developments of several preceding decades. Classical perceptrons (\emph{binary} (BP) or \emph{spherical} (SP)) played a key role in these developments either as integral architectural NN components or as simple prototype models used to conceptually emulate more complex structures. As their sufficient simplicity and satisfactory generality  usually provide key advantages reagrding analytical tractability and faithful emulation of artificial reasoning, it is a no surprise that to this day they remain among the most studied  AI concepts.

While many preceptrons' features have been thoroughly considered, storage/classifying capacity is likely the most well known one. It represents critical data density $\alpha_c$ below/above which perceptron succeeds/fails to operate as a storing memory. Relevance of capacity and its analytical characterizations was recognized in the earliest pattern recognition days \cite{Wendel62,Winder,Cover65}. These works effectively established a strong foundation for understanding the importance of analytical considerations that followed in the ensuing decades. Many of them would go on to interconnect various seemingly distant fields ranging from logic, cognitive thinking, and psychology to optimization, information theory, and statistical physics. As a consequence, it is a no stretch to say that modern AI is unimaginable without a strong and scientifically diverse theoretical support.

%%%%%%%%%%%%%%%%%%%%%%%%%%%%%%%%%%%%%%%%%%%%%%%%%%%%%%%%%%%%%%%%%%%%%%%%%%%%%%%%
%%%%%%%%%%%%%%%%%%%%%%%%%%%%%%%%%%%%%%%%%%%%%%%%%%%%%%%%%%%%%%%%%%%%%%%%%%%%%%%%
%%%%%%%%%%%%%%%%%%%%%%%%%%%%%%%%%%%%%%%%%%%%%%%%%%%%%%%%%%%%%%%%%%%%%%%%%%%%%%%%
%%%%%%%%%%%%%%%%%%%%%%%%%%%%%%%%%%%%%%%%%%%%%%%%%%%%%%%%%%%%%%%%%%%%%%%%%%%%%%%%
\subsection{Theoretic vs algorithmic achievability }
\label{secanaldiff}
%%%%%%%%%%%%%%%%%%%%%%%%%%%%%%%%%%%%%%%%%%%%%%%%%%%%%%%%%%%%%%%%%%%%%%%%%%%%%%%%
%%%%%%%%%%%%%%%%%%%%%%%%%%%%%%%%%%%%%%%%%%%%%%%%%%%%%%%%%%%%%%%%%%%%%%%%%%%%%%%%
%%%%%%%%%%%%%%%%%%%%%%%%%%%%%%%%%%%%%%%%%%%%%%%%%%%%%%%%%%%%%%%%%%%%%%%%%%%%%%%%
%%%%%%%%%%%%%%%%%%%%%%%%%%%%%%%%%%%%%%%%%%%%%%%%%%%%%%%%%%%%%%%%%%%%%%%%%%%%%%%%

A large body of very influential work
\cite{Gar88,GarDer88,SchTir02,SchTir03,Tal05,Talbook11a,Talbook11b,StojnicGardGen13,StojnicGardSphNeg13} followed into the footsteps of the initial SP considerations \cite{Wendel62,Winder,Cover65}. Overall understanding deepened and focus often shifted to features and architectures  beyond capacities and simple perceptrons. The so-called  \emph{positive} spherical perceptrons (PSP),   became particularly well  understood. Analytical studies \cite{SchTir02,SchTir03,StojnicGardGen13} exploited  presence of strong deterministic duality/convexity and rigorously proved replica predictions \cite{Gar88,GarDer88} thereby significantly superseding pioneering works \cite{Wendel62,Winder,Cover65}. On the other hand, absence  of such features was typically perceived as  analytically unsurpassable obstacle on the path towards repeating success of \cite{SchTir02,SchTir03,StojnicGardGen13}. The negative spherical perceptron (NSP) distinguished itself as a prime example in that regard. A seemingly tiny change from positive to negative threshold disallows convexity advantage, makes analytical considerations much harder \cite{StojnicGardSphNeg13,FPSUZ17,FraHwaUrb19,FraPar16,FraSclUrb19,FraSclUrb20,AlaSel20,BMPZ23}, and all but ensures that more sophisticated approaches are needed \cite{Stojnicsflgscompyx23,Stojnicnflgscompyx23,Stojnicflrdt23,Stojnicnegsphflrdt23}.

Convexity and the strong deterministic duality are similarly absent in BPs as well. Simple replica symmetric (RS) predictions \cite{Gar88,GarDer88,StojnicDiscPercp13} do not hold for \emph{asymmetric} BP (ABP) and, more involved,  replica symmetry breaking (RSB) ones from \cite{KraMez89} are needed \cite{DingSun19,NakSun23,BoltNakSunXu22,Huang24,Stojnicbinperflrdt23}. On the other hand, when it comes to  \emph{symmetric} BP (SBP) of interest here things are a bit more interesting. While RS predictions again do not hold, remarkably favorable underlying combinatorics allows for fairly elegant analytical characterizations \cite{AbbLiSly21b,PerkXu21,AbbLiSly21a,AubPerZde19,GamKizPerXu22}.

The capacity notion mentioned above (and present in a majority of discussed papers) relates to theoretical achievability limits. Its value $\alpha_c$ is often called \emph{satisfiability threshold} and it represents the maximal data (constraints) density that one can achieve if unlimited computational resources are available. Whether one can approach those limits in a  computationally efficient manner is a different question. Thinking along these lines naturally motivates the introduction of \emph{algorithmic threshold}, $\alpha_a$, as a critical (maximal) constraints density achievable via computationally efficient methods (in other words, $\alpha_a$ is the maximal density for which perceptron's weights can be determined efficiently). While it is clear that  $\alpha_a\leq \alpha_c$ always holds, it is not obvious at all whether  $\alpha_a< \alpha_c$ in which case one says that there is a \emph{statistical-computational} gap (SCG). The gap size, $SCG=\alpha_c- \alpha_a$, in a way measures utilization of the perceptron's predicated power.

Determining the size of SCG is an extraordinary challenge. In classical \emph{worst-case} complexity theory ABP and SBP (or its discrepancy minimization equivalent) are NP problems \cite{Ama91} which likely implies existence of a gap and $\alpha_a< \alpha_c$. However,  as a \emph{worst case} concept,  NP-ness rarely properly reflects on \emph{typical} algorithmic solvability (without a doubt SCGs are way more reflective in that regard). Moreover, efficient ABP algorithms actually do exist in a large portion of $\alpha< \alpha_c$ range  \cite{BrZech06,BaldassiBBZ07,Hubara16,KimRoc98}. For example, while ABP's $\alpha_c\approx 0.8331$, the best available algorithms \cite{BaldassiBBZ07,Bald15,BMPZ23} suggest  $\alpha_{a}\approx 0.75 - 0.77$. With all pointers towards the existence of a gap, it should be noted that it is way more favorable than the one NP theory predicts (examples of other problems (including planted ones) with similar algorithmic implications can be found in, e.g., \cite{MMZ05,GamarSud14,GamarSud17,GamarSud17a,AchlioptasR06,AchlioptasCR11,GamMZ22,BarbierKMMZ18,KMSSZ12a}). For SBPs things are a bit different. In the most well studied $\alpha\rightarrow 0$ regime the best known algorithms \cite{BanSpen20} are in the range $\kappa\sim\sqrt{\alpha}$ which is way below the corresponding $\alpha_c$ scaling.

%%%%%%%%%%%%%%%%%%%%%%%%%%%%%%%%%%%%%%%%%%%%%%%%%%%%%%%%%%%%%%%%%%%%%%%%%%%%%%%%
%%%%%%%%%%%%%%%%%%%%%%%%%%%%%%%%%%%%%%%%%%%%%%%%%%%%%%%%%%%%%%%%%%%%%%%%%%%%%%%%
%%%%%%%%%%%%%%%%%%%%%%%%%%%%%%%%%%%%%%%%%%%%%%%%%%%%%%%%%%%%%%%%%%%%%%%%%%%%%%%%
%%%%%%%%%%%%%%%%%%%%%%%%%%%%%%%%%%%%%%%%%%%%%%%%%%%%%%%%%%%%%%%%%%%%%%%%%%%%%%%%
\subsection{Prior work and contextualization of our contributions}
\label{sec:examples}
%%%%%%%%%%%%%%%%%%%%%%%%%%%%%%%%%%%%%%%%%%%%%%%%%%%%%%%%%%%%%%%%%%%%%%%%%%%%%%%%
%%%%%%%%%%%%%%%%%%%%%%%%%%%%%%%%%%%%%%%%%%%%%%%%%%%%%%%%%%%%%%%%%%%%%%%%%%%%%%%%
%%%%%%%%%%%%%%%%%%%%%%%%%%%%%%%%%%%%%%%%%%%%%%%%%%%%%%%%%%%%%%%%%%%%%%%%%%%%%%%%
%%%%%%%%%%%%%%%%%%%%%%%%%%%%%%%%%%%%%%%%%%%%%%%%%%%%%%%%%%%%%%%%%%%%%%%%%%%%%%%%

Despite extensive studies across a variety of scientific fields over the last two decades computational gaps are still a mystery. Strong progress has been made on many particular problems but a generic resolution remains unreachable. We here briefly review two approaches that gained a lot of attention in recent years. They are focused on typical/atypical solutions clustering and allowed for many excellent results that are of independent value as well. The first approach relates to \emph{Overlap gap property} (OGP) \cite{Gamar21,GamarSud14,GamarSud17,GamarSud17a,AchlioptasCR11,HMMZ08,MMZ05} and the second one to \emph{Local entropy} (LE) \cite{Bald15,Bald16,Bald20,BarbAKZ23,Stojnicabple25}.

The OGP approach \cite{Gamar21,GamarSud14,GamarSud17,GamarSud17a,AchlioptasCR11,HMMZ08,MMZ05} proposes a connection between  algorithmic efficiency and gaps in the spectrum of attainable solutions (near-solutions) overlaps. It generically postulates that absence of such gaps implies existence of efficient algorithms and views $\alpha_a$ as the maximal $\alpha$ which ensures OGP's absence. For example, for SBP the OGP's presence extends well below $\alpha_c$ \cite{GamKizPerXu22,Bald20} (for analogous discrepancy minimization results see \cite{GamKizPerXu23}). Provided that  OGP is indeed algorithmically relevant, this would strongly suggest that SCG indeed exists. However, a couple of things should be kept in mind regarding OGP: (\textbf{\emph{i}}) the generic OGP hardness implications are disproved via the shortest path counterexample \cite{LiSch24} (earlier simple algebraic disproving examples were viewed as exceptions); and (\textbf{\emph{ii}})  \cite{LiSch24} does not disprove OGP's relevance for other problems or specific algorithms. Probably the best hard example discovered so far where OGP concepts might be in full power is the famous $2$-spin Ising Sherrington-Kirkpatrick (SK) model \cite{SheKir72}. As Monatanari showed in \cite{Montanari19},  a widely believed absence of OGP directly implies polynomial solvability of this model (for corresponding $p$-spin considerations see, e.g.,  \cite{AlaouiMS22,AlaouiMS21}; for earlier spherical SK models  related results see, e.g.,  \cite{Subag17,Subag17a,Subag21,Subag24} and for closely connected martingale based reversals of Parisi functional, see, e.g. \cite{JCM25}; for corresponding NSP discussions see, e.g., \cite{AlaSel20,AMZ24}; and for importance of more sophisticated OGPs see, e.g., \cite{Kiz23,HuangS22}). While at present the OGP's role in generic algorithmic hardness remains undetermined, its presence disallows efficient implementations for many specific algorithmic classes  \cite{GamKizPerXu22}. Moreover,  for many problems \cite{RahVir17,GamarSud14,GamarSud17,GamAW24,Wein22},  practical algorithms exist in $\alpha$ ranges where OGP is absent.

Differently from OGP, \cite{Huang13,Huang14} connect  algorithmic hardness clustering  relevance to entropies of \emph{typical} solutions. A completely \emph{frozen} isolation of typical solutions is predicated (and proven for SBP in \cite{PerkXu21,AbbLiSly21a,AbbLiSly21b}).  \cite{Bald15,Bald16,Bald20}  propose a \emph{local entropy} (LE) approach to study  \emph{atypical} well-connected clusters. Even if predominant typical solutions are disconnected  (and likely unreachable via local searches) \cite{Huang13,Huang14,PerkXu21,AbbLiSly21b}, one may still have rare (atypical) well-connected clusters. Those rare clusters are then predicated as precisely those found by the efficient algorithms (for an SBP's sampling type of justification seemingly along these lines, see \cite{ElAlGam24}). Provided that such a pictorial portrayal is correct, a direct correlation between the existence of SCGs and properties of rare clusters is then very likely. Moreover, \cite{Bald15,Bald16,Bald20} also speculate that LE features (negativity, monotonicity, and breakdown) might be reflections of rare clusters' structures impact on algorithmic hardness. Such a phenomenology is further supported by results of \cite{AbbLiSly21a} where SBPs maximal diameter clusters are shown to exist with sufficiently small $\alpha$. Moreover, \cite{AbbLiSly21a} also showed that similar clusters  (albeit of linear diameter)  exist for any $\alpha<\alpha_c$ (modulo additional technical assumptions, \cite{AbbLiSly21a}'s SBP results do translate to ABP as well). Reconnecting LE back to OGP, one notes that the small $\alpha$ SBP LE results (both rigorous contiguous and 1RSB ones) are shown in \cite{BarbAKZ23} to scaling-wise match \cite{GamKizPerXu22}'s OGP predictions  (modulo a log term, they also correspond to the best known algorithmic once achieved in  \cite{BanSpen20}). Such a nice OGP -- LE correspondence seems fairly reassuring and might  present steps in the right direction towards demystifying presumable connection between these phenomena and SCGs.

%%%%%%%%%%%%%%%%%%%%%%%%%%%%%%%%%%%%%%%%%%%%%%%%%%%%%%%%%%%%%%%%%%%%%%%%%%%%%%%%
%%%%%%%%%%%%%%%%%%%%%%%%%%%%%%%%%%%%%%%%%%%%%%%%%%%%%%%%%%%%%%%%%%%%%%%%%%%%%%%%
%%%%%%%%%%%%%%%%%%%%%%%%%%%%%%%%%%%%%%%%%%%%%%%%%%%%%%%%%%%%%%%%%%%%%%%%%%%%%%%%
\subsubsection{Our contributions}
\label{sec:cont}
%%%%%%%%%%%%%%%%%%%%%%%%%%%%%%%%%%%%%%%%%%%%%%%%%%%%%%%%%%%%%%%%%%%%%%%%%%%%%%%%
%%%%%%%%%%%%%%%%%%%%%%%%%%%%%%%%%%%%%%%%%%%%%%%%%%%%%%%%%%%%%%%%%%%%%%%%%%%%%%%%
%%%%%%%%%%%%%%%%%%%%%%%%%%%%%%%%%%%%%%%%%%%%%%%%%%%%%%%%%%%%%%%%%%%%%%%%%%%%%%%%

We focus on a different approach and study potential presence of SCG  via a parametric utilization of  \emph{fully lifted random duality theory} (fl-RDT) \cite{Stojnicflrdt23}. A key fl-RDT parametric component -- so-called $\c$-sequence -- is observed to exhibit a structural ordering change after the second lifting level with a perfect (natural/physical) decreasing order abruptly disappearing. This is then connected to a change from $\alpha_c$ to $\alpha_a$ and a potential existence of a nonzero computational gap. The second level estimate turns out to identically match the theoretical $\alpha_c$ whereas the corresponding $r\rightarrow \infty$ level one is predicated to correspond to $\alpha_a$. While such a proposition extends to any margin $\kappa$,  canonical $\kappa=1$ margin (typically singled out in prior literature as a convenient benchmark) allows to make numerically concrete observations. Namely, in $\kappa=1$ scenario, our proposition produces  $\alpha_c\approx 1.8159$ on the second lifting level  which matches the theoretical \emph{satisfiability threshold}  (the matching actually extends to any $\kappa$). Also, on the seventh level we obtain $\alpha_a\approx 1.6021$ and observe a convergence tendency with estimate for $r\rightarrow \infty$ level $\alpha_a\sim 1.59-1.60$. This, on the other hand,  closely approaches   local entropy replica prediction $\alpha_{LE}\approx 1.58$ for clustering defragmentation postulated to be directly related to \emph{algorithmic threshold}  \cite{Bald20}. Moreover, in $\alpha\rightarrow 0$ regime we find on the third lifting level $\kappa\approx 1.2385\sqrt{\frac{\alpha_a}{-\log\lp \alpha_a \rp}}$ which scaling-wise matches OGP based predictions of \cite{GamKizPerXu22} and identically (constant-wise) matches local entropy based predictions of \cite{BarbAKZ23}. In other words, we further reconfirm the above mentioned OGP-LE connection between \cite{GamKizPerXu22} and \cite{BarbAKZ23}. Finally, the very same $\c$-sequence ordering change phenomenology is observed in ABP and negative Hopfield models in \cite{Stojnicalgbp25,Stojniccluphop25} suggesting both (\textbf{\emph{i}}) a potentially universal power of the proposed methodology possibly utilizable in other problems and (\textbf{\emph{ii}}) a likely presence of a generic underlying mechanism that drives the whole story. Adequate conjectures aligned with these observations are formulated as well. Finally, analytical predictions are complemented with the design of a CLuP (\emph{controlled loosening-up}) based algorithm whose practical performance remarkably closely approaches proposed theoretical results. This is further observed to conveniently match \cite{Stojnicalgbp25,Stojniccluphop25} where CLuP like algorithms with performances closely mimicking the theoretical ones were designed as well.

%%%%%%%%%%%%%%%%%%%%%%%%%%%%%%%%%%%%%%%%%%%%%%%%%%%%%%%%%%%%%%%%%%%%%%%%%%%%%%%%%%%%%%%%%%%%%%%%%%%%%%%%%%%%%%%%%%%%%%%
%%%%%%%%%%%%%%%%%%%%%%%%%%%%%%%%%%%%%%%%%%%%%%%%%%%%%%%%%%%%%%%%%%%%%%%%%%%%%%%%%%%%%%%%%%%%%%%%%%%%%%%%%%%%%%%%%%%%%%%
%%%%%%%%%%%%%%%%%%%%%%%%%%%%%%%%%%%%%%%%%%%%%%%%%%%%%%%%%%%%%%%%%%%%%%%%%%%%%%%%%%%%%%%%%%%%%%%%%%%%%%%%%%%%%%%%%%%%%%%
%%%%%%%%%%%%%%%%%%%%%%%%%%%%%%%%%%%%%%%%%%%%%%%%%%%%%%%%%%%%%%%%%%%%%%%%%%%%%%%%%%%%%%%%%%%%%%%%%%%%%%%%%%%%%%%%%%%%%%%
%%%%%%%%%%%%%%%%%%%%%%%%%%%%%%%%%%%%%%%%%%%%%%%%%%%%%%%%%%%%%%%%%%%%%%%%%%%%%%%%%%%%%%%%%%%%%%%%%%%%%%%%%%%%%%%%%%%%%%%
%%%%%%%%%%%%%%%%%%%%%%%%%%%%%%%%%%%%%%%%%%%%%%%%%%%%%%%%%%%%%%%%%%%%%%%%%%%%%%%%%%%%%%%%%%%%%%%%%%%%%%%%%%%%%%%%%%%%%%%
\section{SBP -- mathematical formulations and capacity}
 \label{sec:bprfps}
%%%%%%%%%%%%%%%%%%%%%%%%%%%%%%%%%%%%%%%%%%%%%%%%%%%%%%%%%%%%%%%%%%%%%%%%%%%%%%%%%%%%%%%%%%%%%%%%%%%%%%%%%%%%%%%%%%%%%%%
%%%%%%%%%%%%%%%%%%%%%%%%%%%%%%%%%%%%%%%%%%%%%%%%%%%%%%%%%%%%%%%%%%%%%%%%%%%%%%%%%%%%%%%%%%%%%%%%%%%%%%%%%%%%%%%%%%%%%%%
%%%%%%%%%%%%%%%%%%%%%%%%%%%%%%%%%%%%%%%%%%%%%%%%%%%%%%%%%%%%%%%%%%%%%%%%%%%%%%%%%%%%%%%%%%%%%%%%%%%%%%%%%%%%%%%%%%%%%%%
%%%%%%%%%%%%%%%%%%%%%%%%%%%%%%%%%%%%%%%%%%%%%%%%%%%%%%%%%%%%%%%%%%%%%%%%%%%%%%%%%%%%%%%%%%%%%%%%%%%%%%%%%%%%%%%%%%%%%%%
%%%%%%%%%%%%%%%%%%%%%%%%%%%%%%%%%%%%%%%%%%%%%%%%%%%%%%%%%%%%%%%%%%%%%%%%%%%%%%%%%%%%%%%%%%%%%%%%%%%%%%%%%%%%%%%%%%%%%%%
%%%%%%%%%%%%%%%%%%%%%%%%%%%%%%%%%%%%%%%%%%%%%%%%%%%%%%%%%%%%%%%%%%%%%%%%%%%%%%%%%%%%%%%%%%%%%%%%%%%%%%%%%%%%%%%%%%%%%%%

%%%%%%%%%%%%%%%%%%%%%%%%%%%%%%%%%%%%%%%%%%%%%%%%%%%%%%%%%%%%%%%%%%%%%%%%%%%%%%%%%%%%%%%%%%%%%%%%%%%%%%%%%%%%%%%%%%%%%%%
%%%%%%%%%%%%%%%%%%%%%%%%%%%%%%%%%%%%%%%%%%%%%%%%%%%%%%%%%%%%%%%%%%%%%%%%%%%%%%%%%%%%%%%%%%%%%%%%%%%%%%%%%%%%%%%%%%%%%%%
%%%%%%%%%%%%%%%%%%%%%%%%%%%%%%%%%%%%%%%%%%%%%%%%%%%%%%%%%%%%%%%%%%%%%%%%%%%%%%%%%%%%%%%%%%%%%%%%%%%%%%%%%%%%%%%%%%%%%%%
%%%%%%%%%%%%%%%%%%%%%%%%%%%%%%%%%%%%%%%%%%%%%%%%%%%%%%%%%%%%%%%%%%%%%%%%%%%%%%%%%%%%%%%%%%%%%%%%%%%%%%%%%%%%%%%%%%%%%%%
%%%%%%%%%%%%%%%%%%%%%%%%%%%%%%%%%%%%%%%%%%%%%%%%%%%%%%%%%%%%%%%%%%%%%%%%%%%%%%%%%%%%%%%%%%%%%%%%%%%%%%%%%%%%%%%%%%%%%%%
%%%%%%%%%%%%%%%%%%%%%%%%%%%%%%%%%%%%%%%%%%%%%%%%%%%%%%%%%%%%%%%%%%%%%%%%%%%%%%%%%%%%%%%%%%%%%%%%%%%%%%%%%%%%%%%%%%%%%%%
\subsection{SBP as a feasibility, optimization, and free energy problem}
 \label{sec:bprfps}
%%%%%%%%%%%%%%%%%%%%%%%%%%%%%%%%%%%%%%%%%%%%%%%%%%%%%%%%%%%%%%%%%%%%%%%%%%%%%%%%%%%%%%%%%%%%%%%%%%%%%%%%%%%%%%%%%%%%%%%
%%%%%%%%%%%%%%%%%%%%%%%%%%%%%%%%%%%%%%%%%%%%%%%%%%%%%%%%%%%%%%%%%%%%%%%%%%%%%%%%%%%%%%%%%%%%%%%%%%%%%%%%%%%%%%%%%%%%%%%
%%%%%%%%%%%%%%%%%%%%%%%%%%%%%%%%%%%%%%%%%%%%%%%%%%%%%%%%%%%%%%%%%%%%%%%%%%%%%%%%%%%%%%%%%%%%%%%%%%%%%%%%%%%%%%%%%%%%%%%
%%%%%%%%%%%%%%%%%%%%%%%%%%%%%%%%%%%%%%%%%%%%%%%%%%%%%%%%%%%%%%%%%%%%%%%%%%%%%%%%%%%%%%%%%%%%%%%%%%%%%%%%%%%%%%%%%%%%%%%
%%%%%%%%%%%%%%%%%%%%%%%%%%%%%%%%%%%%%%%%%%%%%%%%%%%%%%%%%%%%%%%%%%%%%%%%%%%%%%%%%%%%%%%%%%%%%%%%%%%%%%%%%%%%%%%%%%%%%%%
%%%%%%%%%%%%%%%%%%%%%%%%%%%%%%%%%%%%%%%%%%%%%%%%%%%%%%%%%%%%%%%%%%%%%%%%%%%%%%%%%%%%%%%%%%%%%%%%%%%%%%%%%%%%%%%%%%%%%%%

For two positive integers $m$ and $n$, $\alpha=\frac{m}{n}$, $G\in\mR^{n\times n}$, $\b\in\mR^{m\times 1}$, and $\cX\in\mR^n$, we consider the following \emph{feasibility} problems
\begin{eqnarray}
\hspace{-1.5in}\mbox{$\mathbf{\mathcal F}(G,\b,\cX,\alpha)$:} \hspace{1in}\mbox{find} & & \x\nonumber \\
\mbox{subject to}
& & G\x\geq \b \nonumber \\
& & \x\in\cX. \label{eq:ex1}
\end{eqnarray}
Throughout paper, we operate in large dimensional \emph{linear/proportional} regime with constraint density $\alpha= \lim_{n\rightarrow\infty} \frac{m}{n}$ remaining constant as $m$ and $n$ grow. Mathematical representations of many of the perceptron types mentioned earlier can be deduced as special cases of (\ref{eq:ex1}). For example, taking $\cX=\{\x | \| \x \|_2=1\} \triangleq \mS^m$ gives PSP when $\b\geq 0$  \cite{StojnicGardGen13,GarDer88,Gar88,Schlafli,Cover65,Winder,Winder61,Wendel62,Cameron60,Joseph60,BalVen87,Ven86,SchTir02,SchTir03} and NSP when $\b< 0$ \cite{AMZ24,BMPZ23,FPSUZ17,Talbook11a,FraHwaUrb19,FraPar16,FraSclUrb19,FraSclUrb20,AlaSel20,StojnicGardSphNeg13,Stojnicnegsphflrdt23}. Constraining $\cX$ to a binary cube, i.e., taking $\cX=\left \{-\frac{1}{\sqrt{n}},\frac{1}{\sqrt{n}} \right \}^n \triangleq \mB^n$ gives ABP \cite{Talbook11a,StojnicGardGen13,GarDer88,Gar88,StojnicDiscPercp13,KraMez89,GutSte90,KimRoc98,NakSun23,BoltNakSunXu22,PerkXu21,CXu21,DingSun19,Huang24,Stojnicbinperflrdt23,LiSZ24}.

Of our interest here are SBPs \cite{AubPerZde19,AbbLiSly21a,AbbLiSly21b,Bald20,GamKizPerXu22,PerkXu21,ElAlGam24,SahSaw23,Barb24,djalt22,BarbAKZ23}. They also follow as a specialization of (\ref{eq:ex1}) after one sets $\cX=\mB^n$ and additionally remodels linear constraints so that they become $|G\x |\leq \b$. This formally  renders the following feasibility variant of (\ref{eq:ex1})  
\begin{eqnarray}
\hspace{-1.5in}\mbox{$\mathbf{\mathcal S}(G,\b,\alpha)$:} \hspace{1in}\mbox{find} & & \x\nonumber \\
\mbox{subject to}
& & |G\x|\leq \b \nonumber \\
& & \x\in \left \{-\frac{1}{\sqrt{n}},\frac{1}{\sqrt{n}} \right \}^n \triangleq \mB^n, \label{eq:ex1a0a0}
\end{eqnarray}
which represents a generic (non-uniform) threshold SBP. Specializing further to $\b=\kappa \1$ (with $\1$ being the $m$-dimensional  column vector of all ones) gives the traditional (fixed/uniform threshold) SBP formulation
\begin{eqnarray}
\hspace{-1.5in}\mbox{$\mathbf{\mathcal S}(G,\kappa,\alpha)$:} \hspace{1in}\mbox{find} & & \x\nonumber \\
\mbox{subject to}
& & |G\x|\leq \kappa \1 \nonumber \\
& & \x\in \left \{-\frac{1}{\sqrt{n}},\frac{1}{\sqrt{n}} \right \}^n \triangleq \mB^n. \label{eq:ex1a0}
\end{eqnarray}
Depending on whether $G$ is random or not one has statistical or deterministic perceptrons. Our focus will be on the classical Gaussian perceptrons where the components of $G$ are independent standard normals. We should also add that the following, so-called  discrepancy minimization problems \cite{KaKLO86,Spen85,LovMek15,GamKizPerXu23,Roth17,AlwLiuSaw21},  typically encountered in computer science literature, are often taken as alternative SBP formulations
\begin{eqnarray}
\hspace{-1.5in}\mbox{$\mathbf{\mathcal D}(G,\alpha)$:} \hspace{1in}\min_{\x} & & |\G\x| \nonumber \\
\mbox{subject to}
 & & \x\in \left \{-\frac{1}{\sqrt{n}},\frac{1}{\sqrt{n}} \right \}^n \triangleq \mB^n. \label{eq:ex1a1}
\end{eqnarray}
From the optimization point of view problems (\ref{eq:ex1a0}) and (\ref{eq:ex1a1})  are fundamentally different. The former is a feasibility problem whereas the latter is a standard objective minimization. Nonetheless, $\mathbf{\mathcal D}(G,\alpha)$ to a large degree emulates $\mathbf{\mathcal S}(G,\kappa,\alpha)$ and can be used to solve it. Namely, the optimal objective value of  (\ref{eq:ex1a1})  is the minimal $\kappa$ for which $\mathbf{\mathcal S}(G,\kappa,\alpha)$ is feasible. This basically means that to check feasibility of  (\ref{eq:ex1a0}) it is sufficient to solve  (\ref{eq:ex1a1}) and check whether $\kappa$ in  (\ref{eq:ex1a0}) upper bounds the obtained objective optimum. 

In what follows we focus on traditional form (\ref{eq:ex1a0}). After rewriting it as 
\begin{eqnarray}
\hspace{-1.5in}\mbox{$\mathbf{\mathcal S}(G,\kappa,\alpha)$:} \hspace{1in}\mbox{find} & & \x\nonumber \\
\mbox{subject to}
& & G\x = \z \nonumber \\
& & \|\z \|_{\infty} \leq \kappa \nonumber \\
& & \x\in \left \{-\frac{1}{\sqrt{n}},\frac{1}{\sqrt{n}} \right \}^n \triangleq \mB^n, \label{eq:ex1a2}
\end{eqnarray}
we associate to it the following 
\begin{eqnarray}
\xi_{SBP}
& \triangleq  &
 \min_{\x\in \mB^n, \| \z \|_{\infty} \leq \kappa} \max_{\y\in\mS_+^m}  \lp \y^TG\x - \kappa \y^T\1 \rp,
 \label{eq:ex3}
\end{eqnarray}
where $\mS_+^m$ is the  positive orthant part of the $m$-dimensional unit sphere  (i.e., $\mS_+^m=\{\y|\|\y\|_2=1,\y\geq 0\}$). One 
then has 
\begin{eqnarray}
  \mathbf{\mathcal S}(G,\kappa,\alpha)  \mbox{ is infeasible }  \quad \Longleftrightarrow   \quad\xi_{SBP}>0 . \label{eq:ex1a3a0}
\end{eqnarray}
This further implies that for all practical purposes  the optimal objective seeking (\ref{eq:ex3}) is basically equivalent to  and can replace  $\mathbf{\mathcal S}(G,\kappa,\alpha)$.

The above feasibility and standard optimization formulations can also be incorporated into a free energy context. Consider   Hamiltonian
\begin{equation}
\cH(G)= \y^TG\x,\label{eq:ham1}
\end{equation}
and associated (virtual) partition function
\begin{equation}
Z(\beta,G)=\sum_{\x\in\cX} \lp \sum_{\y\in\cY}e^{\beta \lp  \cH(G) + f(\y)  \rp }\rp^{-1},  \label{eq:partfun}
\end{equation}
with (for the time being) general sets $\cX$ and $\cY$. Let the corresponding average thermodynamic limit (``\emph{reciprocal}'') free energy be
\begin{eqnarray}
f_{sq}(\beta) & = & \lim_{n\rightarrow\infty}\frac{\mE_G\log{(Z(\beta,G)})}{\beta \sqrt{n}}
=\lim_{n\rightarrow\infty} \frac{\mE_G\log\lp \sum_{\x\in\cX} \lp \sum_{\y\in\cY}e^{\beta \lp   \cH(G)  + f(\y) \rp   }\rp^{-1}\rp}{\beta \sqrt{n}} \nonumber \\
& = &\lim_{n\rightarrow\infty} \frac{\mE_G\log\lp \sum_{\x\in\cX} \lp \sum_{\y\in\cY}e^{\beta \lp   \y^TG\x   + f(\y) \rp  )}\rp^{-1}\rp}{\beta \sqrt{n}}.\label{eq:logpartfunsqrt}
\end{eqnarray}
Specialization $\beta\rightarrow \infty$ gives the ground state energy (GSE)
\begin{eqnarray}
f_{sq}(\infty)   \triangleq    \lim_{\beta\rightarrow\infty}f_{sq}(\beta) & = &
\lim_{\beta,n\rightarrow\infty}\frac{\mE_G\log{(Z(\beta,G)})}{\beta \sqrt{n}}
=
 \lim_{n\rightarrow\infty}\frac{\mE_G \max_{\x\in\cX}  -  \max_{\y\in\cY} \lp \y^TG\x  + f(\y)\rp   }{\sqrt{n}}
 \nonumber \\
& = & - \lim_{n\rightarrow\infty}\frac{\mE_G \min_{\x\in\cX}  \max_{\y\in\cY} \lp \y^TG\x  + f(\y)  \rp    }{\sqrt{n}}\nonumber \\
& = & - \lim_{n\rightarrow\infty}\frac{\mE_G  \xi_0(f,\cX,\cY)  }{\sqrt{n}},
  \label{eq:limlogpartfunsqrta0}
\end{eqnarray}
with
\begin{eqnarray}
\xi_0(f,\cX,\cY) & \triangleq &  \min_{\x\in\cX}  \max_{\y\in\cY} \lp \y^TG\x  + f(\y)  \rp .
  \label{eq:limlogpartfunsqrta0b0}
\end{eqnarray}
Structural similarity of $\xi_0(f,\cX,\cY) $ and  $\xi_{SBP}$  provides a direct connection between $f_{sq}(\infty)$ and $\xi_{SBP}$. In other words, if one can determine GSE  $f_{sq}(\infty)$ then $\xi_{SBP}$ would be determined as well. Knowing $\xi_{SBP}$ would allow to utilize (\ref{eq:ex1a3a0}) and determine feasibility of $\mathbf{\mathcal S}(G,\kappa,\alpha) $  (i.e., it would allow to determine whether SBP can operate as a storage memory). However, computing $f_{sq}(\infty)$ directly is usually not easy. We instead focus on studying $f_{sq}(\beta)$ for a general $\beta$ and later on specialize to $\beta\rightarrow\infty$ GSE regime. Since this specialization eventually awaits, throughout the ensuing analysis we often neglect terms without GSE relevance.

%%%%%%%%%%%%%%%%%%%%%%%%%%%%%%%%%%%%%%%%%%%%%%%%%%%%%%%%%%%%%%%%%%%%%%%%%%%%%%%%%%%%%%%%%%%%%%%%%%%%%%%%%%%%%%%%%%%%%%%
%%%%%%%%%%%%%%%%%%%%%%%%%%%%%%%%%%%%%%%%%%%%%%%%%%%%%%%%%%%%%%%%%%%%%%%%%%%%%%%%%%%%%%%%%%%%%%%%%%%%%%%%%%%%%%%%%%%%%%%
\subsection{Capacity and computational gaps}
 \label{sec:caprole}
%%%%%%%%%%%%%%%%%%%%%%%%%%%%%%%%%%%%%%%%%%%%%%%%%%%%%%%%%%%%%%%%%%%%%%%%%%%%%%%%%%%%%%%%%%%%%%%%%%%%%%%%%%%%%%%%%%%%%%%
%%%%%%%%%%%%%%%%%%%%%%%%%%%%%%%%%%%%%%%%%%%%%%%%%%%%%%%%%%%%%%%%%%%%%%%%%%%%%%%%%%%%%%%%%%%%%%%%%%%%%%%%%%%%%%%%%%%%%%%

We now mathematically formalize some of the concepts mentioned in previous sections. As stated earlier, the key measure that determines perceptron's efficiency is its \emph{storage/classifying capacity}. It is defined as the maximal constraints density for which $ \mathbf{\mathcal S}(G,\kappa,\alpha)$ is infeasible. Since we operate in random mediums capacity is a  random quantity and needs to be introduced accordingly. Following to a degree practice established earlier for  ABP \cite{StojnicGardGen13,StojnicDiscPercp13,Stojnicbinperflrdt23,Stojnicalgbp25}  we introduce  SBP's \emph{statistical} storage capacity as
 \begin{eqnarray}
\alpha & = &    \lim_{n\rightarrow \infty} \frac{m}{n}  \nonumber \\
\alpha_c(\kappa) 
& \triangleq & \min \{\alpha |\hspace{.08in}  \lim_{n\rightarrow\infty}\mP_G\lp {\mathcal S}(G,\kappa,\alpha) \hspace{.07in}\mbox{is infeasible} \rp\longrightarrow 1\}
 \nonumber \\
& = &  \min \{\alpha |\hspace{.08in}  \lim_{n\rightarrow\infty}\mP_G\lp  \xi_{SBP}>0\rp\longrightarrow 1\}.
  \label{eq:ex4}
\end{eqnarray}
Throughout the paper subscripts next to $\mP$ and/or $\mE$ (if present) denote the underlying source of randomness (in (\ref{eq:ex4}),  it is $G$). As is clear that all our analytical considerations will be in a statistical context, to facilitate writing, we most often skip reemphasizing  that statements hold with high probability. 

Practically speaking, the above defined capacity represents the critical constraints density below which SBP properly operates. In statistical contexts of interest here, it also highlights the phase-transitioning component of the underlying randomness. For example, for $\alpha>\alpha_c$  ${\mathcal S}(G,\kappa,\alpha)$ is infeasible or in the so-called UNSAT phase.  Analogously,  ${\mathcal S}(G,\kappa,\alpha)$ is feasible and in the SAT phase when $\alpha<\alpha_c$. At $\alpha_c$ a transition between these phases occurs and an exponentially large number of ${\mathcal S}(G,\kappa,\alpha)$'s solutions from SAT phase  shrinks to an empty set in the UNSAT phase \cite{KraMez89,NakSun23,BoltNakSunXu22,DingSun19,Huang24,Stojnicbinperflrdt23}. As mentioned earlier, this is the so-called \emph{satisfiability threshold} and represent the theoretical limit of SBP's storage/classifying power. Utilization of that power is possible only if  perceptron weights $\x$ can be determined. Whenever SBP is feasible (SAT phase) this is doable albeit not necessarily in a computationally efficient manner. This allows to distinguish between two $\alpha$ regimes within the SAT phase. To properly introduce them, we first define the so-called \emph{algorithmic threshold} $ \alpha_a(\kappa)$
 \begin{equation}
 \alpha_a(\kappa) 
= \alpha_{\underline{a}}(\kappa) 
 \triangleq  \max \{\alpha |\hspace{.08in}  \lim_{n\rightarrow\infty}\mP_G\lp \mbox{a feasible $\x$ in $ {\mathcal S}(G,\kappa,\alpha)$ can be found in polynomial time} \rp\longrightarrow 1\}.
  \label{eq:ex4a0}
\end{equation}
For completeness we also set
 \begin{equation}
 \alpha_{\bar{a}}(\kappa) 
 \triangleq  \min \{\alpha |\hspace{.08in}  \lim_{n\rightarrow\infty}\mP_G\lp \mbox{no feasible $\x$ in $ {\mathcal S}(G,\kappa,\alpha)$ can be found in polynomial time} \rp\longrightarrow 1\}.
  \label{eq:ex4a0a0}
\end{equation}
One then has 
\begin{eqnarray}
 \alpha_a(\kappa)  =  \alpha_{\underline{a}}(\kappa)  \leq  \alpha_{\bar{a}}(\kappa)  \leq \alpha_c(\kappa).
  \label{eq:ex4a1}
\end{eqnarray}
Under concentrations one might  further expect $\alpha_{\underline{a}}(\kappa)  =  \alpha_{\bar{a}}(\kappa) $ and instead of (\ref{eq:ex4a1}) just focus on
\begin{eqnarray}
 \alpha_a(\kappa)    \leq \alpha_c(\kappa).
  \label{eq:ex4a1a0}
\end{eqnarray}
However, it is an extraordinary challenge to determine whether the remaining inequality is strict. If it is, one has nonzero \emph{statistical-computational gap} (SCG)
 \begin{eqnarray}
SCG \triangleq  \alpha_c(\kappa) -\alpha_a(\kappa).
  \label{eq:ex4a2}
\end{eqnarray}
Settling the above question has direct implications not only on classical P vs NP but on its stronger statistical variants as well. We below propose a parametric RDT approach to tackle this challenge.

%%%%%%%%%%%%%%%%%%%%%%%%%%%%%%%%%%%%%%%%%%%%%%%%%%%%%%%%%%%%%%%%%
%%%%%%%%%%%%%%%%%%%%%%%%%%%%%%%%%%%%%%%%%%%%%%%%%%%%%%%%%%%%%%%%%
\section{Attacking SBP via (parametric) RDT}
\label{sec:implem}
%%%%%%%%%%%%%%%%%%%%%%%%%%%%%%%%%%%%%%%%%%%%%%%%%%%%%%%%%%%%%%%%%
%%%%%%%%%%%%%%%%%%%%%%%%%%%%%%%%%%%%%%%%%%%%%%%%%%%%%%%%%%%%%%%%%

Before being able to explicitly establish SBP--RDT connection we need a few technical preliminaries. 

%%%%%%%%%%%%%%%%%%%%%%%%%%%%%%%%%%%%%%%%%%%%%%%%%%%%%%%%%%%%%%%%%
%%%%%%%%%%%%%%%%%%%%%%%%%%%%%%%%%%%%%%%%%%%%%%%%%%%%%%%%%%%%%%%%%
%%%%%%%%%%%%%%%%%%%%%%%%%%%%%%%%%%%%%%%%%%%%%%%%%%%%%%%%%%%%%%%%%
\subsection{RDT preliminaries}
\label{sec:randlincons}
%%%%%%%%%%%%%%%%%%%%%%%%%%%%%%%%%%%%%%%%%%%%%%%%%%%%%%%%%%%%%%%%%
%%%%%%%%%%%%%%%%%%%%%%%%%%%%%%%%%%%%%%%%%%%%%%%%%%%%%%%%%%%%%%%%%
%%%%%%%%%%%%%%%%%%%%%%%%%%%%%%%%%%%%%%%%%%%%%%%%%%%%%%%%%%%%%%%%%

 We start by noting that the free energy from (\ref{eq:logpartfunsqrt}),
\begin{eqnarray}
f_{sq}(\beta) & = &\lim_{n\rightarrow\infty} \frac{\mE_G\log\lp \sum_{\x\in\cX} \lp \sum_{\y\in\cY}e^{\beta \lp \y^TG\x +f(\y) \rp }\rp^{-1}\rp}{\beta \sqrt{n}},\label{eq:hmsfl1}
\end{eqnarray}
is effectively a statistics of a bilinearly indexed random process $\y^TG\x$.  Powerful results related to processes of this type are developed in \cite{Stojnicsflgscompyx23,Stojnicnflgscompyx23,Stojnicflrdt23}. To make them applicable here we follow into the footsteps of \cite{Stojnicbinperflrdt23,Stojnicalgbp25}.  To that end, we take $r\in\mN$ and consider ordered vectors $\p=[\p_0,\p_1,\dots,\p_{r+1}]$, $\q=[\q_0,\q_1,\dots,\q_{r+1}]$, and $\c=[\c_0,\c_1,\dots,\c_{r+1}]$ with
 \begin{eqnarray}\label{eq:hmsfl2}
1=\p_0\geq \p_1\geq \p_2\geq \dots \geq \p_r\geq \p_{r+1} & = & 0 \nonumber \\
1=\q_0\geq \q_1\geq \q_2\geq \dots \geq \q_r\geq \q_{r+1} & = &  0,
 \end{eqnarray}
$\c_0=1$, $\c_{r+1}=0$. Let sets $\cX$ and $\cY$ be subsets of $\mS^n$  and  $\mS^m$, respectively (as earlier, $\mS^p$ stands for the unit sphere in $\mR^p$). Moreover, for $k\in\{1,2,\dots,r+1\}$ we set ${\mathcal U}_k\triangleq [u^{(4,k)},\u^{(2,k)},\h^{(k)}]$  where components of  $u^{(4,k)}\in\mR$, $\u^{(2,k)}\in\mR^m$, and $\h^{(k)}\in\mR^n$ are standard normals independent among themselves. For $s\in\{-1,1\}$ and a given function $f_S(\cdot):\mR^n\rightarrow R$ we set
  \begin{eqnarray}\label{eq:fl4}
\bar{\psi}_{S,\infty}(f_{S},\cX,\cY,\p,\q,\c,s)  =
 \mE_{G,{\mathcal U}_{r+1}} \frac{1}{n\c_r} \log
\lp \mE_{{\mathcal U}_{r}} \lp \dots \lp \mE_{{\mathcal U}_3}\lp\lp\mE_{{\mathcal U}_2} \lp \lp Z_{S,\infty}\rp^{\c_2}\rp\rp^{\frac{\c_3}{\c_2}}\rp\rp^{\frac{\c_4}{\c_3}} \dots \rp^{\frac{\c_{r}}{\c_{r-1}}}\rp, \nonumber \\
 \end{eqnarray}
with
\begin{eqnarray}\label{eq:fl5}
Z_{S,\infty} & \triangleq & e^{D_{0,S,\infty}} \nonumber \\
 D_{0,S,\infty} & \triangleq  & \max_{\x\in\cX } s \max_{\y\in\cY }
 \lp \sqrt{n} f_{S}
+\sqrt{n}      \lp\sum_{k=2}^{r+1}c_k\h^{(k)}\rp^T\x
+ \sqrt{n}  \y^T\lp\sum_{k=2}^{r+1}b_k\u^{(2,k)}\rp \rp \nonumber  \\
 b_k & \triangleq & b_k(\p,\q)=\sqrt{\p_{k-1}-\p_k} \nonumber \\
c_k & \triangleq & c_k(\p,\q)=\sqrt{\q_{k-1}-\q_k}.
 \end{eqnarray}
Equipped with the above definitions we are in position to recall on the following fundamental fl-RDT theorem (as we will see later on, the theorem is a key component that allows to connect SBP and RDT).

\begin{theorem} \cite{Stojnicbinperflrdt23,Stojnicflrdt23,Stojnicalgbp25}
\label{thm:thmsflrdt1}  Under so-called proportional regime with  $\alpha=\lim_{n\rightarrow\infty} \frac{m}{n}$ remaining constant as  $n$ grows, let elements of $G\in\mR^{m\times n}$ be independent standard normals and let $\cX\subseteq \mS^n$ and $\cY\subseteq \mS^m$ be given sets (subsets of corresponding unit spheres).  Assume that equations (\ref{eq:hmsfl2})--({\ref{eq:fl5}}) hold for an $r\in \mN$ (potentially for $r\rightarrow \infty$ as well). Consider the complete stationarized fl-RDT frame from \cite{Stojnicflrdt23,Stojnicalgbp25} and for a given function $f(\y):R^m\rightarrow R$ set
\begin{align}\label{eq:thmsflrdt2eq1a0}
   \psi_{rp} & \triangleq - \max_{\x\in\cX} s \max_{\y\in\cY} \lp \y^TG\x + f(\y) \rp
   \qquad  \mbox{(\bl{\textbf{random primal}})} \nonumber \\
   \psi_{rd}(\p,\q,\c) & \triangleq    \frac{1}{2}    \sum_{k=2}^{r+1}\Bigg(\Bigg.
   \p_{k-1}\q_{k-1}
   -\p_{k}\q_{k}
  \Bigg.\Bigg)
\c_k
  - \bar{\psi}_{S,\infty}(f(\y),\cX,\cY,\p,\q,\c,s) \quad \mbox{(\bl{\textbf{fl random dual}})}. \nonumber \\
 \end{align}
Let $\hat{\p_0}\rightarrow 1$, $\hat{\q_0}\rightarrow 1$, and $\hat{\c_0}\rightarrow 1$, $\hat{\p}_{r+1}=\hat{\q}_{r+1}=\hat{\c}_{r+1}=0$ and  let the non-fixed parts of $\hat{\p}$, $\hat{\q}$, and  $\hat{\c}$ be the solutions of the following system
\begin{eqnarray}\label{eq:thmsflrdt2eq2a0}
   \frac{d \psi_{rd}(\p,\q,\c)}{d\p} =  0,\quad
   \frac{d \psi_{rd}(\p,\q,\c)}{d\q} =  0,\quad
   \frac{d \psi_{rd}(\p,\q,\c)}{d\c} =  0.
 \end{eqnarray}
 Then,
\begin{eqnarray}\label{eq:thmsflrdt2eq3a0}
    \lim_{n,r\rightarrow\infty} \frac{\mE_G  \psi_{rp}}{\sqrt{n}}
  & = &
 \lim_{n,r\rightarrow\infty} \psi_{rd}(\hat{\p},\hat{\q},\hat{\c}) \qquad \mbox{(\bl{\textbf{strong sfl random duality}})},\nonumber \\
 \end{eqnarray}
where $\bar{\psi}_{S,\infty}(\cdot)$ is as in (\ref{eq:fl4})-(\ref{eq:fl5}).
 \end{theorem}
\begin{proof}
 Presented in  \cite{Stojnicbinperflrdt23,Stojnicflrdt23} (see also \cite{Stojnicalgbp25}).
 \end{proof}

The readers familiar with replica theory and  Parisi's RSB schemes \cite{Parisi80,Par79,Par80,Par83} may recognize the nesting type of averaging of added Gaussian sequences. In RS and RSB regimes replicated systems are utilized to create an environment where averaging over large number of replicas allows \emph{index decoupling} as an ultimate analytical simplification sufficient to eventually handle underlying random structures. Such replicated systems are not present in (\ref{eq:hmsfl1})--({\ref{eq:thmsflrdt2eq3a0}). However, one might sense a bit of decoupling flair after observing the presence of (bilinearly) coupled $f_{sq}(\cdot)$ and fully decoupled $\bar{\psi}_{S,\infty}(\cdot)$. Also, the number of lifts or nestedness level, $r$, resembles the number of replica breaking steps. The essence of Theorem \ref{thm:thmsflrdt1} is analytical connection between $f_{sq}(\cdot)$ and $\bar{\psi}_{S,\infty}(\cdot)$. This is in a way similar to how the non-replicated (coupled) systems are magically characterized by the decoupled replicated ones in a variety of  statistical physics  scenarios (for further explicit connections, see \cite{Stojnicsgenprovrsb23,Stojnicnflgscompyx23}  as well). However, the mathematical machinery enabling the above theorem  relies on studying nested random processes and it is in no way related to any form of replication.

To ensure a better contextualization, proper connections with replica results and algorithmic upgrades proposed in this paper, it is also useful to digress for a moment and take a brief look back at the chronology of key RDT developments. The plain form of the theory (usually referred to as RDT or \emph{plain} RDT) is developed as a mathematical engine that provides a generic proof of replica symmetry (RS) properties. Its main foundational principle postulates that ``\emph{strong deterministic duality $\implies$ strong random duality}''   \cite{StojnicRegRndDlt10,StojnicGardGen13}. This automatically enables direct mathematically rigorous connection between RS and existence of fast optimization algorithms. Before  \cite{StojnicRegRndDlt10,StojnicGardGen13} appeared this was a long belief of physicists usually argued via replica arguments and typically approached by rigorous mathematics on a case-by-case basis. Results of \cite{StojnicRegRndDlt10,StojnicGardGen13} confirmed such belief and proved RS in general. After partially lifted form of the theory was developed to handle hard problems with the above principle absent \cite{StojnicLiftStrSec13,StojnicMoreSophHopBnds10},
\cite{Stojnicflrdt23} introduced the fully lifted form (fl-RDT) (full match with a particularly designed RSB form was shown in \cite{Stojnicsgenprovrsb23}). The above theorem is the key component of the theory. Basically, it allows characterization of random optimization structures (called \emph{primal} ones) via the corresponding (presumably simpler) \emph{dual} ones. In what follows we propose a potential further upgrade. Speaking in the feasibility context, we suggest that monotonicity of $\c$-sequence leads to satisfiability (existential) thresholds while arbitrary ordering might lead to algorithmic ones. Or speaking in the GSE terminology, the naturally (physically) ordered  $\c$-sequence gives the theoretical GSEs whereas the unordered one gives values that efficient algorithms can achieve.

As stated above, the key advantage offered by Theorem \ref{thm:thmsflrdt1} is a simplified structure of the \emph{random dual}. However, such an advantage requires a careful interpretation. Despite that the results of Theorem \ref{thm:thmsflrdt1} may appear elegant, their ultimate practical relevance relies on one's ability to successfully conduct all underlying numerical evaluations. Two types of problems are typically faced in this regard: (\textbf{\emph{i}}) in general, sets $\cX$ and $\cY$  do not have component-wise structure which may render  straightforwardness of $\x$ and/or $\y$ decouplings as questionable; and  (\textbf{\emph{ii}}) since \emph{a priori}  $r$  can be any positive integer, the convergence in $r$ may not be fast enough to allow practical realization of numerical evaluations  on higher lifting levels. Both of these obstacles do appear. They can be handled in a satisfactory manner, but the challenge posed by the latter one is rather substantial.

%%%%%%%%%%%%%%%%%%%%%%%%%%%%%%%%%%%%%%%%%%%%%%%%%%%%%%%%%%%%%%%%%%%%%%%%%%%%%%%%%%%%%%%%%%%%%%%%%
%%%%%%%%%%%%%%%%%%%%%%%%%%%%%%%%%%%%%%%%%%%%%%%%%%%%%%%%%%%%%%%%%%%%%%%%%%%%%%%%%%%%%%%%%%%%%%%%%
\subsection{Theorem \ref{thm:thmsflrdt1} SBP specializations}
\label{sec:neg}
%%%%%%%%%%%%%%%%%%%%%%%%%%%%%%%%%%%%%%%%%%%%%%%%%%%%%%%%%%%%%%%%%%%%%%%%%%%%%%%%%%%%%%%%%%%%%%%%%
%%%%%%%%%%%%%%%%%%%%%%%%%%%%%%%%%%%%%%%%%%%%%%%%%%%%%%%%%%%%%%%%%%%%%%%%%%%%%%%%%%%%%%%%%%%%%%%%%

Theorem \ref{thm:thmsflrdt1} is very generic.  To make it useful in the context considered here we need several SBP related specializations.

\begin{corollary}
\label{cor:corsflrdt1}  Assume the setup of Theorem \ref{thm:thmsflrdt1}. For $\kappa\in \mR_+$ set $f(\y)= -\y^T\z$ and \begin{align}\label{eq:corthmsflrdt2eq1a0}
   \psi_{rp} & \triangleq - \max_{\x\in\cX,\|\z\|_{\infty}\leq \kappa} s \max_{\y\in\cY} \lp \y^TG\x - \y^T\z \rp
   \qquad  \mbox{(\bl{\textbf{random primal}})} \nonumber \\
   \psi_{rd}(\p,\q,\c) & \triangleq    \frac{1}{2}    \sum_{k=2}^{r+1}\Bigg(\Bigg.
   \p_{k-1}\q_{k-1}
   -\p_{k}\q_{k}
  \Bigg.\Bigg)
\c_k
  - \bar{\psi}_{S,\infty}(f(\y),\cX,\cY,\p,\q,\c,s) \quad \mbox{(\bl{\textbf{fl random dual}})}. \nonumber \\
 \end{align}
Let $\hat{\p_0}\rightarrow 1$, $\hat{\q_0}\rightarrow 1$, and $\hat{\c_0}\rightarrow 1$, $\hat{\p}_{r+1}=\hat{\q}_{r+1}=\hat{\c}_{r+1}=0$ and  let the non-fixed parts of $\hat{\p}$, $\hat{\q}$, and  $\hat{\c}$ be the solutions of the following system
\begin{eqnarray}\label{eq:corthmsflrdt2eq2a0}
   \frac{d \psi_{rd}(\p,\q,\c)}{d\p} =  0,\quad
   \frac{d \psi_{rd}(\p,\q,\c)}{d\q} =  0,\quad
   \frac{d \psi_{rd}(\p,\q,\c)}{d\c} =  0.
 \end{eqnarray}
 Then,
\begin{eqnarray}\label{eq:corthmsflrdt2eq3a0}
    \lim_{n,r\rightarrow\infty} \frac{\mE_G  \psi_{rp}}{\sqrt{n}}
  & = &
 \lim_{n,r\rightarrow\infty} \psi_{rd}(\hat{\p},\hat{\q},\hat{\c}) \qquad \mbox{(\bl{\textbf{strong sfl random duality}})},\nonumber \\
 \end{eqnarray}
where $\bar{\psi}_{S,\infty}(\cdot)$ is as in (\ref{eq:fl4})-(\ref{eq:fl5}) and
\begin{eqnarray}\label{eq:corfl5}
  D_{0,S,\infty} & = & \max_{\x\in\cX , \| \z \|_{\infty} \leq \kappa} s \max_{\y\in\cY }
 \lp \sqrt{n} f_{S}
+\sqrt{n}      \lp\sum_{k=2}^{r+1}c_k\h^{(k)}\rp^T\x
+ \sqrt{n}  \y^T\lp\sum_{k=2}^{r+1}b_k\u^{(2,k)}\rp \rp  .
 \end{eqnarray}
 \end{corollary}
\begin{proof}
Follow as an immediate consequence of Theorem \ref{thm:thmsflrdt1} after replacing $f(\y)$ by $\y^T\z$ and cosmetically adjusting for an additional maximization over $\z$.
 \end{proof}

Various options for $\cX$, $\cY$, and $s$ are available as well. Keeping in mind  (\ref{eq:ex3}) we consider
$\cX=\left \{-\frac{1}{\sqrt{n}},\frac{1}{\sqrt{n}} \right  \}^n$ and $\cY=\mS_+^m$ and recognize that such specialization allows to rewrite the  \emph{random dual} in the following way
\begin{align}\label{eq:prac1}
    \psi_{rd}(\p,\q,\c) & \triangleq    \frac{1}{2}    \sum_{k=2}^{r+1}\Bigg(\Bigg.
   \p_{k-1}\q_{k-1}
   -\p_{k}\q_{k}
  \Bigg.\Bigg)
\c_k
  - \bar{\psi}_{S,\infty}( \kappa \y^T\1,\cX,\cY,\p,\q,\c,s). \nonumber \\
  & =   \frac{1}{2}    \sum_{k=2}^{r+1}\Bigg(\Bigg.
   \p_{k-1}\q_{k-1}
   -\p_{k}\q_{k}
  \Bigg.\Bigg)
\c_k
  - \frac{1}{n}\varphi(D^{(bin)}(s)) - \frac{1}{n}\varphi(D^{(sph)}(s)), \nonumber \\
  \end{align}
where analogously to (\ref{eq:fl4})-(\ref{eq:fl5})
  \begin{eqnarray}\label{eq:prac2}
\varphi(D,\c) & = &
 \mE_{G,{\mathcal U}_{r+1}} \frac{1}{\c_r} \log
\lp \mE_{{\mathcal U}_{r}} \lp \dots \lp \mE_{{\mathcal U}_3}\lp\lp\mE_{{\mathcal U}_2} \lp
\lp    e^{D}   \rp^{\c_2}\rp\rp^{\frac{\c_3}{\c_2}}\rp\rp^{\frac{\c_4}{\c_3}} \dots \rp^{\frac{\c_{r}}{\c_{r-1}}}\rp, \nonumber \\
  \end{eqnarray}
and
\begin{eqnarray}\label{eq:prac3}
D^{(bin)}(s) & = & \max_{\x\in \{-\frac{1}{\sqrt{n}},\frac{1}{\sqrt{n}}\}^n} \lp   s\sqrt{n}      \lp\sum_{k=2}^{r+1}c_k\h^{(k)}\rp^T\x  \rp \nonumber \\
  D^{(sph)}(s) & \triangleq  &  \max_{\|\z|_{\infty} \leq \kappa}   s \max_{\y\in\mS_+^m}
\lp   - \sqrt{n} \y^T\z + \sqrt{n}  \y^T\lp\sum_{k=2}^{r+1}b_k\u^{(2,k)}\rp \rp.
 \end{eqnarray}
Following \cite{Stojnicbinperflrdt23,Stojnicalgbp25} we set
\begin{eqnarray}\label{eq:prac4}
D^{(bin)}(s)
=  \sum_{i=1}^n D^{(bin)}_i,  \quad \mbox{with}\quad
D^{(bin)}_i(c_k)=\left |\lp\sum_{k=2}^{r+1}c_k\h_i^{(k)}\rp \right |,
\end{eqnarray}
and obtain
  \begin{eqnarray}\label{eq:prac6}
\varphi(D^{(bin)}(s),\c) & = &
n \mE_{G,{\mathcal U}_{r+1}} \frac{1}{\c_r} \log
\lp \mE_{{\mathcal U}_{r}} \lp \dots \lp \mE_{{\mathcal U}_3}\lp\lp\mE_{{\mathcal U}_2} \lp
    e^{\c_2D_1^{(bin)}}  \rp\rp^{\frac{\c_3}{\c_2}}\rp\rp^{\frac{\c_4}{\c_3}} \dots \rp^{\frac{\c_{r}}{\c_{r-1}}}\rp
    = n\varphi(D_1^{(bin)}). \nonumber \\
   \end{eqnarray}
We then observe
\begin{eqnarray}\label{eq:prac7}
   D^{(sph)}(s) &  =  &    \max_{\|\z\|_{\infty} \leq  \kappa } s \sqrt{n}  \max_{\y\in\mS_+^m}
\lp  - \y^T\z +   \y^T\lp\sum_{k=2}^{r+1}b_k\u^{(2,k)}\rp \rp
\nonumber \\
& = & \max_{\|\z\|_{\infty} \leq  \kappa }   s  \sqrt{n}   \left \| \max \lp  - \z +\sum_{k=2}^{r+1}b_k\u^{(2,k)},0 \rp  \right \|_2.
 \end{eqnarray}
 Utilization of the  \emph{square root trick} introduced in \cite{StojnicMoreSophHopBnds10,StojnicLiftStrSec13,StojnicGardSphErr13,StojnicGardSphNeg13} we further write
\begin{align}\label{eq:prac8}
   D^{(sph)} (s)
& =  \max_{\|\z\|_{\infty} \leq  \kappa }   s  \sqrt{n}   \left \| \max \lp  - \z +\sum_{k=2}^{r+1}b_k\u^{(2,k)},0 \rp  \right \|_2
\nonumber \\
 & =   \max_{\|\z\|_{\infty} \leq  \kappa }   s\sqrt{n}  \min_{\gamma} \lp \frac{\left \| \max \lp - \z +  \sum_{k=2}^{r+1}b_k\u^{(2,k)},0 \rp  \right \|_2^2}{4\gamma}+\gamma \rp \nonumber \\
 & =  \max_{\|\z\|_{\infty} \leq  \kappa }    s\sqrt{n}  \min_{\gamma} \lp \frac{\sum_{i=1}^{m}  \max \lp -\z_i +  \sum_{k=2}^{r+1}b_k\u_i^{(2,k)},0 \rp ^2}{4\gamma}+\gamma \rp.
 \end{align}
Specialization $s=-1$ and scaling  $\gamma=\gamma_{sq}\sqrt{n}$  allow to rewrite  (\ref{eq:prac8}) as
\begin{eqnarray}\label{eq:prac9}
   D^{(sph)}(s)
  & =  &  \max_{\|\z\|_{\infty} \leq  \kappa }   -\sqrt{n}  \min_{\gamma_{sq}} \lp \frac{\sum_{i=1}^{m} \max \lp -\z_i + \sum_{k=2}^{r+1}b_k\u_i^{(2,k)},0 \rp^2}{4\gamma_{sq}\sqrt{n}}+\gamma_{sq}\sqrt{n} \rp   \nonumber \\
  & = & \max_{\|\z\|_{\infty} \leq  \kappa }  - \min_{\gamma_{sq}} \lp \frac{\sum_{i=1}^{m} \max \lp -\z_i + \sum_{k=2}^{r+1}b_k\u_i^{(2,k)},0  \rp^2}{4\gamma_{sq}}+\gamma_{sq}n \rp \nonumber \\
  & = & - \min_{\|\z\|_{\infty} \leq  \kappa }  \min_{\gamma_{sq}} \lp \frac{\sum_{i=1}^{m} \max \lp -\z_i + \sum_{k=2}^{r+1}b_k\u_i^{(2,k)},0  \rp^2}{4\gamma_{sq}}+\gamma_{sq}n \rp \nonumber \\
  & = & -  \min_{\gamma_{sq}} \lp \frac{\sum_{i=1}^{m} \min_{ | \z_i | \leq  \kappa }  \max \lp -\z_i + \sum_{k=2}^{r+1}b_k\u_i^{(2,k)},0  \rp^2}{4\gamma_{sq}}+\gamma_{sq}n \rp \nonumber \\
  & = & -  \min_{\gamma_{sq}} \lp \frac{\sum_{i=1}^{m} \max \lp -\kappa + \left  | \sum_{k=2}^{r+1}b_k\u_i^{(2,k)}  \right | ,0  \rp^2}{4\gamma_{sq}}+\gamma_{sq}n \rp \nonumber \\
  & =  &  - \min_{\gamma_{sq}} \lp \sum_{i=1}^{m} D_i^{(sph)}(b_k)+\gamma_{sq}n \rp, \nonumber \\
 \end{eqnarray}
with
\begin{eqnarray}\label{eq:prac10}
   D_i^{(sph)}(b_k)= \frac{\max \lp -\kappa + \left | \sum_{k=2}^{r+1}b_k\u_i^{(2,k)} \right |  ,0  \rp^2}{4\gamma_{sq}}.
 \end{eqnarray}
After noting the following connection between the GSE, $f_{sq}(\infty)$, from (\ref{eq:limlogpartfunsqrta0}), and the random primal, $\psi_{rp}(\cdot)$, from Corollary \ref{cor:corsflrdt1},
 \begin{equation}
-f_{sq}(\infty)
  =
- \lim_{n\rightarrow\infty}\frac{\mE_G \max_{\x\in\mB^n, \| \z \|_{\infty} \leq \kappa }  -  \max_{\y\in\mS_+^m} \lp \y^TG\x - \y^T\z \rp  }{\sqrt{n}}
 =
    \lim_{n\rightarrow\infty} \frac{\mE_G  \psi_{rp}}{\sqrt{n}}
   =
 \lim_{n\rightarrow\infty} \psi_{rd}(\hat{\p},\hat{\q},\hat{\c}),
  \label{eq:negprac11}
\end{equation}
one obtains through a combination of (\ref{eq:prac1})-(\ref{eq:negprac11})
 \begin{eqnarray}\label{eq:negprac13}
 \lim_{n\rightarrow\infty} \psi_{rd}(\hat{\p},\hat{\q},\hat{\c})
     & = &  \frac{1}{2}    \sum_{k=2}^{r+1}\Bigg(\Bigg.
   \hat{\p}_{k-1}\hat{\q}_{k-1}
   -\hat{\p}_{k}\hat{\q}_{k}
  \Bigg.\Bigg)
\hat{\c}_k
\nonumber \\
& &  - \varphi(D_1^{(bin)}(c_k(\hat{\p},\hat{\q})),\hat{\c}) + \hat{\gamma}_{sq}- \alpha\varphi(-D_1^{(sph)}(b_k(\hat{\p},\hat{\q})),\hat{\c})
\nonumber \\
& \triangleq &   \bar{\psi}_{rd}(\hat{\p},\hat{\q},\hat{\c},\hat{\gamma}_{sq}) .
  \end{eqnarray}
From  (\ref{eq:negprac11}) and (\ref{eq:negprac13}) we also have
 \begin{eqnarray}
-f_{sq}(\infty)
& = &  -\lim_{n\rightarrow\infty}\frac{\mE_G \max_{\x\in\mB^n , \| \z \|_{\infty} \leq \kappa  }  -  \max_{\y\in\mS_+^m} \lp  \y^TG\x - \y^T\z \rp  }{\sqrt{n}}
\nonumber \\
 & = &
 \lim_{n\rightarrow\infty} \psi_{rd}(\hat{\p},\hat{\q},\hat{\c})
 =   \bar{\psi}_{rd}(\hat{\p},\hat{\q},\hat{\c},\hat{\gamma}_{sq}) \nonumber \\
 & = &   \frac{1}{2}    \sum_{k=2}^{r+1}\Bigg(\Bigg.
   \hat{\p}_{k-1}\hat{\q}_{k-1}
   -\hat{\p}_{k}\hat{\q}_{k}
  \Bigg.\Bigg)
\hat{\c}_k
  - \varphi(D_1^{(bin)}(c_k(\hat{\p},\hat{\q})),\hat{\c}) + \hat{\gamma}_{sq} - \alpha\varphi(-D_1^{(sph)}(b_k(\hat{\p},\hat{\q})),\hat{\c}). \nonumber \\
  \label{eq:negprac18}
\end{eqnarray}
The above discussion is conveniently summarized in the following corollary.

\begin{corollary}
  \label{thme:negthmprac1}
  Assume the setup of Corollary \ref{cor:corsflrdt1} with $\cX=\mB^n$, $\cY=\mR_+^m$, and $s=-1$. Let $\varphi(\cdot)$ and $\bar{\psi}_{rd}(\cdot)$ be as in (\ref{eq:prac2}) and (\ref{eq:negprac13}), respectively.  Let the ``fixed'' parts of $\hat{\p}$, $\hat{\q}$, and $\hat{\c}$ satisfy $\hat{\p}_1\rightarrow 1$, $\hat{\q}_1\rightarrow 1$, $\hat{\c}_1\rightarrow 1$, $\hat{\p}_{r+1}=\hat{\q}_{r+1}=\hat{\c}_{r+1}=0$ and let the ``non-fixed'' parts of $\hat{\p}_k$, $\hat{\q}_k$, and $\hat{\c}_k$ ($k\in\{2,3,\dots,r\}$) satisfy
  \begin{eqnarray}\label{eq:negthmprac1eq1}
   \frac{d \bar{\psi}_{rd}(\p,\q,\c,\gamma_{sq})}{d\p} =
   \frac{d \bar{\psi}_{rd}(\p,\q,\c,\gamma_{sq})}{d\q} =
   \frac{d \bar{\psi}_{rd}(\p,\q,\c,\gamma_{sq})}{d\c} =
   \frac{d \bar{\psi}_{rd}(\p,\q,\c,\gamma_{sq})}{d\gamma_{sq}} =  0.
 \end{eqnarray}
One then has
 \begin{eqnarray}
-f_{sq}(\infty)
& = &     \frac{1}{2}    \sum_{k=2}^{r+1}\Bigg(\Bigg.
   \hat{\p}_{k-1}\hat{\q}_{k-1}
   -\hat{\p}_{k}\hat{\q}_{k}
  \Bigg.\Bigg)
\hat{\c}_k
  - \varphi(D_1^{(bin)}(c_k(\hat{\p},\hat{\q})),\hat{\c}) + \hat{\gamma}_{sq} - \alpha\varphi(-D_1^{(sph)}(b_k(\hat{\p},\hat{\q})),\hat{\c}), \nonumber \\
  \label{eq:negthmprac1eq2}
\end{eqnarray}
with $c_k(\hat{\p},\hat{\q})   =  \sqrt{\hat{\q}_{k-1}-\hat{\q}_k}$ and $b_k(\hat{\p},\hat{\q})   =  \sqrt{\hat{\p}_{k-1}-\hat{\p}_k}$.
\end{corollary}
\begin{proof}
Follows directly from the preceding discussion, Corollary \ref{cor:corsflrdt1}, and ultimately Theorem \ref{thm:thmsflrdt1}.
\end{proof}

%%%%%%%%%%%%%%%%%%%%%%%%%%%%%%%%%%%%%%%%%%%%%%%%%%%%%%%%%%%%%%%%%%%%%%%%%%%%%%%%%%%%%%%%%%%%%%%%%%%%%%%%%%%%%%%%%%%%%%%%%
%%%%%%%%%%%%%%%%%%%%%%%%%%%%%%%%%%%%%%%%%%%%%%%%%%%%%%%%%%%%%%%%%%%%%%%%%%%%%%%%%%%%%%%%%%%%%%%%%%%%%%%%%%%%%%%%%%%%%%%%%
%%%%%%%%%%%%%%%%%%%%%%%%%%%%%%%%%%%%%%%%%%%%%%%%%%%%%%%%%%%%%%%%%%%%%%%%%%%%%%%%%%%%%%%%%%%%%%%%%%%%%%%%%%%%%%%%%%%%%%%%%
%%%%%%%%%%%%%%%%%%%%%%%%%%%%%%%%%%%%%%%%%%%%%%%%%%%%%%%%%%%%%%%%%%%%%%%%%%%%%%%%%%%%%%%%%%%%%%%%%%%%%%%%%%%%%%%%%%%%%%%%%
%%%%%%%%%%%%%%%%%%%%%%%%%%%%%%%%%%%%%%%%%%%%%%%%%%%%%%%%%%%%%%%%%%%%%%%%%%%%%%%%%%%%%%%%%%%%%%%%%%%%%%%%%%%%%%%%%%%%%%%%%
\section{Numerical evaluations -- general $\alpha$ and $\kappa$}
\label{sec:numerical}
%%%%%%%%%%%%%%%%%%%%%%%%%%%%%%%%%%%%%%%%%%%%%%%%%%%%%%%%%%%%%%%%%%%%%%%%%%%%%%%%%%%%%%%%%%%%%%%%%%%%%%%%%%%%%%%%%%%%%%%%%
%%%%%%%%%%%%%%%%%%%%%%%%%%%%%%%%%%%%%%%%%%%%%%%%%%%%%%%%%%%%%%%%%%%%%%%%%%%%%%%%%%%%%%%%%%%%%%%%%%%%%%%%%%%%%%%%%%%%%%%%%
%%%%%%%%%%%%%%%%%%%%%%%%%%%%%%%%%%%%%%%%%%%%%%%%%%%%%%%%%%%%%%%%%%%%%%%%%%%%%%%%%%%%%%%%%%%%%%%%%%%%%%%%%%%%%%%%%%%%%%%%%
%%%%%%%%%%%%%%%%%%%%%%%%%%%%%%%%%%%%%%%%%%%%%%%%%%%%%%%%%%%%%%%%%%%%%%%%%%%%%%%%%%%%%%%%%%%%%%%%%%%%%%%%%%%%%%%%%%%%%%%%%
%%%%%%%%%%%%%%%%%%%%%%%%%%%%%%%%%%%%%%%%%%%%%%%%%%%%%%%%%%%%%%%%%%%%%%%%%%%%%%%%%%%%%%%%%%%%%%%%%%%%%%%%%%%%%%%%%%%%%%%%%

Corollary \ref{thme:negthmprac1} is utilized below to conduct numerical evaluations. We start with first lifting level ($r=1$) and continue to higher levels incrementally. To make the key points clearer, we distinguish between two scenarios: (\textbf{\emph{i}}) $r\in\{1,2\}$; and (\textbf{\emph{ii}}) $r\geq 3$. For a particular value $\kappa=1$  concrete numerical values of all relevant parameters' values and estimated capacities are given as well.

%%%%%%%%%%%%%%%%%%%%%%%%%%%%%%%%%%%%%%%%%%%%%%%%%%%%%%%%%%%%%%%%%%%%%%%%%%%%%%%%%%%%%%%%%%%%%%%%%%%%%%%%%%%%%%%%%%%%%%%%%
%%%%%%%%%%%%%%%%%%%%%%%%%%%%%%%%%%%%%%%%%%%%%%%%%%%%%%%%%%%%%%%%%%%%%%%%%%%%%%%%%%%%%%%%%%%%%%%%%%%%%%%%%%%%%%%%%%%%%%%%%
%%%%%%%%%%%%%%%%%%%%%%%%%%%%%%%%%%%%%%%%%%%%%%%%%%%%%%%%%%%%%%%%%%%%%%%%%%%%%%%%%%%%%%%%%%%%%%%%%%%%%%%%%%%%%%%%%%%%%%%%%
%%%%%%%%%%%%%%%%%%%%%%%%%%%%%%%%%%%%%%%%%%%%%%%%%%%%%%%%%%%%%%%%%%%%%%%%%%%%%%%%%%%%%%%%%%%%%%%%%%%%%%%%%%%%%%%%%%%%%%%%%
%%%%%%%%%%%%%%%%%%%%%%%%%%%%%%%%%%%%%%%%%%%%%%%%%%%%%%%%%%%%%%%%%%%%%%%%%%%%%%%%%%%%%%%%%%%%%%%%%%%%%%%%%%%%%%%%%%%%%%%%%
\subsection{Satisfiability threshold --  $r\in\{1,2\}$ }
\label{sec:nuemricalsat}
%%%%%%%%%%%%%%%%%%%%%%%%%%%%%%%%%%%%%%%%%%%%%%%%%%%%%%%%%%%%%%%%%%%%%%%%%%%%%%%%%%%%%%%%%%%%%%%%%%%%%%%%%%%%%%%%%%%%%%%%%
%%%%%%%%%%%%%%%%%%%%%%%%%%%%%%%%%%%%%%%%%%%%%%%%%%%%%%%%%%%%%%%%%%%%%%%%%%%%%%%%%%%%%%%%%%%%%%%%%%%%%%%%%%%%%%%%%%%%%%%%%
%%%%%%%%%%%%%%%%%%%%%%%%%%%%%%%%%%%%%%%%%%%%%%%%%%%%%%%%%%%%%%%%%%%%%%%%%%%%%%%%%%%%%%%%%%%%%%%%%%%%%%%%%%%%%%%%%%%%%%%%%
%%%%%%%%%%%%%%%%%%%%%%%%%%%%%%%%%%%%%%%%%%%%%%%%%%%%%%%%%%%%%%%%%%%%%%%%%%%%%%%%%%%%%%%%%%%%%%%%%%%%%%%%%%%%%%%%%%%%%%%%%
%%%%%%%%%%%%%%%%%%%%%%%%%%%%%%%%%%%%%%%%%%%%%%%%%%%%%%%%%%%%%%%%%%%%%%%%%%%%%%%%%%%%%%%%%%%%%%%%%%%%%%%%%%%%%%%%%%%%%%%%%

%%%%%%%%%%%%%%%%%%%%%%%%%%%%%%%%%%%%%%%%%%%%%%%%%%%%%%%%%%%%%%%%%%%%%%%%%%%%%%%%%%%%%%%%%%%%%%%%%%%%%%%%%%%%%%%%%%%%%%%%%
%%%%%%%%%%%%%%%%%%%%%%%%%%%%%%%%%%%%%%%%%%%%%%%%%%%%%%%%%%%%%%%%%%%%%%%%%%%%%%%%%%%%%%%%%%%%%%%%%%%%%%%%%%%%%%%%%%%%%%%%%
%%%%%%%%%%%%%%%%%%%%%%%%%%%%%%%%%%%%%%%%%%%%%%%%%%%%%%%%%%%%%%%%%%%%%%%%%%%%%%%%%%%%%%%%%%%%%%%%%%%%%%%%%%%%%%%%%%%%%%%%%
%%%%%%%%%%%%%%%%%%%%%%%%%%%%%%%%%%%%%%%%%%%%%%%%%%%%%%%%%%%%%%%%%%%%%%%%%%%%%%%%%%%%%%%%%%%%%%%%%%%%%%%%%%%%%%%%%%%%%%%%%
%%%%%%%%%%%%%%%%%%%%%%%%%%%%%%%%%%%%%%%%%%%%%%%%%%%%%%%%%%%%%%%%%%%%%%%%%%%%%%%%%%%%%%%%%%%%%%%%%%%%%%%%%%%%%%%%%%%%%%%%%
\subsubsection{$r=1$ -- first level of lifting}
\label{sec:firstlev}
%%%%%%%%%%%%%%%%%%%%%%%%%%%%%%%%%%%%%%%%%%%%%%%%%%%%%%%%%%%%%%%%%%%%%%%%%%%%%%%%%%%%%%%%%%%%%%%%%%%%%%%%%%%%%%%%%%%%%%%%%
%%%%%%%%%%%%%%%%%%%%%%%%%%%%%%%%%%%%%%%%%%%%%%%%%%%%%%%%%%%%%%%%%%%%%%%%%%%%%%%%%%%%%%%%%%%%%%%%%%%%%%%%%%%%%%%%%%%%%%%%%
%%%%%%%%%%%%%%%%%%%%%%%%%%%%%%%%%%%%%%%%%%%%%%%%%%%%%%%%%%%%%%%%%%%%%%%%%%%%%%%%%%%%%%%%%%%%%%%%%%%%%%%%%%%%%%%%%%%%%%%%%
%%%%%%%%%%%%%%%%%%%%%%%%%%%%%%%%%%%%%%%%%%%%%%%%%%%%%%%%%%%%%%%%%%%%%%%%%%%%%%%%%%%%%%%%%%%%%%%%%%%%%%%%%%%%%%%%%%%%%%%%%
%%%%%%%%%%%%%%%%%%%%%%%%%%%%%%%%%%%%%%%%%%%%%%%%%%%%%%%%%%%%%%%%%%%%%%%%%%%%%%%%%%%%%%%%%%%%%%%%%%%%%%%%%%%%%%%%%%%%%%%%%

For $r=1$ we first note $\hat{\p}_1\rightarrow 1$ and $\hat{\q}_1\rightarrow 1$. Together with $\hat{\p}_{r+1}=\hat{\p}_{2}=\hat{\q}_{r+1}=\hat{\q}_{2}=0$, and $\hat{\c}_{2}\rightarrow 0$  this allows to write
\begin{align}\label{eq:negprac19}
    \bar{\psi}_{rd}(\hat{\p},\hat{\q},\hat{\c},\gamma_{sq})   & =   \frac{1}{2}
\c_2
  - \frac{1}{\c_2}\log\lp \mE_{{\mathcal U}_2} e^{\c_2|\sqrt{1-0}\h_1^{(2)} |}\rp +\gamma_{sq}
- \alpha\frac{1}{\c_2}\log\lp \mE_{{\mathcal U}_2} e^{-\c_2\frac{\max\lp -\kappa+ \left |\sqrt{1-0}\u_1^{(2,2)} \right | ,0 \rp^2}{4\gamma_{sq}}}\rp \nonumber \\
& \rightarrow
  - \frac{1}{\c_2}\log\lp 1+ \mE_{{\mathcal U}_2} \c_2|\sqrt{1-0}\h_1^{(2)} |\rp +\gamma_{sq} \nonumber \\
& \qquad - \alpha\frac{1}{\c_2}\log\lp 1- \mE_{{\mathcal U}_2} \c_2\frac{ \max\lp -\kappa+ \left |\sqrt{1-0}\u_1^{(2,2)} \right | ,0 \rp^2 }{4\gamma_{sq}}\rp \nonumber \\
& \rightarrow
   - \frac{1}{\c_2}\log\lp 1+ \c_2\sqrt{\frac{2}{\pi}}\rp +\gamma_{sq} \nonumber \\
& \qquad - \alpha\frac{1}{\c_2}\log\lp 1- \frac{\c_2}{4\gamma_{sq}} \mE_{{\mathcal U}_2} \max\lp -\kappa+ \left |\sqrt{1-0}\u_1^{(2,2)} \right | ,0 \rp^2 \rp \nonumber \\
& \rightarrow
  - \sqrt{\frac{2}{\pi}}+\gamma_{sq}
+  \frac{\alpha}{4\gamma_{sq}}\mE_{{\mathcal U}_2} \max\lp -\kappa+ \left |\sqrt{1-0}\u_1^{(2,2)} \right | ,0 \rp^2 .
  \end{align}
Optimizing over $\gamma_{sq}$ one finds $\hat{\gamma}_{sq}=\frac{\sqrt{\alpha}}{2}\sqrt{\mE_{{\mathcal U}_2} \max\lp -\kappa+ \left |\u_1^{(2,2)} \right | ,0 \rp^2  }$ and
\begin{align}\label{eq:negprac20}
 - f_{sq}^{(1)}(\infty)=\bar{\psi}_{rd}(\hat{\p},\hat{\q},\hat{\c},\hat{\gamma}_{sq} )   & =
  - \sqrt{\frac{2}{\pi}}+\sqrt{\alpha}\sqrt{\mE_{{\mathcal U}_2} \max\lp -\kappa+ \left | \u_1^{(2,2)} \right | ,0 \rp^2  }.
  \end{align}
The first level of lifting SBP capacity estimate, $\alpha_c^{(1)}$, is obtained as critical $\alpha$ for which $f_{sq}^{(1)}(\infty)=0$. For a general $\kappa$ we have
\begin{equation}\label{eq:negprac20a0}
a_c^{(1)}(\kappa)
=  \frac{2}{\pi\mE_{{\mathcal U}_2} \max\lp - \kappa + \left | \u_1^{(2,2)} \right |  ,0\rp^2}
=  \frac{1}{\pi\lp \frac{ -\kappa e^{-\frac{\kappa^2}{2}}}{\sqrt{2\pi}} + \frac{(\kappa^2+1)\erfc\lp \frac{\kappa}{\sqrt{2}} \rp}{2}  \rp}.
  \end{equation}
Specialization to $\kappa=1$ gives
\begin{equation}\label{eq:negprac21}
\hspace{-.5in}(\mbox{\textbf{first level:}}) \qquad \quad  \alpha_c^{(1)}(1) =
\frac{2}{\pi\mE_{{\mathcal U}_2} \max\lp - 1 + \left | \u_1^{(2,2)} \right |  ,0\rp^2}
=  \frac{1}{\pi\lp \frac{ - e^{-\frac{1}{2}}}{\sqrt{2\pi}} + \erfc\lp \frac{1}{\sqrt{2}} \rp   \rp}
\approx  \bl{\mathbf{4.2250}}.
  \end{equation}

%%%%%%%%%%%%%%%%%%%%%%%%%%%%%%%%%%%%%%%%%%%%%%%%%%%%%%%%%%%%%%%%%%%%%%%%%%%%%%%%%%%%%%%%%%%%%%%%%%%%%%%%%%%%%%%%%%%%%%%%%
%%%%%%%%%%%%%%%%%%%%%%%%%%%%%%%%%%%%%%%%%%%%%%%%%%%%%%%%%%%%%%%%%%%%%%%%%%%%%%%%%%%%%%%%%%%%%%%%%%%%%%%%%%%%%%%%%%%%%%%%%
%%%%%%%%%%%%%%%%%%%%%%%%%%%%%%%%%%%%%%%%%%%%%%%%%%%%%%%%%%%%%%%%%%%%%%%%%%%%%%%%%%%%%%%%%%%%%%%%%%%%%%%%%%%%%%%%%%%%%%%%%
%%%%%%%%%%%%%%%%%%%%%%%%%%%%%%%%%%%%%%%%%%%%%%%%%%%%%%%%%%%%%%%%%%%%%%%%%%%%%%%%%%%%%%%%%%%%%%%%%%%%%%%%%%%%%%%%%%%%%%%%%
%%%%%%%%%%%%%%%%%%%%%%%%%%%%%%%%%%%%%%%%%%%%%%%%%%%%%%%%%%%%%%%%%%%%%%%%%%%%%%%%%%%%%%%%%%%%%%%%%%%%%%%%%%%%%%%%%%%%%%%%%
\subsubsection{$r=2$ -- second level of lifting}
\label{sec:secondlev}
%%%%%%%%%%%%%%%%%%%%%%%%%%%%%%%%%%%%%%%%%%%%%%%%%%%%%%%%%%%%%%%%%%%%%%%%%%%%%%%%%%%%%%%%%%%%%%%%%%%%%%%%%%%%%%%%%%%%%%%%%
%%%%%%%%%%%%%%%%%%%%%%%%%%%%%%%%%%%%%%%%%%%%%%%%%%%%%%%%%%%%%%%%%%%%%%%%%%%%%%%%%%%%%%%%%%%%%%%%%%%%%%%%%%%%%%%%%%%%%%%%%
%%%%%%%%%%%%%%%%%%%%%%%%%%%%%%%%%%%%%%%%%%%%%%%%%%%%%%%%%%%%%%%%%%%%%%%%%%%%%%%%%%%%%%%%%%%%%%%%%%%%%%%%%%%%%%%%%%%%%%%%%
%%%%%%%%%%%%%%%%%%%%%%%%%%%%%%%%%%%%%%%%%%%%%%%%%%%%%%%%%%%%%%%%%%%%%%%%%%%%%%%%%%%%%%%%%%%%%%%%%%%%%%%%%%%%%%%%%%%%%%%%%
%%%%%%%%%%%%%%%%%%%%%%%%%%%%%%%%%%%%%%%%%%%%%%%%%%%%%%%%%%%%%%%%%%%%%%%%%%%%%%%%%%%%%%%%%%%%%%%%%%%%%%%%%%%%%%%%%%%%%%%%%

For $r=2$ we again have $\hat{\p}_1\rightarrow 1$ and $\hat{\q}_1\rightarrow 1$,  and $\hat{\p}_{r+1}=\hat{\p}_{3}=\hat{\q}_{r+1}=\hat{\q}_{3}=0$, but in general   $\hat{\p}_2\neq0$, $\hat{\q}_2\neq0$, and $\hat{\c}_{2}\neq 0$.  Similarly to (\ref{eq:negprac19}), we first write
\begin{eqnarray}\label{eq:negprac24}
    \bar{\psi}_{rd}(\p,\q,\c,\gamma_{sq})   & = &  \frac{1}{2}
(1-\p_2\q_2)\c_2
  - \frac{1}{\c_2}\mE_{{\mathcal U}_3}\log\lp \mE_{{\mathcal U}_2} e^{\c_2|\sqrt{1-\q_2}\h_1^{(2)} +\sqrt{\q_2}\h_1^{(3)} |}\rp \nonumber \\
& &   + \gamma_{sq}
 -\alpha\frac{1}{\c_2}\mE_{{\mathcal U}_3} \log\lp \mE_{{\mathcal U}_2} e^{-\c_2\frac{\max\lp \left | \sqrt{1-\p_2}\u_1^{(2,2)}+\sqrt{\p_2}\u_1^{(2,3)} \right |  - \kappa ,0 \rp^2}{4\gamma_{sq}}}\rp.
    \end{eqnarray}
Solving the inner integral of the second term gives
\begin{eqnarray}\label{eq:negprac24a0}
f_{(z)}^{(2)} & = & \mE_{{\mathcal U}_2} e^{\c_2|\sqrt{1-\q_2}\h_1^{(2)} +\sqrt{\q_2}\h_1^{(3)} |}
  \nonumber \\
 & = &  \frac{1}{2}
 e^{\frac{(1-\q_2)\c_2^2}{2}}
 \Bigg(\Bigg.
 e^{-\c_2\sqrt{\q_2}\h_1^{(3)}}
 \erfc\lp - \lp\c_2\sqrt{1-\q_2}-\frac{\sqrt{\q_2}\h_1^{(3)}}{\sqrt{1-\q_2}}\rp\frac{1}{\sqrt{2}}\rp \nonumber \\
& &  + e^{\c_2\sqrt{\q_2}\h_1^{(3)}}
   \erfc\lp - \lp\c_2\sqrt{1-\q_2}+\frac{\sqrt{\q_2}\h_1^{(3)}}{\sqrt{1-\q_2}}\rp\frac{1}{\sqrt{2}}\rp
   \Bigg.\Bigg),
     \end{eqnarray}
and
\begin{eqnarray}\label{eq:negprac24a1}
  \mE_{{\mathcal U}_3}\log\lp \mE_{{\mathcal U}_2} e^{\c_2|\sqrt{1-\q_2}\h_1^{(2)} +\sqrt{\q_2}\h_1^{(3)} |}\rp
=  \mE_{{\mathcal U}_3}\log\lp f_{(z)}^{(2)}\rp.
    \end{eqnarray}
Analogously, solving the inner integral of the last term gives
\begin{eqnarray}\label{eq:negprac24a2}
\bar{h}_+ & = &  -\frac{\sqrt{\p_2}\u_1^{(2,3)}+\kappa}{\sqrt{1-\p_2}}    \nonumber \\
\bar{h}_- & = &  -\frac{\sqrt{\p_2}\u_1^{(2,3)}-\kappa}{\sqrt{1-\p_2}}    \nonumber \\
\bar{B} & = & \frac{\c_2}{4\gamma_{sq}} %\rightarrow \infty
\nonumber \\
\bar{C}_+ & = & \sqrt{\p_2}\u_1^{(2,3)}+\kappa \nonumber \\
\bar{C}_- & = & \sqrt{\p_2}\u_1^{(2,3)}-\kappa \nonumber \\
f_{(z-)}^{(2,f)}& = & \frac{e^{-\frac{\bar{B}\bar{C}_-^2}{2(1-\p_2)\bar{B} + 1}}}{2\sqrt{2(1-\p_2)\bar{B} + 1}}
\erfc\lp\frac{\bar{h}_-}{\sqrt{4(1-\p_2)\bar{B} + 2}}\rp
%\rightarrow 0
\nonumber \\
f_{(z0)}^{(2,f)}& = & \frac{1}{2}\erfc\lp \frac{\bar{h}_+}{\sqrt{2}}\rp
-\frac{1}{2}\erfc\lp \frac{\bar{h}_-}{\sqrt{2}}\rp
,  %\rightarrow \frac{1}{2}\erfc\lp \frac{\kappa}{\sqrt{2}}\rp
\nonumber \\
f_{(z+)}^{(2,f)}& = & \frac{e^{-\frac{\bar{B}\bar{C}_+^2}{2(1-\p_2)\bar{B} + 1}}}{2\sqrt{2(1-\p_2)\bar{B} + 1}}
\erfc\lp\frac{\bar{h}_+}{\sqrt{4(1-\p_2)\bar{B} + 2}}\rp
%\rightarrow 0.
   \end{eqnarray}
and
\begin{eqnarray}\label{eq:negprac24a3}
   \mE_{{\mathcal U}_3} \log\lp \mE_{{\mathcal U}_2} e^{-\c_2\frac{\max \lp \left | \sqrt{1-\p_2}\u_1^{(2,2)}+\sqrt{\p_2}\u_1^{(2,3)} \right |  - \kappa  ,0 \rp^2}{4\gamma_{sq}}}\rp
=   \mE_{{\mathcal U}_3} \log\lp f_{(z-)}^{(2,f)}+f_{(z0)}^{(2,f)} + f_{(z+)}^{(2,f)}  \rp.
    \end{eqnarray}
%    hc=-(b1*zvec+nu)/a1;
%    B=c3/4/gama;
%    C=b1*zvec+nu;
%    fzdown=(exp(-(B *C.^2)./(2 *a1^2*B + 1)).*( erfc((hc)./sqrt(4 *a1^2* B + 2))))./(2* sqrt(2 *a1^2 *B + 1));
%    fzup=.5*erfc(-hc/sqrt(2));
%    fz=fzdown+fzup;
%    Ibin=sum(log(fz).*exp(-zvec.^2/2))*dz/sqrt(2*pi);
After differentiation one finds  $\c_2\rightarrow \infty$, $\gamma_{sq}\rightarrow 0$,
$\q_2\c_2^2\rightarrow \q_2^{(s)}$, and
\begin{eqnarray}\label{eq:negprac24a4}
f_{(z)}^{(2)}
 & \rightarrow &
 e^{\frac{\c_2^2-\q_2^{(s)}}{2}}
 \Bigg(\Bigg.
 e^{-\sqrt{\q_2^{(s)}}\h_1^{(3)}}
    + e^{\sqrt{\q_2^{(s)}}\h_1^{(3)}}
    \Bigg.\Bigg).
     \end{eqnarray}
This then allows to rewrite (\ref{eq:negprac24a2}) as
\begin{eqnarray}\label{eq:negprac24a5}
\bar{h}_+ & = &  -\frac{\sqrt{\p_2}\u_i^{(2,3)}+\kappa}{\sqrt{1-\p_2}}    \nonumber \\
\bar{h}_- & = &  -\frac{\sqrt{\p_2}\u_i^{(2,3)}-\kappa}{\sqrt{1-\p_2}}    \nonumber \\
\bar{B} & = & \frac{\c_2}{4\gamma_{sq}} \rightarrow \infty
\nonumber \\
\bar{C}_+ & = & \sqrt{\p_2}\u_i^{(2,3)}+\kappa \nonumber \\
\bar{C}_- & = & \sqrt{\p_2}\u_i^{(2,3)}-\kappa \nonumber \\
f_{(z-)}^{(2,f)}& = & \frac{e^{-\frac{\bar{B}\bar{C}_-^2}{2(1-\p_2)\bar{B} + 1}}}{2\sqrt{2(1-\p_2)\bar{B} + 1}}
\erfc\lp\frac{\bar{h}_-}{\sqrt{4(1-\p_2)\bar{B} + 2}}\rp
\rightarrow 0
\nonumber \\
f_{(z0)}^{(2,f)}& = & \frac{1}{2}\erfc\lp \frac{\bar{h}_+}{\sqrt{2}}\rp
-\frac{1}{2}\erfc\lp \frac{\bar{h}_-}{\sqrt{2}}\rp
\rightarrow \frac{1}{2}\erfc\lp -\frac{\sqrt{\p_2}\u_i^{(2,3)}+\kappa}{\sqrt{2}\sqrt{1-\p_2}} \rp
-\frac{1}{2}\erfc\lp -\frac{\sqrt{\p_2}\u_i^{(2,3)}-\kappa}{\sqrt{2}\sqrt{1-\p_2}} \rp
\nonumber \\
f_{(z+)}^{(2,f)}& = & \frac{e^{-\frac{\bar{B}\bar{C}_+^2}{2(1-\p_2)\bar{B} + 1}}}{2\sqrt{2(1-\p_2)\bar{B} + 1}}
\erfc\lp\frac{\bar{h}_+}{\sqrt{4(1-\p_2)\bar{B} + 2}}\rp
\rightarrow 0.
   \end{eqnarray}
Combining (\ref{eq:negprac24}), (\ref{eq:negprac24a1}), (\ref{eq:negprac24a3}), (\ref{eq:negprac24a4}), and (\ref{eq:negprac24a5}) one finally arrives at
\begin{eqnarray}\label{eq:negprac24a6}
-f_{sq}^{(2,f)}(\infty) & = &     \bar{\psi}_{rd}(\p,\q,\c,\gamma_{sq}) \nonumber \\
  & = &  \frac{1}{2}
(1-\p_2\q_2)\c_2
  - \frac{1}{\c_2}\mE_{{\mathcal U}_3}\log\lp \mE_{{\mathcal U}_2} e^{\c_2|\sqrt{1-\q_2}\h_1^{(2)} +\sqrt{\q_2}\h_1^{(3)} |}\rp \nonumber \\
& &   + \gamma_{sq}
 -\alpha\frac{1}{\c_2}\mE_{{\mathcal U}_3} \log\lp \mE_{{\mathcal U}_2} e^{-\c_2\frac{\max\lp \left | \sqrt{1-\p_2}\u_1^{(2,2)}+\sqrt{\p_2}\u_1^{(2,3)} \right |  - \kappa ,0 \rp^2}{4\gamma_{sq}}}\rp \nonumber \\
 & = &  \frac{1}{2}
(1-\p_2\q_2)\c_2
  - \frac{1}{\c_2}\mE_{{\mathcal U}_3}\log\lp f_{(z)}^{(2)} \rp   + \gamma_{sq}
 -\alpha\frac{1}{\c_2}\mE_{{\mathcal U}_3} \log\lp f_{(z-)}^{(2,f)}+f_{(z0)}^{(2,f)} + f_{(z+)}^{(2,f)}  \rp \nonumber\\
 & \rightarrow & \frac{1}{2}
(1-\p_2\q_2)\c_2
  - \frac{1}{\c_2}\mE_{{\mathcal U}_3}\log\lp
 e^{\frac{\c_2^2-\q_2^{(s)}}{2}}
 \Bigg(\Bigg.
 e^{-\sqrt{\q_2^{(s)}}\h_1^{(3)}}
    + e^{\sqrt{\q_2^{(s)}}\h_1^{(3)}}
    \Bigg.\Bigg) \rp  \nonumber \\
& &
 -\alpha\frac{1}{\c_2}\mE_{{\mathcal U}_3} \log\lp   \frac{1}{2}\erfc\lp -\frac{\sqrt{\p_2}\u_1^{(2,3)}+\kappa}{\sqrt{2}\sqrt{1-\p_2}} \rp 
 -
 \frac{1}{2}\erfc\lp -\frac{\sqrt{\p_2}\u_1^{(2,3)}-\kappa}{\sqrt{2}\sqrt{1-\p_2}} \rp 
   \rp \nonumber\\
 & \rightarrow & \frac{1}{2}
\frac{(1-\p_2)\q_2^{(s)}}{\c_2}
  - \frac{1}{\c_2}\mE_{{\mathcal U}_3}\log\lp 2 \cosh \lp \sqrt{\q_2^{(s)}}\h_1^{(3)}\rp \rp
  \nonumber \\
& &
 -\alpha\frac{1}{\c_2}\mE_{{\mathcal U}_3} \log\lp   \frac{1}{2}\erfc\lp -\frac{\sqrt{\p_2}\u_1^{(2,3)}+\kappa}{\sqrt{2}\sqrt{1-\p_2}} \rp 
 -
 \frac{1}{2}\erfc\lp -\frac{\sqrt{\p_2}\u_1^{(2,3)}-\kappa}{\sqrt{2}\sqrt{1-\p_2}} \rp 
  \rp. \nonumber\\
    \end{eqnarray}
After determining all the derivatives and specializing to $\kappa=1$, we found no stationary points for which $\p_2\neq 0$ and/or $\q_2\neq 0$. This effectively implies that the optimum is reached on the \emph{partial} second level of lifting. We obtain as concrete numerical value of critical capacity estimate
\begin{equation}\label{eq:negprac25}
\hspace{-2in}(\mbox{\textbf{second level:}}) \qquad \qquad  \alpha_c^{(2)}(1) \approx
 \bl{\mathbf{1.8159}}.
  \end{equation}
Values of all parameters for both first  and second  level of lifting  (1-sfl-RDT and 2-sfl-RDT)  are shown in Table \ref{tab:tab1}. For decreasing  sequence $\c$, i.e., for
\begin{eqnarray}
\label{eq:addalg1}
1= \c_1 \geq \c_2 \geq \c_3\geq \dots\geq \c_{r+1}=0,
\end{eqnarray}
we found no further changes on higher ($r\geq 3$) lifting levels. This reconfirms that $\alpha_c^{(2)}(1) \approx
1.8159$ is indeed the SBP's satisfiability threshold (in agreement with what was obtained in \cite{AubPerZde19,GamKizPerXu22,PerkXu21,AbbLiSly21a,AbbLiSly21b,Bald20,Barb24,BarbAKZ23}).  In other words, even though the mathematical methodologies that we developed are completely different the final results they produce match the corresponding results available in the existing literature.

\begin{table}[h]
\caption{$r$-sfl-RDT parameters ($r\leq 2$);
 $\kappa=1$; $n,\beta\rightarrow\infty$}\vspace{.1in}
%\begin{adjustwidth}{-1.4cm}{}
\centering
\def\arraystretch{1.2}
\begin{tabular}{||l||c||c|c||c|c||c||c||}\hline\hline
 \hspace{-0in}$r$-sfl-RDT                                             & $\hat{\gamma}_{sq}$    &  $\hat{\p}_2$ & $\hat{\p}_1$     & $\hat{\q}_2^{(s)}\rightarrow \hat{\q}_2\hat{\c}_2^2$  & $\hat{\q}_1$ &  $\hat{\c}_2$    & $\alpha_c^{(r)}(1)$  \\ \hline\hline
$1$-sfl-RDT                                      & $0.3989$ &  $0$  & $\rightarrow 1$   & $0$ & $\rightarrow 1$
 &  $\rightarrow 0$  & \bl{$\mathbf{4.2250}$} \\ \hline
   $2$-sfl-RDT                                      & $\rightarrow 0$  & $\rightarrow 0$ & $\rightarrow 1$ &  $\rightarrow 0$ & $\rightarrow 1$
 &  $\rightarrow \infty$   & \bl{$\mathbf{1.8159}$}  \\ \hline\hline
  \end{tabular}
%\end{adjustwidth}
\label{tab:tab1}
\end{table}

%%%%%%%%%%%%%%%%%%%%%%%%%%%%%%%%%%%%%%%%%%%%%%%%%%%%%%%%%%%%%%%%%%%%%%%%%%%%%%%%%%%%%%%%%%%%%%%%%%%%%%%%%%%%%%%%%%%%%%%%%
%%%%%%%%%%%%%%%%%%%%%%%%%%%%%%%%%%%%%%%%%%%%%%%%%%%%%%%%%%%%%%%%%%%%%%%%%%%%%%%%%%%%%%%%%%%%%%%%%%%%%%%%%%%%%%%%%%%%%%%%%
%%%%%%%%%%%%%%%%%%%%%%%%%%%%%%%%%%%%%%%%%%%%%%%%%%%%%%%%%%%%%%%%%%%%%%%%%%%%%%%%%%%%%%%%%%%%%%%%%%%%%%%%%%%%%%%%%%%%%%%%%
%%%%%%%%%%%%%%%%%%%%%%%%%%%%%%%%%%%%%%%%%%%%%%%%%%%%%%%%%%%%%%%%%%%%%%%%%%%%%%%%%%%%%%%%%%%%%%%%%%%%%%%%%%%%%%%%%%%%%%%%%
%%%%%%%%%%%%%%%%%%%%%%%%%%%%%%%%%%%%%%%%%%%%%%%%%%%%%%%%%%%%%%%%%%%%%%%%%%%%%%%%%%%%%%%%%%%%%%%%%%%%%%%%%%%%%%%%%%%%%%%%%
\subsection{Algorithmic threshold predictions --  $r\geq 3$}
\label{sec:numericalr3}
%%%%%%%%%%%%%%%%%%%%%%%%%%%%%%%%%%%%%%%%%%%%%%%%%%%%%%%%%%%%%%%%%%%%%%%%%%%%%%%%%%%%%%%%%%%%%%%%%%%%%%%%%%%%%%%%%%%%%%%%%
%%%%%%%%%%%%%%%%%%%%%%%%%%%%%%%%%%%%%%%%%%%%%%%%%%%%%%%%%%%%%%%%%%%%%%%%%%%%%%%%%%%%%%%%%%%%%%%%%%%%%%%%%%%%%%%%%%%%%%%%%
%%%%%%%%%%%%%%%%%%%%%%%%%%%%%%%%%%%%%%%%%%%%%%%%%%%%%%%%%%%%%%%%%%%%%%%%%%%%%%%%%%%%%%%%%%%%%%%%%%%%%%%%%%%%%%%%%%%%%%%%%
%%%%%%%%%%%%%%%%%%%%%%%%%%%%%%%%%%%%%%%%%%%%%%%%%%%%%%%%%%%%%%%%%%%%%%%%%%%%%%%%%%%%%%%%%%%%%%%%%%%%%%%%%%%%%%%%%%%%%%%%%
%%%%%%%%%%%%%%%%%%%%%%%%%%%%%%%%%%%%%%%%%%%%%%%%%%%%%%%%%%%%%%%%%%%%%%%%%%%%%%%%%%%%%%%%%%%%%%%%%%%%%%%%%%%%%%%%%%%%%%%%%

In \cite{Stojniccluphop25,Stojnicalgbp25} a similar phenomenon regarding the ordering of $\c$ sequence was observed. Namely, provided that $\c$ is decreasing, for $r\geq 3$ no stationary points (beyond those already existing on the first two levels) were found. However, once the decreasing restriction on $\c$ was relaxed a different story unfolded. For both scenarios -- the negative Hopfield one from \cite{Stojniccluphop25} and the ABP from \cite{Stojnicalgbp25} -- continuations of perfectly ordered sequences of decreasing capacity estimates were obtained. The sequences also seem to exhibit fast convergence properties and the converging values happen to be in an excellent numerical agrement with the predicated algorithmic thresholds.  Such a development is very intriguing and one may wonder if it is a consequence of a more general principle potentially applicable to different problems. We below provide a set of considerations in this direction and observe that SBP might indeed be another example where conclusions similar to those reached in \cite{Stojniccluphop25,Stojnicalgbp25}  might be applicable.

To ensure clarity of exposition and easiness of following we below try to parallel presentations of \cite{Stojniccluphop25,Stojnicalgbp25} to as large extent as possible. To that end, in the first next section we start with discussion related to $r=3$ scenario and proceed further with considerations related to incrementally increased $r$'s in the subsequent sections.

%%%%%%%%%%%%%%%%%%%%%%%%%%%%%%%%%%%%%%%%%%%%%%%%%%%%%%%%%%%%%%%%%%%%%%%%%%%%%%%%%%%%%%%%%%%%%%%%%%%%%%%%%%%%%%%%%%%%%%%%%
%%%%%%%%%%%%%%%%%%%%%%%%%%%%%%%%%%%%%%%%%%%%%%%%%%%%%%%%%%%%%%%%%%%%%%%%%%%%%%%%%%%%%%%%%%%%%%%%%%%%%%%%%%%%%%%%%%%%%%%%%
%%%%%%%%%%%%%%%%%%%%%%%%%%%%%%%%%%%%%%%%%%%%%%%%%%%%%%%%%%%%%%%%%%%%%%%%%%%%%%%%%%%%%%%%%%%%%%%%%%%%%%%%%%%%%%%%%%%%%%%%%
%%%%%%%%%%%%%%%%%%%%%%%%%%%%%%%%%%%%%%%%%%%%%%%%%%%%%%%%%%%%%%%%%%%%%%%%%%%%%%%%%%%%%%%%%%%%%%%%%%%%%%%%%%%%%%%%%%%%%%%%%
%%%%%%%%%%%%%%%%%%%%%%%%%%%%%%%%%%%%%%%%%%%%%%%%%%%%%%%%%%%%%%%%%%%%%%%%%%%%%%%%%%%%%%%%%%%%%%%%%%%%%%%%%%%%%%%%%%%%%%%%%
\subsubsection{$r=3$ -- third level of lifting}
\label{sec:thirdlev}
%%%%%%%%%%%%%%%%%%%%%%%%%%%%%%%%%%%%%%%%%%%%%%%%%%%%%%%%%%%%%%%%%%%%%%%%%%%%%%%%%%%%%%%%%%%%%%%%%%%%%%%%%%%%%%%%%%%%%%%%%
%%%%%%%%%%%%%%%%%%%%%%%%%%%%%%%%%%%%%%%%%%%%%%%%%%%%%%%%%%%%%%%%%%%%%%%%%%%%%%%%%%%%%%%%%%%%%%%%%%%%%%%%%%%%%%%%%%%%%%%%%
%%%%%%%%%%%%%%%%%%%%%%%%%%%%%%%%%%%%%%%%%%%%%%%%%%%%%%%%%%%%%%%%%%%%%%%%%%%%%%%%%%%%%%%%%%%%%%%%%%%%%%%%%%%%%%%%%%%%%%%%%
%%%%%%%%%%%%%%%%%%%%%%%%%%%%%%%%%%%%%%%%%%%%%%%%%%%%%%%%%%%%%%%%%%%%%%%%%%%%%%%%%%%%%%%%%%%%%%%%%%%%%%%%%%%%%%%%%%%%%%%%%
%%%%%%%%%%%%%%%%%%%%%%%%%%%%%%%%%%%%%%%%%%%%%%%%%%%%%%%%%%%%%%%%%%%%%%%%%%%%%%%%%%%%%%%%%%%%%%%%%%%%%%%%%%%%%%%%%%%%%%%%%

Analogously to what was done on the first two levels, we first note that for $r=3$, $\hat{\p}_1\rightarrow 1$ and $\hat{\q}_1\rightarrow 1$,  and $\hat{\p}_{r+1}=\hat{\p}_{4}=\hat{\q}_{r+1}=\hat{\q}_{4}=0$, but in general   $\hat{\p}_2\neq0$,  $\hat{\p}_3\neq0$, $\hat{\q}_2\neq0$, $\hat{\q}_3\neq0$,  $\hat{\c}_{2}\neq 0$, and $\hat{\c}_{3}\neq 0$.  Analogously to (\ref{eq:negprac19}) and (\ref{eq:negprac24}), we then write
\begin{eqnarray}\label{eq:algnegprac24}
    \bar{\psi}_{rd}(\p,\q,\c,\gamma_{sq})   & = &
     \frac{1}{2}
(1-\p_2\q_2)\c_2
 +  \frac{1}{2}
(\p_2\q_2-\p_3\q_3)\c_3
\nonumber \\
& &
  - \frac{1}{\c_3}\mE_{{\mathcal U}_4}\log \lp  \mE_{{\mathcal U}_3}  \lp \mE_{{\mathcal U}_2} e^{\c_2|\sqrt{1-\q_2}\h_1^{(2)} +\sqrt{\q_2-\q_3}\h_1^{(3)}+\sqrt{\q_3}\h_1^{(4)}   |}\rp^{\frac{\c_3}{\c_2}}  \rp \nonumber \\
& &   + \gamma_{sq}
 -\alpha\frac{1}{\c_3}\mE_{{\mathcal U}_4} \log \lp \mE_{{\mathcal U}_3}   \lp \mE_{{\mathcal U}_2} e^{-\c_2\frac{\max \lp \left | \sqrt{1-\p_2}\u_1^{(2,2)}+\sqrt{\p_2-\p_3}\u_1^{(2,3)}+\sqrt{\p_3}\u_1^{(2,4)} \right |  -\kappa ,0 \rp^2}{4\gamma_{sq}}}\rp^{\frac{\c_3}{\c_2}} \rp.\nonumber \\
    \end{eqnarray}
Following \cite{Stojnicalgbp25}'s (43)-(45), one solves the most inner integral of the third term and finds
\begin{eqnarray}\label{eq:algnegprac24a1}
  \mE_{{\mathcal U}_4}\log   \lp \mE_{{\mathcal U}_3}  \lp \mE_{{\mathcal U}_2} e^{\c_2|\sqrt{1-\q_2}\h_1^{(2)} +\sqrt{\q_2-\q_3}\h_1^{(3)}  +\sqrt{\q_3}\h_1^{(4)}  |}\rp^{\frac{\c_3}{\c_2}} \rp
=  \mE_{{\mathcal U}_4}\log  \lp  \mE_{{\mathcal U}_3}   \lp f_{(z)}^{(3)}\rp^{\frac{\c_3}{\c_2}} \rp,
    \end{eqnarray}
where
\begin{eqnarray}\label{eq:algnegprac24a0}
f_{(z)}^{(3)}  
 & = &  \frac{1}{2}
 e^{\frac{(1-\q_2)\c_2^2}{2}}
 \Bigg(\Bigg.
 e^{-\c_2\zeta_3  }
 \erfc\lp - \lp\c_2\sqrt{1-\q_2}-\frac{\zeta_3 } {\sqrt{1-\q_2}}\rp\frac{1}{\sqrt{2}}\rp \nonumber \\
& &  + e^{\c_2\zeta_3   }
   \erfc\lp - \lp\c_2\sqrt{1-\q_2}+\frac{\zeta_3 } {\sqrt{1-\q_2}}\rp\frac{1}{\sqrt{2}}\rp
   \Bigg.\Bigg)   ,
\end{eqnarray}
and
\begin{eqnarray} \label{eq:algnegprac24a0a0}
\zeta_3 = \sqrt{\q_2-\q_3}\h_1^{(3)} + \sqrt{\q_3}\h_1^{(4)}.
\end{eqnarray}
Proceeding in a similar fashion one solves the most inner integral of the last term and obtains
\begin{eqnarray}\label{eq:algnegprac24a2}
\eta_3 & = & \sqrt{\p_2-\p_3}\u_1^{(2,3)} + \sqrt{\p_3}\u_1^{(2,4)}    \nonumber \\
\bar{h}_{3+} & = &  -\frac{\sqrt{\p_2-\p_3}\u_1^{(2,3)}+\sqrt{\p_3}\u_1^{(2,4)}+\kappa}{\sqrt{1-\p_2}} = -\frac{\eta_3+\kappa}{\sqrt{1-\p_2}}    \nonumber \\
\bar{h}_{3-} & = &  -\frac{\sqrt{\p_2-\p_3}\u_1^{(2,3)}+\sqrt{\p_3}\u_1^{(2,4)}-\kappa}{\sqrt{1-\p_2}} = -\frac{\eta_3-\kappa}{\sqrt{1-\p_2}}    \nonumber \\
\bar{B} & = & \frac{\c_2}{4\gamma_{sq}} %\rightarrow \infty
\nonumber \\
\bar{C}_{3+} & = & \sqrt{\p_2-\p_3}\u_1^{(2,3)}+\sqrt{\p_3}\u_1^{(2,4)}+\kappa   = \eta_3+\kappa \nonumber \\
\bar{C}_{3-} & = & \sqrt{\p_2-\p_3}\u_1^{(2,3)}+\sqrt{\p_3}\u_1^{(2,4)}-\kappa   = \eta_3-\kappa \nonumber \\
f_{(z-)}^{(3,f)}& = & \frac{e^{-\frac{\bar{B}\bar{C}_{3-}^2}{2(1-\p_2)\bar{B} + 1}}}{2\sqrt{2(1-\p_2)\bar{B} + 1}}
\erfc\lp\frac{\bar{h}_{3-}}{\sqrt{4(1-\p_2)\bar{B} + 2}}\rp
%\rightarrow 0
\nonumber \\
f_{(z0)}^{(3,f)}& = & \frac{1}{2}\erfc\lp \frac{\bar{h}_{3+}}{\sqrt{2}}\rp
- \frac{1}{2}\erfc\lp \frac{\bar{h}_{3-}}{\sqrt{2}}\rp,  %\rightarrow \frac{1}{2}\erfc\lp \frac{\kappa}{\sqrt{2}}\rp
\nonumber \\
f_{(z+)}^{(3,f)}& = & \frac{e^{-\frac{\bar{B}\bar{C}_{3+}^2}{2(1-\p_2)\bar{B} + 1}}}{2\sqrt{2(1-\p_2)\bar{B} + 1}}
\erfc\lp\frac{\bar{h}_{3+}}{\sqrt{4(1-\p_2)\bar{B} + 2}}\rp ,
%\rightarrow 0 .
   \end{eqnarray}
and
\begin{multline} \label{eq:algnegprac24a3}
   \mE_{{\mathcal U}_4} \log  \lp \mE_{{\mathcal U}_3} \lp \mE_{{\mathcal U}_2} e^{-\c_2\frac{\max\lp \left | \sqrt{1-\p_2}\u_1^{(2,2)}+\sqrt{\p_2-\p_3}\u_1^{(2,3)} +\sqrt{\p_3}\u_1^{(2,4)} \right | -  \kappa  ,0\rp^2}{4\gamma_{sq}}}\rp^{\frac{\c_3}{\c_2}} \rp
=  \\
 = \mE_{{\mathcal U}_4} \log   \lp \mE_{{\mathcal U}_3}  \lp f_{(z-)}^{(3,f)}+f_{(z0)}^{(3,f)}\rp^{\frac{\c_3}{\c_2}}
+ f_{(z+)}^{(3,f)} \rp .
    \end{multline}
Drawing parallel with second lifting level, we now have $\c_2\rightarrow \infty$, $\gamma_{sq}\rightarrow 0$,
$\q_2\c_2^2\rightarrow \q_2^{(s)}$, $\q_3\c_2^2\rightarrow \q_3^{(s)}$, $\frac{\c_3}{\c_2}\rightarrow \c_3^{(s)} $ and
\begin{eqnarray}\label{eq:algnegprac24a4}
f_{(z)}^{(3)}
&  \rightarrow &
 e^{\frac{\c_2^2-\q_2^{(s)}}{2}}
 \Bigg(\Bigg.
 e^{- \lp  \sqrt{\q_2^{(s)} - \q_3^{(s)} }\h_1^{(3)}   +  \sqrt{\q_3^{(s)}}\h_1^{(4)}      \rp }
    + e^{\sqrt{\q_2^{(s)} - \q_3^{(s)} }\h_1^{(3)}   +  \sqrt{\q_3^{(s)}}\h_1^{(4)}   }
    \Bigg.\Bigg)
\nonumber \\
&  =  &
   e^{\frac{\c_2^2-\q_2^{(s)}}{2}}
 \Bigg(\Bigg.
 e^{- \zeta_3^{(s)} }
    + e^{\zeta_3^{(s)}   }
    \Bigg.\Bigg)
      =
  e^{\frac{\c_2^2-\q_2^{(s)}}{2}}
 \Bigg(\Bigg.
  2\cosh \lp \zeta_3^{(s)} \rp
    \Bigg.\Bigg)  ,
\end{eqnarray}
where
\begin{eqnarray}\label{eq:algnegprac24a4a0}
\zeta_3^{(s)}  = \sqrt{\q_2^{(s)} - \q_3^{(s)} }\h_1^{(3)}   +  \sqrt{\q_3^{(s)}}\h_1^{(4)} .
\end{eqnarray}
Moreover, (\ref{eq:algnegprac24a2}) can be transformed  into
\begin{eqnarray}\label{eq:algnegprac24a5}
\bar{h}_{3+} & = &  -\frac{\eta_3+\kappa}{\sqrt{1-\p_2}}    \nonumber \\
\bar{h}_{3-} & = &  -\frac{\eta_3-\kappa}{\sqrt{1-\p_2}}    \nonumber \\
\bar{B} & = & \frac{\c_2}{4\gamma_{sq}} \rightarrow \infty
\nonumber \\
\bar{C}_{3+} & = & \eta_3+\kappa \nonumber \\
\bar{C}_{3-} & = & \eta_3-\kappa \nonumber \\
f_{(z-)}^{(3,f)}& = & \frac{e^{-\frac{\bar{B}\bar{C}_{3-}^2}{2(1-\p_2)\bar{B} + 1}}}{2\sqrt{2(1-\p_2)\bar{B} + 1}}
\erfc\lp\frac{\bar{h}_{3-}}{\sqrt{4(1-\p_2)\bar{B} + 2}}\rp
\rightarrow 0
\nonumber \\
f_{(z0)}^{(3,f)}& = & \frac{1}{2}\erfc\lp \frac{\bar{h}_{3+}}{\sqrt{2}}\rp
- \frac{1}{2}\erfc\lp \frac{\bar{h}_{3-}}{\sqrt{2}}\rp
\rightarrow \frac{1}{2}\erfc\lp -\frac{\eta_3+ \kappa}{\sqrt{2}\sqrt{1-\p_2}} \rp
- \frac{1}{2}\erfc\lp -\frac{\eta_3 - \kappa}{\sqrt{2}\sqrt{1-\p_2}} \rp
\nonumber \\
f_{(z+)}^{(3,f)}& = & \frac{e^{-\frac{\bar{B}\bar{C}_{3+}^2}{2(1-\p_2)\bar{B} + 1}}}{2\sqrt{2(1-\p_2)\bar{B} + 1}}
\erfc\lp\frac{\bar{h}_{3+}}{\sqrt{4(1-\p_2)\bar{B} + 2}}\rp
\rightarrow 0.
   \end{eqnarray}
A combination of (\ref{eq:algnegprac24}), (\ref{eq:algnegprac24a1}), (\ref{eq:algnegprac24a3}), (\ref{eq:algnegprac24a4}), and (\ref{eq:algnegprac24a5}) gives
\begin{eqnarray}\label{eq:algnegprac24a6}
-f_{sq}^{(3,f)}(\infty) & = &     \bar{\psi}_{rd}(\p,\q,\c,\gamma_{sq})
\nonumber \\
 & = &
     \frac{1}{2}
(1-\p_2\q_2)\c_2
 +  \frac{1}{2}
(\p_2\q_2-\p_3\q_3)\c_3
\nonumber \\
& &
  - \frac{1}{\c_3}\mE_{{\mathcal U}_4}\log \lp  \mE_{{\mathcal U}_3}  \lp \mE_{{\mathcal U}_2} e^{\c_2|\sqrt{1-\q_2}\h_1^{(2)} +\sqrt{\q_2-\q_3}\h_1^{(3)}+\sqrt{\q_3}\h_1^{(4)}   |}\rp^{\frac{\c_3}{\c_2}}  \rp \nonumber \\
& &   + \gamma_{sq}
 -\alpha\frac{1}{\c_3}\mE_{{\mathcal U}_4} \log \lp \mE_{{\mathcal U}_3}   \lp \mE_{{\mathcal U}_2} e^{-\c_2\frac{\max\lp \left | \sqrt{1-\p_2}\u_1^{(2,2)}+\sqrt{\p_2-\p_3}\u_1^{(2,3)}+\sqrt{\p_3}\u_1^{(2,4)} \right | - \kappa ,0\rp^2}{4\gamma_{sq}}}\rp^{\frac{\c_3}{\c_2}} \rp 
 \nonumber \\
& \rightarrow & \frac{1}{2}
(1-\p_2\q_2)\c_2
 +  \frac{1}{2}
(\p_2\q_2-\p_3\q_3)\c_3
  - \frac{1}{\c_3}\mE_{{\mathcal U}_4}\log  \lp \mE_{{\mathcal U}_3}   \lp
 e^{\frac{\c_2^2-\q_2^{(s)}}{2}}
 \Bigg(\Bigg.
  2\cosh \lp \zeta_3^{(s)} \rp
    \Bigg.\Bigg) \rp^{\frac{\c_3}{\c_2}}  \rp
     \nonumber \\
& &
 -\alpha\frac{1}{\c_3}  \mE_{{\mathcal U}_4} \log  \lp \mE_{{\mathcal U}_3}  \lp   \frac{1}{2}\erfc\lp -\frac{\eta_3+\kappa}{\sqrt{2}\sqrt{1-\p_2}} \rp 
 -\frac{1}{2}\erfc\lp -\frac{\eta_3-\kappa}{\sqrt{2}\sqrt{1-\p_2}} \rp
  \rp^{\frac{\c_3}{\c_2}} \rp
 \nonumber\\
 & \rightarrow & \frac{1}{2}
\frac{(1-\p_2)\q_2^{(s)}}{\c_2}
 +  \frac{1}{2}
(\p_2\q_2-\p_3\q_3)\c_3
  - \frac{1}{\c_3}\mE_{{\mathcal U}_4}\log  \lp \mE_{{\mathcal U}_3}   \lp
   2\cosh \lp \zeta_3^{(s)} \rp
  \rp^{\frac{\c_3}{\c_2}}  \rp
  \nonumber \\
& &
 -\alpha\frac{1}{\c_3}  \mE_{{\mathcal U}_4} \log  \lp \mE_{{\mathcal U}_3}  \lp   \frac{1}{2}\erfc\lp -\frac{\eta_3+\kappa}{\sqrt{2}\sqrt{1-\p_2}} \rp 
 -\frac{1}{2}\erfc\lp -\frac{\eta_3-\kappa}{\sqrt{2}\sqrt{1-\p_2}} \rp
  \rp^{\frac{\c_3}{\c_2}} \rp
  \nonumber
  \\
& \rightarrow & \frac{1}{2}
\frac{(1-\p_2)\q_2^{(s)}}{\c_2}
 +  \frac{1}{2}
\frac{(\p_2\q_2^{(s)}-\p_3\q_3^{(s)})\c_3^{(s)}}{\c_2}
  - \frac{1}{\c_2\c_3^{(s)}}\mE_{{\mathcal U}_4}\log  \lp \mE_{{\mathcal U}_3}   \lp
  2\cosh \lp \zeta_3^{(s)} \rp
     \rp^{\c_3^{(s)}}  \rp
  \nonumber \\
& &
 -\alpha\frac{1}{\c_2\c_3^{(s)}}  \mE_{{\mathcal U}_4} \log  \lp \mE_{{\mathcal U}_3}  \lp   \frac{1}{2}\erfc\lp -\frac{\eta_3+\kappa}{\sqrt{2}\sqrt{1-\p_2}} \rp 
 - \frac{1}{2}\erfc\lp -\frac{\eta_3-\kappa}{\sqrt{2}\sqrt{1-\p_2}} \rp \rp^{\c_3^{(s)}} \rp .
    \end{eqnarray}
After setting
\begin{eqnarray}\label{eq:algnegprac24a6a0}
   \bar{\psi}_{rd}^{(3,s)}(\p,\q,\c,\gamma_{sq})
\hspace{-.0in}  & \triangleq & \hspace{-.0in}
  \frac{1}{2}
(1-\p_2)\q_2^{(s)}
 +  \frac{1}{2}
\lp \p_2\q_2^{(s)}-\p_3\q_3^{(s)} \rp\c_3^{(s)}
  - \frac{1}{\c_3^{(s)}}\mE_{{\mathcal U}_4}\log  \lp \mE_{{\mathcal U}_3}   \lp
  2\cosh \lp \zeta_3^{(s)} \rp
     \rp^{\c_3^{(s)}}  \rp
  \nonumber \\
\hspace{-.0in} & & \hspace{-.0in}
 -\alpha\frac{1}{\c_3^{(s)}}  \mE_{{\mathcal U}_4} \log  \lp \mE_{{\mathcal U}_3}  \lp  
  \frac{1}{2}\erfc\lp -\frac{\eta_3+\kappa}{\sqrt{2}\sqrt{1-\p_2}} \rp  
  -
    \frac{1}{2}\erfc\lp -\frac{\eta_3-\kappa}{\sqrt{2}\sqrt{1-\p_2}} \rp 
  \rp^{\c_3^{(s)}} \rp , \nonumber \\
    \end{eqnarray}
the critical condition to determine capacity, $f_{sq}(\infty)=0$, on the third level becomes $  \bar{\psi}_{rd}^{(3,s)}(\p,\q,\c,\gamma_{sq}) = 0$. Evaluating numerically for $\kappa=1$ specialization we obtain \emph{virtual} capacity estimate
\begin{equation}\label{eq:algnegprac25}
\hspace{-2in}(\mbox{\textbf{third level:}}) \qquad \qquad  \alpha_c^{(3)}(1) \approx
 \red{\mathbf{1.6576}}.
  \end{equation}
 The parameters values that produce the above estimate are given in Table \ref{tab:tab2} (the corresponding first and second lifting levels values obtained in earlier sections are shown in parallel as well).
\begin{table}[h]
\caption{$r$-sfl-RDT parameters  ($r\leq 3$);   $\hat{\c}_2\rightarrow \infty$;   $\hat{\c}_3^{(s)} = \lim_{\hat{\c}_2\rightarrow\infty} \frac{\hat{\c}_3}{\hat{\c}_2}$; $\kappa=1$; $n,\beta\rightarrow\infty$}\vspace{.1in}
%\begin{adjustwidth}{-1.4cm}{}
\centering
\def\arraystretch{1.2}
{\small \begin{tabular}{||l||c||c|c|c||c|c|c||c|c||c||}\hline\hline
 \hspace{-0in}$r$-sfl-RDT                                             & $\hat{\gamma}_{sq}$    &  $\hat{\p}_3$  &  $\hat{\p}_2$ & $\hat{\p}_1$     & $\hat{\q}_3^{(s)}\rightarrow \hat{\q}_3\hat{\c}_2^2$  & $\hat{\q}_2^{(s)}\rightarrow \hat{\q}_2\hat{\c}_2^2$  & $\hat{\q}_1$ &  $\hat{\c}_3^{(s)}$  &  $\hat{\c}_2$    & $\alpha_c^{(r)}(1)$  \\ \hline\hline
$1$-sfl-RDT                                      & $0.3989$ &  $0$  &  $0$   & $\rightarrow 1$  &  $0$    & $0$ & $\rightarrow 1$
& $ \rightarrow 0 $ &  $\rightarrow 0$  & \bl{$\mathbf{4.2250}$} \\ \hline
   $2$-sfl-RDT                                      & $\rightarrow 0$ &  $0$   & $\rightarrow 0$ & $\rightarrow 1$ &   $0$  & $\rightarrow 0$ & $\rightarrow 1$
& $ \rightarrow 0 $ &  $\rightarrow \infty$   & \bl{$\mathbf{1.8159}$}  \\ \hline\hline
   $3$-sfl-RDT                                      & $\rightarrow 0$ &  $\rightarrow 0$    & $0.9852$ & $\rightarrow 1$ &  $\rightarrow 0$   &  $0.8794$ & $\rightarrow 1$
& $4.2629$  &  $\rightarrow \infty$   & \red{$\mathbf{1.6576}$}  \\ \hline\hline
  \end{tabular}}
%\end{adjustwidth}
\label{tab:tab2}
\end{table}
Looking at Table \ref{tab:tab2} one notes that $\frac{\c_3}{\c_2}=\c_3^{(s)}>1$, which, under the natural (physical) decreasing sequence $\c$ assumption, renders the obtained capacity estimate as nonphysical.  While the estimate does not seem to be directly related to SBP \emph{satisfiability} capacity, it has two interesting properties: (\textbf{\emph{i}}) 
$\alpha_c^{(3)}(1)<\alpha_c^{(2)}(1)$;  and (\textbf{\emph{ii}}) not only is $\alpha_c^{(3)}(1)\approx 1.6576$ substantially smaller than $\alpha_c^{(2)}(1)\approx 1.8159$, it is actually much closer to a range around $\sim 1.6$. Properties of this type were observed and tightly connected to statistical computational gaps (SCG) when studying ABP and negative Hopfield models in \cite{Stojniccluphop25,Stojnicalgbp25}. In particular, in \cite{Stojnicalgbp25} the limiting value $\lim_{r\rightarrow \infty}\alpha_c^{(r)}(1)$ was proposed as the \emph{algorithmic} capacity (analogous predictions were made in \cite{Stojniccluphop25} for the algorithmically achievable GSEs). Whether propositions of this type extend to SBP is intriguing. In that regard it is useful to recall that ABP's SCG connection was additionally substantiated with presence of: (\textbf{\emph{i}}) polynomial algorithms that can approach (at least some reasonable vicinity of) the proposed algorithmic capacity; and  (\textbf{\emph{ii}}) solid logical/theoretical justifications that could serve as additional motivation. Conveniently, in \cite{Bald20} clustering structure of atypical SBP solutions was studied via local entropy approaches. Utilizing replica methods rare dense solutions' clusters were predicted to disappear for constraints densities $\alpha_{LE}\approx 1.58$. This is in a way similar to ABP results of \cite{Bald15,Bald16,Bald21,Stojnicabple25}  where $\alpha$ interval $(0.77,0.78)$ was indicated as the clustering defragmentation range (\cite{Bald15} further postulates that such abrupt clusters' structuring change may be a plausible cause for failure of locally improving algorithms).

While the potential algorithmic capacity estimate that we have obtained above, $\alpha_c^{(3)}(1)=1.6576$, is much lower than satisfiability one $1.8159$ (thereby pointing to a substantial SCG), it is still somewhat distant from the local entropy  $\alpha_{LE}\approx 1.58$ defragmentation prediction. Moreover, the efficient algorithms in the predicated $\alpha\leq \alpha_{LE}$ range are desired as well. The following sections address these two mismatches. The former one is pursued via higher lifting levels analyses, whereas the latter one relies on particular designs of SBP adapted CLuP algorithms \cite{Stojnicclupspreg20,Stojnicclupint19,Stojniccluphop25,Stojnicclupsk25,Stojnicalgbp25}.

%%%%%%%%%%%%%%%%%%%%%%%%%%%%%%%%%%%%%%%%%%%%%%%%%%%%%%%%%%%%%%%%%%%%%%%%%%%%%%%%%%%%%%%%%%%%%%%%%%%%%%%%%%%%%%%%%%%%%%%%%
%%%%%%%%%%%%%%%%%%%%%%%%%%%%%%%%%%%%%%%%%%%%%%%%%%%%%%%%%%%%%%%%%%%%%%%%%%%%%%%%%%%%%%%%%%%%%%%%%%%%%%%%%%%%%%%%%%%%%%%%%
%%%%%%%%%%%%%%%%%%%%%%%%%%%%%%%%%%%%%%%%%%%%%%%%%%%%%%%%%%%%%%%%%%%%%%%%%%%%%%%%%%%%%%%%%%%%%%%%%%%%%%%%%%%%%%%%%%%%%%%%%
\subsubsection{General $r$--th level of lifting}
\label{sec:rthlev}
%%%%%%%%%%%%%%%%%%%%%%%%%%%%%%%%%%%%%%%%%%%%%%%%%%%%%%%%%%%%%%%%%%%%%%%%%%%%%%%%%%%%%%%%%%%%%%%%%%%%%%%%%%%%%%%%%%%%%%%%%
%%%%%%%%%%%%%%%%%%%%%%%%%%%%%%%%%%%%%%%%%%%%%%%%%%%%%%%%%%%%%%%%%%%%%%%%%%%%%%%%%%%%%%%%%%%%%%%%%%%%%%%%%%%%%%%%%%%%%%%%%
%%%%%%%%%%%%%%%%%%%%%%%%%%%%%%%%%%%%%%%%%%%%%%%%%%%%%%%%%%%%%%%%%%%%%%%%%%%%%%%%%%%%%%%%%%%%%%%%%%%%%%%%%%%%%%%%%%%%%%%%%

It is not that hard to predict that numerical evaluations would get more computationally challenging  as $r$ increases. Efficient $r$-level generalizations might be a bit helpful in that regard and we show them next.

Pralleling what was done on the first three levels, we for general $r$ first note  $\hat{\p}_1\rightarrow 1$ and $\hat{\q}_1\rightarrow 1$,  and $\hat{\p}_{r+1}=\hat{\q}_{r+1}=0$, and    $\hat{\p}_k\neq0$, $\hat{\q}_k\neq0$, and $\hat{\c}_{k}\neq 0$ for $2\leq k\leq r$.  Analogously to (\ref{eq:algnegprac24}) (and earlier (\ref{eq:negprac19}) and (\ref{eq:negprac24})), we first write
\begin{eqnarray}\label{eq:algalgnegprac24}
    \bar{\psi}_{rd}(\p,\q,\c,\gamma_{sq})   & = &
     \frac{1}{2}
\sum_{k=2}^{r+1}(\p_{k-1}\q_{k-1}-\p_k\q_k)\c_k
 \nonumber \\
& &
  - \frac{1}{\c_r}\mE_{{\mathcal U}_{r+1}}\log \lp \dots \mE_{{\mathcal U}_4}  \lp  \mE_{{\mathcal U}_3}  \lp \mE_{{\mathcal U}_2} e^{\c_2|  \sum_{k=2}^{r+1} c_k\h_1^{(k)}    |}\rp^{\frac{\c_3}{\c_2}}  \rp^{\frac{\c_4}{\c_3}} \dots \rp    \nonumber \\
& &   + \gamma_{sq}
 -\alpha\frac{1}{\c_r}\mE_{{\mathcal U}_{r+1}} \log \lp \dots \mE_{{\mathcal U}_4}  \lp \mE_{{\mathcal U}_3}   \lp \mE_{{\mathcal U}_2} e^{-\c_2\frac{\max\lp \left |  \sum_{k=2}^{r+1} b_k \u_1^{(2,k)} \right |  -\kappa ,0\rp^2}{4\gamma_{sq}}}\rp^{\frac{\c_3}{\c_2}} \rp^{\frac{\c_4}{\c_3}} \dots \rp ,\nonumber \\
    \end{eqnarray}
where, as in (\ref{eq:fl5}),
\begin{eqnarray}\label{eq:algalgfl5}
  b_k  & = & \sqrt{\p_{k-1}-\p_k} \nonumber \\
c_k   & = & \sqrt{\q_{k-1}-\q_k}.
 \end{eqnarray}
As in \cite{Stojnicalgbp25}'s (56)-(58), solving the most inner integral gives 
\begin{equation}\label{eq:algalgnegprac24a1}
\mE_{{\mathcal U}_{r+1}}\log \lp \dots \mE_{{\mathcal U}_4}  \lp  \mE_{{\mathcal U}_3}  \lp \mE_{{\mathcal U}_2} e^{\c_2|  \sum_{k=2}^{r+1} c_k\h_1^{(k)}    |}\rp^{\frac{\c_3}{\c_2}}  \rp^{\frac{\c_4}{\c_3}} \dots \rp
=  \mE_{{\mathcal U}_{r+1}}\log  \lp  \dots \mE_{{\mathcal U}_4}  \lp  \mE_{{\mathcal U}_3}   \lp f_{(z)}^{(r)}\rp^{\frac{\c_3}{\c_2}} \rp^{\frac{\c_4}{\c_3}} \dots \rp,
    \end{equation}
 where
\begin{eqnarray}\label{eq:algalgnegprac24a0}
f_{(z)}^{(r)} & = & \mE_{{\mathcal U}_2} e^{\c_2| \sum_{k=2}^{r+1} c_k \h_1^{(k)}  |}
  \nonumber \\
 & = &
  \frac{1}{2}
 e^{\frac{(1-\q_2)\c_2^2}{2}}
 \Bigg(\Bigg.
 e^{-\c_2\zeta_r  }
 \erfc\lp - \lp\c_2\sqrt{1-\q_2}-\frac{\zeta_r } {\sqrt{1-\q_2}}\rp\frac{1}{\sqrt{2}}\rp
 \nonumber \\
& &
 + e^{\c_2\zeta_r   }
   \erfc\lp - \lp\c_2\sqrt{1-\q_2}+\frac{\zeta_r } {\sqrt{1-\q_2}}\rp\frac{1}{\sqrt{2}}\rp
   \Bigg.\Bigg)   ,
\end{eqnarray}
and
\begin{eqnarray} \label{eq:algalgnegprac24a0a0}
\zeta_r = \sum_{k=3}^{r+1} c_k \h_1^{(k)}.
\end{eqnarray}
 After solving the most inner integral of the last term we find
\begin{eqnarray}\label{eq:algalgnegprac24a2}
\eta_r & = & \sum_{k=3}^{r+1} b_k \u_1^{(2,k)}    \nonumber \\
\bar{h}_{3+} & = &  -\frac{\sum_{k=3}^{r+1} b_k \u_1^{(2,k)} +\kappa}{\sqrt{1-\p_2}} = -\frac{\eta_r+\kappa}{\sqrt{1-\p_2}}    \nonumber \\
\bar{h}_{3-} & = &  -\frac{\sum_{k=3}^{r+1} b_k \u_1^{(2,k)} +\kappa}{\sqrt{1-\p_2}} = -\frac{\eta_r-\kappa}{\sqrt{1-\p_2}}    \nonumber \\
\bar{B} & = & \frac{\c_2}{4\gamma_{sq}} %\rightarrow \infty
\nonumber \\
\bar{C}_{r+} & = & \sum_{k=3}^{r+1} b_k \u_1^{(2,k)}  +\kappa   = \eta_r+\kappa \nonumber \\
\bar{C}_{r-} & = & \sum_{k=3}^{r+1} b_k \u_1^{(2,k)}  +\kappa   = \eta_r-\kappa \nonumber \\
f_{(z-)}^{(r,f)}& = & \frac{e^{-\frac{\bar{B}\bar{C}_{r-}^2}{2(1-\p_2)\bar{B} + 1}}}{2\sqrt{2(1-\p_2)\bar{B} + 1}}
\erfc\lp\frac{\bar{h}_{r-}}{\sqrt{4(1-\p_2)\bar{B} + 2}}\rp
%\rightarrow 0
\nonumber \\
f_{(z0)}^{(r,f)}& = & \frac{1}{2}\erfc\lp \frac{\bar{h}_{r+}}{\sqrt{2}}\rp
-
\frac{1}{2}\erfc\lp \frac{\bar{h}_{r-}}{\sqrt{2}}\rp,  %\rightarrow \frac{1}{2}\erfc\lp \frac{\kappa}{\sqrt{2}}\rp
\nonumber \\
f_{(z+)}^{(r,f)}& = & \frac{e^{-\frac{\bar{B}\bar{C}_{r+}^2}{2(1-\p_2)\bar{B} + 1}}}{2\sqrt{2(1-\p_2)\bar{B} + 1}}
\erfc\lp\frac{\bar{h}_{r+}}{\sqrt{4(1-\p_2)\bar{B} + 2}}\rp
%\rightarrow 0
   \end{eqnarray}
and
\begin{multline} \label{eq:algalgnegprac24a3}
   \mE_{{\mathcal U}_{r+1}} \log  \lp \dots \mE_{{\mathcal U}_4}  \lp \mE_{{\mathcal U}_3} \lp \mE_{{\mathcal U}_2} e^{-\c_2\frac{\max(\sum_{k=2}^{r+1} b_k \u_1^{(2,k)} +  \kappa  ,0)^2}{4\gamma_{sq}}}\rp^{\frac{\c_3}{\c_2}} \rp^{\frac{\c_4}{\c_3}} \dots \rp =
   \\
=   \mE_{{\mathcal U}_{r+1}} \log   \lp \dots  \mE_{{\mathcal U}_4}  \lp \mE_{{\mathcal U}_3}  \lp f_{(zd)}^{(r,f)}+f_{(zu)}^{(r,f)}\rp^{\frac{\c_3}{\c_2}} \rp^{\frac{\c_4}{\c_3}} \dots \rp .
    \end{multline}
Similarly to what was done in Section \ref{sec:thirdlev}, one now finds $\c_2\rightarrow \infty$, $\gamma_{sq}\rightarrow 0$,
$\q_k\c_2^2\rightarrow \q_k^{(s)}$, $\frac{\c_k}{\c_2}\rightarrow \c_k^{(s)} $, $2\leq k\leq r$, and
\begin{eqnarray}\label{eq:algalgnegprac24a4}
f_{(z)}^{(r)}
&  \rightarrow &
 e^{\frac{\c_2^2-\q_2^{(s)}}{2}}
 \Bigg(\Bigg.
 e^{- \lp \sum_{k=3}^{r+1} c_k \h_1^{(k)}      \rp }
    + e^{\sum_{k=3}^{r+1} c_k \h_1^{(k)}    }
    \Bigg.\Bigg)
\nonumber \\
&  =  &
   e^{\frac{\c_2^2-\q_2^{(s)}}{2}}
 \Bigg(\Bigg.
 e^{- \zeta_r^{(s)} }
    + e^{\zeta_r^{(s)}   }
    \Bigg.\Bigg)
      =
  e^{\frac{\c_2^2-\q_2^{(s)}}{2}}
 \Bigg(\Bigg.
  2\cosh \lp \zeta_r^{(s)} \rp
    \Bigg.\Bigg)  ,
\end{eqnarray}
where
\begin{eqnarray}\label{eq:algalgnegprac24a4a0}
\zeta_r^{(s)} & = & \sum_{k=3}^{r+1} c_k^{(s)} \h_1^{(4)} \nonumber \\
c_k^{(s)} & =  &  \sqrt{\q_{k-1}^{(s)}-\q_k^{(s)}},3\leq k\leq r+1.
\end{eqnarray}
Moreover, (\ref{eq:algalgnegprac24a2})  can be transformed into
\begin{eqnarray}\label{eq:algalgnegprac24a5}
\bar{h}_{r+} & = &  -\frac{\eta_r+\kappa}{\sqrt{1-\p_2}}    \nonumber \\
\bar{h}_{r-} & = &  -\frac{\eta_r-\kappa}{\sqrt{1-\p_2}}    \nonumber \\
\bar{B} & = & \frac{\c_2}{4\gamma_{sq}} \rightarrow \infty
\nonumber \\
\bar{C}_{r+} & = & \eta_3+\kappa \nonumber \\
\bar{C}_{r-} & = & \eta_3-\kappa \nonumber \\
f_{(z-)}^{(r,f)}& = & \frac{e^{-\frac{\bar{B}\bar{C}_{r-}^2}{2(1-\p_2)\bar{B} + 1}}}{2\sqrt{2(1-\p_2)\bar{B} + 1}}
\erfc\lp\frac{\bar{h}_{r-}}{\sqrt{4(1-\p_2)\bar{B} + 2}}\rp
\rightarrow 0
\nonumber \\
f_{(z0)}^{(r,f)}& = & \frac{1}{2}\erfc\lp \frac{\bar{h}_{r+}}{\sqrt{2}}\rp
-
\frac{1}{2}\erfc\lp\frac{\bar{h}_{r-}}{\sqrt{2}}\rp
\rightarrow \frac{1}{2}\erfc\lp - \frac{\eta_r+ \kappa}{\sqrt{2}\sqrt{1-\p_2}} \rp
-
\frac{1}{2}\erfc\lp - \frac{\eta_r- \kappa}{\sqrt{2}\sqrt{1-\p_2}} \rp
\nonumber \\
f_{(z+)}^{(r,f)}& = & \frac{e^{-\frac{\bar{B}\bar{C}_{r+}^2}{2(1-\p_2)\bar{B} + 1}}}{2\sqrt{2(1-\p_2)\bar{B} + 1}}
\erfc\lp\frac{\bar{h}_{r+}}{\sqrt{4(1-\p_2)\bar{B} + 2}}\rp
\rightarrow 0.
   \end{eqnarray}
Combining  (\ref{eq:algalgnegprac24}), (\ref{eq:algalgnegprac24a1}), (\ref{eq:algalgnegprac24a3}), (\ref{eq:algalgnegprac24a4}), and (\ref{eq:algalgnegprac24a5}) we obtain
\begin{eqnarray}\label{eq:algalgnegprac24a6b0}
-f_{sq}^{(r,f)}(\infty) & = &     \bar{\psi}_{rd}(\p,\q,\c,\gamma_{sq})
\nonumber \\
 & = &
     \frac{1}{2}
\sum_{k=2}^{r+1}(\p_{k-1}\q_{k-1}-\p_k\q_k)\c_k
 \nonumber \\
& &
  - \frac{1}{\c_r}\mE_{{\mathcal U}_{r+1}}\log \lp \dots \mE_{{\mathcal U}_4}  \lp  \mE_{{\mathcal U}_3}  \lp \mE_{{\mathcal U}_2} e^{\c_2|  \sum_{k=2}^{r+1} c_k\h_1^{(k)}    |}\rp^{\frac{\c_3}{\c_2}}  \rp^{\frac{\c_4}{\c_3}} \dots \rp    \nonumber \\
& &   + \gamma_{sq}
 -\alpha\frac{1}{\c_r}\mE_{{\mathcal U}_{r+1}} \log \lp \dots \mE_{{\mathcal U}_4}  \lp \mE_{{\mathcal U}_3}   \lp \mE_{{\mathcal U}_2} e^{-\c_2\frac{\max\lp \left | \sum_{k=2}^{r+1} b_k \u_1^{(2,k)} \right | - \kappa ,0\rp^2}{4\gamma_{sq}}}\rp^{\frac{\c_3}{\c_2}} \rp^{\frac{\c_4}{\c_3}} \dots \rp ,
 \nonumber \\
 & = &
       \frac{1}{2}
\sum_{k=2}^{r+1}(\p_{k-1}\q_{k-1}-\p_k\q_k)\c_k
  - \frac{1}{\c_r}
  \mE_{{\mathcal U}_{r+1}}\log  \lp  \dots \mE_{{\mathcal U}_4}  \lp  \mE_{{\mathcal U}_3}   \lp f_{(z)}^{(r)}\rp^{\frac{\c_3}{\c_2}} \rp^{\frac{\c_4}{\c_3}} \dots \rp
   \nonumber \\
& &   + \gamma_{sq}
 -\alpha\frac{1}{\c_r}
  \mE_{{\mathcal U}_{r+1}} \log   \lp \dots  \mE_{{\mathcal U}_4}  \lp \mE_{{\mathcal U}_3}  \lp f_{(z-)}^{(r,f)}+f_{(z0)}^{(r,f)} + f_{(z+)}^{(r,f)}   \rp^{\frac{\c_3}{\c_2}} \rp^{\frac{\c_4}{\c_3}} \dots \rp,
    \end{eqnarray}
    end 
\begin{align}    \label{eq:algalgnegprac24a6}
 & -f_{sq}^{(r,f)}(\infty)
    \rightarrow 
       \frac{1}{2\c_2}
\sum_{k=2}^{r+1} (\p_{k-1}\q_{k-1}^{(s)}-\p_k\q_k^{(s)})\c_k^{(s)}
\\
 &  - \frac{1}{\c_2\c_r^{(s)}}\mE_{{\mathcal U}_{r+1}}\log  \lp \dots \mE_{{\mathcal U}_4}   \lp \mE_{{\mathcal U}_3}   \lp
  2\cosh \lp \zeta_r^{(s)} \rp
     \rp^{\c_3^{(s)}}  \rp^{\frac{\c_4^{(s)}}{\c_3^{(s)}}} \dots \rp
  \\
 & -\alpha\frac{1}{\c_2\c_r^{(s)}}  \mE_{{\mathcal U}_{r+1}} \log  \lp \dots \mE_{{\mathcal U}_4}  \lp \mE_{{\mathcal U}_3}  \lp   \frac{1}{2}\erfc\lp -\frac{\eta_3+\kappa}{\sqrt{2}\sqrt{1-\p_2}} \rp  
 -
 \frac{1}{2}\erfc\lp -\frac{\eta_3-\kappa}{\sqrt{2}\sqrt{1-\p_2}} \rp 
 \rp^{\c_3^{(s)}} \rp^{\frac{\c_4^{(s)}}{\c_3^{(s)}}} \dots \rp .
    \end{align}
Following (\ref{eq:algnegprac24a6a0}) one defines
\begin{align}\label{eq:algalgnegprac24a6a0}
    & \bar{\psi}_{rd}^{(r,s)}(\p,\q,\c,\gamma_{sq})
\hspace{-.0in}   \triangleq  \hspace{-.0in}
       \frac{1}{2}
\sum_{k=2}^{r+1} (\p_{k-1}\q_{k-1}^{(s)}-\p_k\q_k^{(s)})\c_k^{(s)}
\nonumber \\
& 
  - \frac{1}{\c_r^{(s)}}\mE_{{\mathcal U}_{r+1}}\log  \lp \dots \mE_{{\mathcal U}_4}   \lp \mE_{{\mathcal U}_3}   \lp
  2\cosh \lp \zeta_r^{(s)} \rp
     \rp^{\c_3^{(s)}}  \rp^{\frac{\c_4^{(s)}}{\c_3^{(s)}}} \dots \rp
  \nonumber \\
& 
 -\alpha\frac{1}{\c_r^{(s)}}  \mE_{{\mathcal U}_{r+1}} \log  \lp \dots \mE_{{\mathcal U}_4}  \lp \mE_{{\mathcal U}_3}  \lp   \frac{1}{2}\erfc\lp -\frac{\eta_r+\kappa}{\sqrt{2}\sqrt{1-\p_2}} \rp 
  -
  \frac{1}{2}\erfc\lp -\frac{\eta_r-\kappa}{\sqrt{2}\sqrt{1-\p_2}} \rp 
  \rp^{\c_3^{(s)}} \rp^{\frac{\c_4^{(s)}}{\c_3^{(s)}}} \dots \rp  ,
    \end{align}
and recognizes that, for any $r\geq 2$, $\bar{\psi}_{rd}^{(r,s)}(\p,\q,\c,\gamma_{sq}) = 0$ replaces the critical capacity condition $f_{sq}(\infty)=0$. Numerical evaluations give
\begin{equation}\label{eq:algnegprac25a0}
\hspace{-.1in}(\mbox{\textbf{4th -- 7th level:}}) \quad  \alpha_c^{(4)}(1) \approx
 \red{\mathbf{1.6218}}, \quad  \alpha_c^{(5)}(1) \approx
 \red{\mathbf{1.6093}}, \quad  \alpha_c^{(6)}(1) \approx
 \red{\mathbf{1.6041}}, \quad  \alpha_c^{(7)}(1) \approx
 \red{\mathbf{1.6021}}.
  \end{equation}
All parametric values for the first seven lifting levels (1,2,3,4,5,6,7-sfl-RDT) are given in Table \ref{tab:tab3}.
\begin{table}[h]
\caption{$r$-sfl RDT parameters  ($r\leq 7$);   $\hat{\c}_2\rightarrow \infty$;   $\hat{\c}_k^{(s)} = \lim_{\hat{\c}_2\rightarrow\infty}\frac{\hat{\c}_k}{\hat{\c}_2},\hat{\q}_k^{(s)} =\lim_{\hat{\c}_2\rightarrow\infty} \hat{\q}_k\hat{\c}_2^2,k\geq 2$; $\kappa=1$; $n,\beta\rightarrow\infty$; $\hat{\p}_1\rightarrow 1$, $\hat{\q}_1\rightarrow 1$; $\hat{\p}_{(r)}= \hat{\p}_{2:r}$, $\hat{\q}_{(r)}^{(s)}= \hat{\q}_{2:r}^{(s)}$, $\hat{\c}_{(r)}^{(s)}= \hat{\c}_{2:r}^{(s)}$, $r\geq 2$; $\hat{\p}_{(1)}=\hat{\q}_{(1)}^{(s)}=\hat{\c}_{(1)}^{(s)}=0$ }\vspace{.1in}
%\begin{adjustwidth}{-1.4cm}{}
\centering
\def\arraystretch{1.2}
 \begin{tabular}{||c||c||c||c||c||c||c||c||}\hline\hline
 \hspace{-0in}$r$                                              & $1$    &  $2$   &  $3$   &  $4$   &  $5$   &  $6$   &  $7$    \\ \hline\hline
$\hat{\gamma}_{sq}$                                       &  $\begin{bmatrix}0.3989 \end{bmatrix}^{\large{ ^{\large{^{ }}}}}_{\large{ _{\large{_{ }}}}}$
& $\begin{bmatrix}0 \end{bmatrix}$  & $\begin{bmatrix} 0 \end{bmatrix}$  & $\begin{bmatrix}0 \end{bmatrix}$  & $\begin{bmatrix}0 \end{bmatrix}$  & $\begin{bmatrix}0 \end{bmatrix}$  & $\begin{bmatrix}0 \end{bmatrix}$
 \\ \hline \hline
  $\hat{\p}_{(r)}^T$                                       & $\begin{bmatrix}0 \end{bmatrix}$
& $\begin{bmatrix}0 \end{bmatrix}$  & $\begin{bmatrix} 0.9852\\0 \end{bmatrix}$  & $\begin{bmatrix} 0.9988 \\ 0.9729 \\0 \end{bmatrix}$  & $\begin{bmatrix} 0.99853 \\ 0.99655 \\ 0.96270 \\ 0 \end{bmatrix}$  & $\begin{bmatrix} 0.99996 \\ 0.99933 \\ 0.99240 \\ 0.94600 \\ 0 \end{bmatrix}$  & $\begin{bmatrix}0.999992 \\ 0.999810 \\ 0.997800 \\ 0.986000 \\ 0.915000 \\ 0 \end{bmatrix}$
 \\ \hline\hline
    $\lp\hat{\q}_{(r)}^{(s)}\rp^T$                                      & $\begin{bmatrix}0 \end{bmatrix}$
& $\begin{bmatrix}0 \end{bmatrix}$  & $\begin{bmatrix} 0.8794 \\0 \end{bmatrix}$  & $\begin{bmatrix}1.1211 \\ 0.0760 \\0 \end{bmatrix}$  & $\begin{bmatrix} 1.2800 \\ 0.0796 \\ 0.0103\\ 0 \end{bmatrix}$  & $\begin{bmatrix}1.4200 \\ .0.0920 \\ 0.0104 \\ 0 \end{bmatrix}$  & $\begin{bmatrix}1.52000 \\ 0.09900 \\ 0.01130 \\ 0.00217 \\ 0.00050 \\ 0 \end{bmatrix}$
 \\ \hline \hline
    $\lp\hat{\c}_{(r)}^{(s)}\rp^T$     & $\begin{bmatrix}0 \end{bmatrix}$
& $\begin{bmatrix}1 \end{bmatrix}$  & $\begin{bmatrix} 4.2629 \\ 1 \end{bmatrix}$  & $\begin{bmatrix}4.1522 \\ 12.0687 \\1 \end{bmatrix}$  & $\begin{bmatrix} 4.3528 \\ 12.7310 \\ 29.6479\\  1 \end{bmatrix}$  & $\begin{bmatrix} 4.430 \\ 12.70 \\ 31.70 \\ 59.50 \\ 1  \end{bmatrix}$  & $\begin{bmatrix} 4.490 \\ 12.90 \\ 31.20 \\ 65.00 \\ 104.8 \\ 1  \end{bmatrix}$
 \\ \hline \hline
$\alpha_c^{(r)}(1)$  & \bl{$\mathbf{4.2250}$}  & \bl{$\mathbf{1.8159}$}  & \red{$\mathbf{1.6576}$}  & \red{$\mathbf{1.6218}$}  & \red{$\mathbf{1.6093}$}  & \red{$\mathbf{1.6041}$}  & \red{$\mathbf{1.6021}$}  \\ \hline\hline
  \end{tabular}
%\end{adjustwidth}
\label{tab:tab3}
\end{table}
As the lifting level, $r$, increases the evaluations are more cumbersome and achieving a full numerical precision might be difficult. Consequently, the parametric results given in Table \ref{tab:tab3} may not be fully accurate. However,
conducting numerical work also allowed us to observe that $\alpha_c^{(r)}(1)$ estimates may not be overly sensitive to parameters changes. This suggests that even if some of the values given in Table \ref{tab:tab3} may need tiny adjustments, we would not expect them to significantly impact given $\alpha_c^{(r)}(1)$ estimates.   

Continuing evaluations beyond the $7$th level does not seem feasible at present. Still, the obtained results up to the $7$th level suffice to make the following observations: (\textbf{\emph{i}}) the values in the last row of the table suggest that the sequence $\alpha_c^{(r)}(1)$ might be converging (we believe that if that is the case the converging value is somewhere in $1.59-1.60$ range); (\textbf{\emph{ii}}) compared to ABP, the convergence (if present) seems a bit slower; (\textbf{\emph{iii}}) the obtained value $\alpha_c^{(7)}(1) \approx 1.6021$  and the predicted convergence range seem to closely approach the local entropy atypical solutions clustering defragmentation prediction $\alpha_{LE}\approx 1.58$ \cite{Bald20}. All of these observations together with  results of \cite{Stojnicalgbp25,Stojniccluphop25}  and concrete algorithmic considerations that will follow in later sections suggest that in the limit of large $r$ the above sequence of estimates, $\alpha_c^{(r)}(1)$, might very well approximate the algorithmic SBP capacity.

The following property is also worth noting. Namely, repeating the above calculations relying on modulo-$\m$ concepts from  \cite{Stojnicflrdt23} produced exactly the same results as those in Tables \ref{tab:tab1}--\ref{tab:tab3} which  uncovers that the $\c$ \emph{stationarity} is of \emph{maximization} type and thereby remarkably matches the very same behavior observed in \cite{Stojnicbinperflrdt23,Stojnicnegsphflrdt23,Stojnicalgbp25}.

%%%%%%%%%%%%%%%%%%%%%%%%%%%%%%%%%%%%%%%%%%%%%%%%%%%%%%%%%%%%%%%%%%%%%%%%%%%%%%%%%%%%%%%%%%%%%%%%%%%%%%%%%%%%%%%%%%%%%%%%%
%%%%%%%%%%%%%%%%%%%%%%%%%%%%%%%%%%%%%%%%%%%%%%%%%%%%%%%%%%%%%%%%%%%%%%%%%%%%%%%%%%%%%%%%%%%%%%%%%%%%%%%%%%%%%%%%%%%%%%%%%
%%%%%%%%%%%%%%%%%%%%%%%%%%%%%%%%%%%%%%%%%%%%%%%%%%%%%%%%%%%%%%%%%%%%%%%%%%%%%%%%%%%%%%%%%%%%%%%%%%%%%%%%%%%%%%%%%%%%%%%%%
%%%%%%%%%%%%%%%%%%%%%%%%%%%%%%%%%%%%%%%%%%%%%%%%%%%%%%%%%%%%%%%%%%%%%%%%%%%%%%%%%%%%%%%%%%%%%%%%%%%%%%%%%%%%%%%%%%%%%%%%%
%%%%%%%%%%%%%%%%%%%%%%%%%%%%%%%%%%%%%%%%%%%%%%%%%%%%%%%%%%%%%%%%%%%%%%%%%%%%%%%%%%%%%%%%%%%%%%%%%%%%%%%%%%%%%%%%%%%%%%%%%
\section{Numerical evaluations -- low $\alpha,\kappa$ regime}
\label{sec:numericalsmallkappa}
%%%%%%%%%%%%%%%%%%%%%%%%%%%%%%%%%%%%%%%%%%%%%%%%%%%%%%%%%%%%%%%%%%%%%%%%%%%%%%%%%%%%%%%%%%%%%%%%%%%%%%%%%%%%%%%%%%%%%%%%%
%%%%%%%%%%%%%%%%%%%%%%%%%%%%%%%%%%%%%%%%%%%%%%%%%%%%%%%%%%%%%%%%%%%%%%%%%%%%%%%%%%%%%%%%%%%%%%%%%%%%%%%%%%%%%%%%%%%%%%%%%
%%%%%%%%%%%%%%%%%%%%%%%%%%%%%%%%%%%%%%%%%%%%%%%%%%%%%%%%%%%%%%%%%%%%%%%%%%%%%%%%%%%%%%%%%%%%%%%%%%%%%%%%%%%%%%%%%%%%%%%%%
%%%%%%%%%%%%%%%%%%%%%%%%%%%%%%%%%%%%%%%%%%%%%%%%%%%%%%%%%%%%%%%%%%%%%%%%%%%%%%%%%%%%%%%%%%%%%%%%%%%%%%%%%%%%%%%%%%%%%%%%%
%%%%%%%%%%%%%%%%%%%%%%%%%%%%%%%%%%%%%%%%%%%%%%%%%%%%%%%%%%%%%%%%%%%%%%%%%%%%%%%%%%%%%%%%%%%%%%%%%%%%%%%%%%%%%%%%%%%%%%%%%

Results presented in Section \ref{sec:numerical} are conceptually sufficient to conduct all necessary numerical evaluations for any $\kappa>0$. Specialization $\kappa=1$ was taken so that a comprehensive set of concrete numerical values can be provided. These concrete values (presented in Tables \ref{tab:tab1}-\ref{tab:tab3}) are complemented by Figure \ref{fig:fig1} where $\alpha_c^{(r)}(\kappa)$ estimates are shown for a much wider range of $\kappa$. The first two curves in Figure \ref{fig:fig1}  (those that correspond to $\alpha_c^{(1)}(\kappa)$ and $\alpha_c^{(2)}(\kappa)$) are obtained for full range $\kappa\in (0,3)$. Each of the remaining two ($\alpha_c^{(3)}(\kappa)$ and $\alpha_c^{(4)}(\kappa)$) is obtained as a combination of two parts. The first part that corresponds to $\kappa \geq 0.3$ is obtained through the mechanism outlined in Section \ref{sec:numerical}. On the other hand, the second one is obtained through an approximation tailored for low $\kappa$ (and consequently low $\alpha$) regime. Such an approach is undertaken since for smaller $\kappa$ the optimal values of certain parameters may become sufficiently large to jeopardize reliability of residual numerical integrations.

Before we get to the details of these approximations, we will find it useful to first revisit $\alpha,\kappa\rightarrow 0$ regime (in recent literature \cite{GamKizPerXu22,GamKizPerXu23,Barb24,BarbAKZ23,BanSpen20} many excellent algorithmic and theoretical results were obtained in this regime which significantly deepened overall understanding of SCGs and provided a strong reassurance for their existence).  Once we gain a strong control of small $\kappa$  regime, the approximation forms will naturally follow.

%%%%%%%%%%%%%%%%%%%%%%%%%%%%%%%%%%%%%%%%%%%%%%%%%%%%%%%%%%%%%%%%%%%%%%%%%%%%%%%%%%%%%%%%%%%%%%%%%%%%%%%%%%%%%%%%%%%%%%%%%
%%%%%%%%%%%%%%%%%%%%%%%%%%%%%%%%%%%%%%%%%%%%%%%%%%%%%%%%%%%%%%%%%%%%%%%%%%%%%%%%%%%%%%%%%%%%%%%%%%%%%%%%%%%%%%%%%%%%%%%%%
\subsection{Small $\kappa$ regime -- third level of lifting ($r=3$)}
\label{sec:numericalskap3}
%%%%%%%%%%%%%%%%%%%%%%%%%%%%%%%%%%%%%%%%%%%%%%%%%%%%%%%%%%%%%%%%%%%%%%%%%%%%%%%%%%%%%%%%%%%%%%%%%%%%%%%%%%%%%%%%%%%%%%%%%
%%%%%%%%%%%%%%%%%%%%%%%%%%%%%%%%%%%%%%%%%%%%%%%%%%%%%%%%%%%%%%%%%%%%%%%%%%%%%%%%%%%%%%%%%%%%%%%%%%%%%%%%%%%%%%%%%%%%%%%%%
 
We start with the third elvel of lifting as that is that lowest level where natural decreasing ordering of $\c$ sequence is abandoned. 
Recalling on (\ref{eq:algnegprac24a6a0}) we write
\begin{eqnarray}\label{eq:skap1}
   \bar{\psi}_{rd}^{(3,s)}(\p,\q,\c,\gamma_{sq})
\hspace{-.0in}  & \triangleq & \hspace{-.0in}
  \frac{1}{2}
(1-\p_2)\q_2^{(s)}
 +  \frac{1}{2}
\lp \p_2\q_2^{(s)}-\p_3\q_3^{(s)} \rp\c_3^{(s)}
  - \frac{1}{\c_3^{(s)}}\mE_{{\mathcal U}_4}\log  \lp \mE_{{\mathcal U}_3}   \lp
  2\cosh \lp \zeta_3^{(s)} \rp
     \rp^{\c_3^{(s)}}  \rp
  \nonumber \\
\hspace{-.0in} & & \hspace{-.0in}
 -\alpha\frac{1}{\c_3^{(s)}}  \mE_{{\mathcal U}_4} \log  \lp \mE_{{\mathcal U}_3}  \lp  
  \frac{1}{2}\erfc\lp -\frac{\eta_3+\kappa}{\sqrt{2}\sqrt{1-\p_2}} \rp  
  -
    \frac{1}{2}\erfc\lp -\frac{\eta_3-\kappa}{\sqrt{2}\sqrt{1-\p_2}} \rp 
  \rp^{\c_3^{(s)}} \rp , \nonumber \\
    \end{eqnarray}
where as in (\ref{eq:algnegprac24a4a0}) and (\ref{eq:algnegprac24a2})
\begin{eqnarray}\label{eq:skap2}
\zeta_3^{(s)}  = \sqrt{\q_2^{(s)} - \q_3^{(s)} }\h_1^{(3)}   +  \sqrt{\q_3^{(s)}}\h_1^{(4)} .
\end{eqnarray}
and
\begin{eqnarray}\label{eq:skap3}
\eta_3 & = & \sqrt{\p_2-\p_3}\u_1^{(2,3)} + \sqrt{\p_3}\u_1^{(2,4)}  .
\end{eqnarray}
Given that we are operating in partial lifting mode, $\p_3=\q_3=0$ and (\ref{eq:skap1})-(\ref{eq:skap3}) can be rewritten as 
\begin{align}\label{eq:skap4}
   \bar{\psi}_{rd}^{(3,s)}(\p,\q,\c,\gamma_{sq})
\hspace{-.0in}  & =  \hspace{-.0in}
  \frac{1}{2}
(1-\p_2)\q_2^{(s)}
-  \frac{1}{2}
(1-\p_2)\q_2^{(s)}\c_3^{(s)}
+ \frac{1}{2} \q_2^{(s)}\c_3^{(s)}
\nonumber \\
&   - \frac{1}{\c_3^{(s)}} \log  \lp \mE_{{\mathcal U}_3}   \lp
  2\cosh \lp \sqrt{\q_2^{(s)} }\h_1^{(3)}  \rp
     \rp^{\c_3^{(s)}}  \rp
  \nonumber \\
  &  
 -\alpha\frac{1}{\c_3^{(s)}}   \log  \lp \mE_{{\mathcal U}_3}  \lp  
  \frac{1}{2}\erfc\lp -\frac{ \sqrt{\p_2}\u_1^{(2,3)}  +\kappa}{\sqrt{2}\sqrt{1-\p_2}} \rp  
  -
    \frac{1}{2}\erfc\lp -\frac{ \sqrt{\p_2}\u_1^{(2,3)}   -\kappa}{\sqrt{2}\sqrt{1-\p_2}} \rp 
  \rp^{\c_3^{(s)}} \rp , \nonumber \\
    \end{align}
For $\alpha,\kappa\rightarrow 0$ one has $\p_2\rightarrow 1$ with $1-\p_2\sim\frac{\alpha}{-\log \lp \alpha \rp}$ and $\kappa\sim \sqrt{\frac{\alpha}{-\log \lp \alpha \rp}}$ which implies that the saddle point last term most outer integration gives $\u_1^{(2,3)}=0$. This allows to rewrite (\ref{eq:skap4}) as
\begin{align}\label{eq:skap5}
   \bar{\psi}_{rd}^{(3,s)}(\p,\q,\c,\gamma_{sq})
  & \rightarrow  \hspace{-.0in}
  \frac{1}{2}
(1-\p_2)\q_2^{(s)}
-  \frac{1}{2}
(1-\p_2)\q_2^{(s)}\c_3^{(s)}
+ \frac{1}{2} \q_2^{(s)}\c_3^{(s)}
\nonumber \\
& \hspace{.2in}  - \frac{1}{\c_3^{(s)}}  \log  \lp \mE_{{\mathcal U}_3}   \lp
  2\cosh \lp \sqrt{\q_2^{(s)} }\h_1^{(3)}  \rp
     \rp^{\c_3^{(s)}}  \rp
  -\alpha  \log   \lp  
   \erf\lp \frac{ \kappa}{\sqrt{2}\sqrt{1-\p_2}} \rp   
  \rp  \nonumber \\
  &  \rightarrow  \hspace{-.0in}
  \frac{1}{2}
(1-\p_2)\q_2^{(s)}
-  \frac{1}{2}
(1-\p_2)\q_2^{(s)}\c_3^{(s)}
+ \frac{1}{2} \q_2^{(s)}\c_3^{(s)}
\nonumber \\
& \hspace{.2in}  - \frac{1}{\c_3^{(s)}}  \log  \lp \mE_{{\mathcal U}_3}  e^{ \c_3^{(s)}\log \lp
  2\cosh \lp \sqrt{\q_2^{(s)} }\h_1^{(3)}  \rp 
     \rp  }  \rp
  -\alpha  \log   \lp  
   \erf\lp \frac{ \kappa}{\sqrt{2}\sqrt{1-\p_2}} \rp   
  \rp .
    \end{align}
We also note
\begin{align} \label{eq:skap6}
\mE_{{\mathcal U}_3}  e^{ \c_3^{(s)}\log \lp
  2\cosh \lp \sqrt{\q_2^{(s)} }\h_1^{(3)}  \rp 
     \rp  } 
&     =
 \frac{1}{\sqrt{2\pi}}\int e^{ \c_3^{(s)}\log \lp
  2\cosh \lp \sqrt{\q_2^{(s)} }\h_1^{(3)}  \rp 
     \rp -\frac{  \lp  \h_1^{(3)}  \rp^2  }{2} } 
     d \h_1^{(3)}  
     \nonumber \\
      &
      =
 \frac{1}{\sqrt{2\pi}}\int e^{ \frac{\c_3^{(s)}   \q_2^{(s)}  } {\q_2^{(s)} }\log \lp
  2\cosh \lp \sqrt{\q_2^{(s)} }\h_1^{(3)}  \rp 
     \rp -\frac{1}{\q_2^{(s)} }\frac{  \lp  \h_1^{(3)} \sqrt{\q_2^{(s)} } \rp^2  }{2} } 
     d \h_1^{(3)} .
\end{align}
For $\alpha,\kappa\rightarrow 0$, $\q_2^{(s)} \rightarrow 0$, and we set $c_x=\c_3^{(s)}   \q_2^{(s)} $. The saddle point integration then gives
\begin{align} \label{eq:skap7}
 \frac{1}{\c_3^{(s)}}  \log  \mE_{{\mathcal U}_3}  e^{ \c_3^{(s)}\log \lp
  2\cosh \lp \sqrt{\q_2^{(s)} }\h_1^{(3)}  \rp 
     \rp  } 
       &
      \rightarrow 
\frac{1}{\c_3^{(s)}}  \log  e^{   \frac{1} {\q_2^{(s)} } 
\lp
\max_{h_x} \lp  c_x \log \lp
  2\cosh \lp h_x \rp 
     \rp - \frac{  h_x^2  }{2} 
     \rp \rp
     } .
\end{align}
After taking the derivative one finds
\begin{equation}\label{eq:skap8}
  c_x\tanh(h_x) -h_x =0.
\end{equation}
Since it will turn out that $c_x\sim -\log(1-\p_2)\sim -\log\lp \frac{\alpha}{-\log(\alpha)}\rp$ (i.e. that $c_x$ grows as $\alpha,\kappa\rightarrow 0$), one gets $h_x=c_x$ for the saddle point.  Plugging this back in (\ref{eq:skap7}) gives
\begin{align} \label{eq:skap9}
 \frac{1}{\c_3^{(s)}}  \log  \mE_{{\mathcal U}_3}  e^{ \c_3^{(s)}\log \lp
  2\cosh \lp \sqrt{\q_2^{(s)} }\h_1^{(3)}  \rp 
     \rp  } 
       &
      \rightarrow 
\frac{1}{\c_3^{(s)}}  \log  e^{   \frac{1} {\q_2^{(s)} } 
 \lp  c_x \log \lp
  2\cosh \lp c_x \rp 
     \rp - \frac{  c_x^2  }{2} 
     \rp 
     } 
     \nonumber \\
&            \rightarrow 
\frac{1}{\c_3^{(s)} \q_2^{(s)} }   
  \lp  c_x \log \lp
   e^{c_x}+e^{-c_x}  
     \rp - \frac{  c_x^2  }{2} 
     \rp 
          \nonumber \\
     & 
           \rightarrow 
\frac{1}{\c_3^{(s)} \q_2^{(s)}  }     
  \lp  c_x^2 + c_x\log \lp
   1+e^{-2c_x}  
     \rp - \frac{  c_x^2  }{2} 
     \rp 
               \nonumber \\
     & 
           \rightarrow 
\frac{1}{\c_3^{(s)} \q_2^{(s)} } 
  \lp  \frac{  c_x^2  }{2} +  c_x \log \lp
   1+e^{-2c_x}  
     \rp  
     \rp 
               \nonumber \\
     & 
           \rightarrow 
\frac{c_x}{2}   
+   \log \lp
   1+e^{-2c_x}  
     \rp  
               \nonumber \\
     & 
           \rightarrow 
\frac{c_x}{2}   
+   \ e^{-2c_x}  
      .
\end{align}
Combining (\ref{eq:skap5}) and  (\ref{eq:skap9}) we obtain
\begin{align}\label{eq:skap10}
   \bar{\psi}_{rd}^{(3,s)}(\p,\q,\c,\gamma_{sq})
   &  \rightarrow  \hspace{-.0in}
  \frac{1}{2}
(1-\p_2)\q_2^{(s)}
-  \frac{1}{2}
(1-\p_2)\q_2^{(s)}\c_3^{(s)}
+ \frac{1}{2} \q_2^{(s)}\c_3^{(s)}
\nonumber \\
& \hspace{.2in}  -  \frac{c_x}{2}  - e^{-2c_x}
  -\alpha  \log   \lp  
   \erf\lp \frac{ \kappa}{\sqrt{2}\sqrt{1-\p_2}} \rp   
  \rp 
   \nonumber \\
   &  \rightarrow  \hspace{-.0in}
 -  \frac{1}{2}
(1-\p_2)c_x
  - e^{-2c_x}
  -\alpha  \log   \lp  
   \erf\lp \frac{ \kappa}{\sqrt{2}\sqrt{1-\p_2}} \rp   
  \rp 
      .
    \end{align}
Finding the derivatives with respect to $c_x$ and $\p_2$ gives the following stationary points equations
\begin{eqnarray}\label{eq:skap11a0}
c_x &  = & -
\frac{1}{2}\log \lp \frac{1 -\p_2}{4}\rp 
\nonumber \\
 \frac{c_x}{2}  & = &   \frac{\alpha\,\kappa \,{\mathrm{e}}^{-\frac{\kappa ^2}{2(1-\p_2)}}}{\sqrt{2\pi }\,\mathrm{erf}\left(\frac{\kappa }{\sqrt{2(1-\p_2)}}\right)\,{\left(1-\p_2\right)}^{3/2}} 
  .
 \end{eqnarray}
One then easily finds
\begin{eqnarray}\label{eq:skap11}
  -
\frac{1}{4}\log \lp \frac{1 -\p_2}{4}\rp 
 & = &   \frac{\alpha\,\kappa \,{\mathrm{e}}^{-\frac{\kappa ^2}{2(1-\p_2)}}}{\sqrt{2\pi }\,\mathrm{erf}\left(\frac{\kappa }{\sqrt{2(1-\p_2)}}\right)\,{\left(1-\p_2\right)}^{3/2}} 
  .
 \end{eqnarray}
A combination of (\ref{eq:skap10})--(\ref{eq:skap11}) then gives that for a small $\alpha$ critical $\kappa$ is determined as the the minimal $\kappa$ such that
\begin{align}\label{eq:skap12}
   \bar{\psi}_{rd}^{(3,s)}(\p,\q,\c,\gamma_{sq})
    &  \rightarrow  \hspace{-.0in}
   \frac{1}{4}
(1- \hat{\p}_2) \log\lp \frac{1- \hat{\p}_2}{4} \rp
  - \frac{1- \hat{\p}_2}{4} 
  -\alpha  \log   \lp  
   \erf\lp \frac{ \kappa}{\sqrt{2}\sqrt{1-\hat{\p}_2}} \rp   
  \rp =0  ,
    \end{align}
with $\hat{\p}_2$ satisfying (\ref{eq:skap11}). This describes how the lower portion of $\alpha_c^{(3)}(\kappa) $ (3-sfl-RDT) curve in Figure \ref{fig:fig1} was obtained. For $\alpha,\kappa\rightarrow 0$ the above results are actually exact. On the other hand, for $\alpha,\kappa$ small (but not necessarily  $\alpha,\kappa\rightarrow 0$) one can take them as an approximation. As Figure \ref{fig:fig1} shows, continuation of $\alpha_c^{(3)}(\kappa) $ curve from $\kappa\geq 0.3$ range (which is determined without any approximations) to $\kappa\leq 0.3$ range (which is determined through the above approximation) seems to be smooth. This, on the other hand, suggests that the utilized approximation might be fairly accurate even for $\kappa$ a bit away from zero.

%%%%%%%%%%%%%%%%%%%%%%%%%%%%%%%%%%%%%%%%%%%%%%%%%%%%%%%%%%%%%%%%%%%%%%%%%%%%%%%%%%%%%%%%%%%%%%%%%%%%%%%%%%%%%%%%%%%%%%%%%
%%%%%%%%%%%%%%%%%%%%%%%%%%%%%%%%%%%%%%%%%%%%%%%%%%%%%%%%%%%%%%%%%%%%%%%%%%%%%%%%%%%%%%%%%%%%%%%%%%%%%%%%%%%%%%%%%%%%%%%%%
\subsubsection{$\alpha,\kappa\rightarrow 0$ -- $r=3$}
\label{sec:numericalskap30}
%%%%%%%%%%%%%%%%%%%%%%%%%%%%%%%%%%%%%%%%%%%%%%%%%%%%%%%%%%%%%%%%%%%%%%%%%%%%%%%%%%%%%%%%%%%%%%%%%%%%%%%%%%%%%%%%%%%%%%%%%
%%%%%%%%%%%%%%%%%%%%%%%%%%%%%%%%%%%%%%%%%%%%%%%%%%%%%%%%%%%%%%%%%%%%%%%%%%%%%%%%%%%%%%%%%%%%%%%%%%%%%%%%%%%%%%%%%%%%%%%%%

The above results can be additionally simplified when $\alpha,\kappa\rightarrow 0$. To that end we set 
\begin{eqnarray}
\label{eq:skap13}  
1-\p_2  &=& \frac{p_x\alpha}{-\log(\alpha)} \nonumber \\
  \kappa &=& \kappa_x\sqrt{2(1-\p_2)} = \kappa_x \sqrt{2 p_x} \sqrt{\frac{\alpha}{-\log(\alpha)}}.
\end{eqnarray}
One can then rewrite (\ref{eq:skap12}) as 
\begin{align}\label{eq:skap14}
   \bar{\psi}_{rd}^{(3,s)}(\p,\q,\c,\gamma_{sq})
    &  \rightarrow  \hspace{-.0in}
   \frac{1}{4}
\frac{p_x\alpha}{-\log(\alpha)} \log\lp \frac{\frac{p_x\alpha}{-\log(\alpha)}}{4} \rp
  - \frac{\frac{p_x\alpha}{-\log(\alpha)}}{4} 
  -\alpha  \log   \lp  
   \erf\lp \kappa_x \rp   
  \rp 
  \nonumber \\
    &  \rightarrow  \hspace{-.0in}
   \frac{1}{4}
\frac{p_x\alpha}{-\log(\alpha)} \log\lp \alpha \rp
+
   \frac{1}{4}
\frac{p_x\alpha}{-\log(\alpha)} \log\lp \frac{\frac{p_x}{-\log(\alpha)}}{4} \rp
  - \frac{\frac{p_x\alpha}{-\log(\alpha)}}{4} 
  -\alpha  \log   \lp  
   \erf\lp \kappa_x \rp   
  \rp 
  \nonumber \\
    &  \rightarrow  \hspace{-.0in}
   -\frac{1}{4}
p_x\alpha
   -\alpha  \log   \lp  
   \erf\lp \kappa_x \rp   
  \rp 
  \nonumber \\
    &  \rightarrow  \hspace{-.0in}
   -\frac{p_x }{4}
   -   \log   \lp  
   \erf\lp \kappa_x \rp   
  \rp 
  =0  ,
    \end{align}
Moreover, (\ref{eq:skap11}) can be rewritten as
\begin{eqnarray}\label{eq:skap15}
  -
\frac{1}{4}\log \lp \frac{\frac{p_x\alpha}{-\log(\alpha)}}{4}\rp 
 & = &   \frac{\alpha\,\kappa_x \,{\mathrm{e}}^{-\kappa_x ^2   }  }{\sqrt{\pi }\,\mathrm{erf}\left( \kappa_x \right)\,{\left(1-\p_2\right)}} 
 =
  \frac{\alpha\,\kappa_x \,{\mathrm{e}}^{-\kappa_x ^2   }  }{\sqrt{\pi }\,\mathrm{erf}\left( \kappa_x \right)\,{\left(  \frac{p_x\alpha}{-\log(\alpha)}   \right)}} 
 =
-  \frac{\log(\alpha)\,\kappa_x \,{\mathrm{e}}^{-\kappa_x ^2   }  }{\sqrt{\pi }\,\mathrm{erf}\left( \kappa_x \right)\,{p_x    }} 
  .
 \end{eqnarray}
 From (\ref{eq:skap15}) we further have
\begin{eqnarray}\label{eq:skap16}
\frac{p_x}{4}\log \lp \alpha \rp 
+
\frac{p_x}{4}\log \lp \frac{\frac{p_x}{-\log(\alpha)}}{4}\rp 
  =
  \frac{\log(\alpha)\,\kappa_x \,{\mathrm{e}}^{-\kappa_x ^2   }  }{\sqrt{\pi }\,\mathrm{erf}\left( \kappa_x \right)} 
  ,
 \end{eqnarray}
 and
\begin{eqnarray}\label{eq:skap17}
\frac{p_x}{4}
-  \frac{ \kappa_x \,{\mathrm{e}}^{-\kappa_x ^2   }  }{\sqrt{\pi }\,\mathrm{erf}\left( \kappa_x \right)} 
 \rightarrow 0 .
 \end{eqnarray}
A combination of (\ref{eq:skap14})  and (\ref{eq:skap17})  then gives
\begin{eqnarray}\label{eq:skap18}
   -   \log   \lp  
   \erf\lp \kappa_x \rp   
  \rp 
-  \frac{ \kappa_x \,{\mathrm{e}}^{-\kappa_x ^2   }  }{\sqrt{\pi }\,\mathrm{erf}\left( \kappa_x \right)} 
 \rightarrow 0 .
 \end{eqnarray}
Finally, from (\ref{eq:skap13}),  (\ref{eq:skap17}),  and (\ref{eq:skap18}) we find
\begin{eqnarray}
\label{eq:skap19}  
   \kappa  = \hat{\kappa}_x \sqrt{   -  8 \log   \lp  
   \erf\lp \hat{\kappa}_x \rp   
  \rp 
} \sqrt{\frac{\alpha}{-\log(\alpha)}},
\end{eqnarray}
where $\hat{\kappa}_x$ satisfies (\ref{eq:skap18}). Concrete numerical evaluations give 
\begin{eqnarray}
\label{eq:skap20}  
\hat{\kappa}_x \approx 0.7534 \quad \mbox{and} \quad 
   \kappa  = \hat{\kappa}_x \sqrt{   -  8 \log   \lp  
   \erf\lp \hat{\kappa}_x \rp   
  \rp 
} \sqrt{\frac{\alpha}{-\log(\alpha)}} 
\approx 
 1.2385  \sqrt{\frac{\alpha}{-\log(\alpha)}}.
\end{eqnarray}
This fully matches the local entropy based prediction for critical $\kappa$ behavior in  $\alpha\rightarrow 0$  regime obtained in \cite{BarbAKZ23} via replica method with 1RSB ansatz. Moreover, \cite{BarbAKZ23} also analyzed the local entropy via a fully rigorous contiguity approach of a planted model. Remarkably, their 1RSB estimate was shown to match the planted contiguity results. To be fully precise, the above result matches the so-called \emph{energetic} barrier of \cite{BarbAKZ23}, which assumes a complete entropy breakdown (i.e., no exponential size clusters at a certain distance in a vicinity of a solution). \cite{BarbAKZ23} also considers the so-called \emph{entropic} barrier and obtains slightly higher proportionality $\approx 1.4288$ (instead of $\approx 1.2385$). As discussed to great length in introductory local entropy papers \cite{Bald15,Bald16},  entropic barrier doesn't necessarily search for a complete entropy breakdown but rather for a clustering ``\emph{thinning out}''. This is reflected through the appearance of a lack of monotonicity of the entropic curve as one moves away from the referent solution. This effectively means that while there may be an exponential number of near solutions at any distance, it might still not suffice for certain classes of algorithms  to fully take advantage of them.

%%%%%%%%%%%%%%%%%%%%%%%%%%%%%%%%%%%%%%%%%%%%%%%%%%%%%%%%%%%%%%%%%%%%%%%%%%%%%%%%%%%%%%%%%%%%%%%%%%%%%%%%%%%%%%%%%%%%%%%%%
%%%%%%%%%%%%%%%%%%%%%%%%%%%%%%%%%%%%%%%%%%%%%%%%%%%%%%%%%%%%%%%%%%%%%%%%%%%%%%%%%%%%%%%%%%%%%%%%%%%%%%%%%%%%%%%%%%%%%%%%%
\subsection{Small $\kappa$ regime -- fourth level of lifting ($r=4$)}
\label{sec:numericalskap4}
%%%%%%%%%%%%%%%%%%%%%%%%%%%%%%%%%%%%%%%%%%%%%%%%%%%%%%%%%%%%%%%%%%%%%%%%%%%%%%%%%%%%%%%%%%%%%%%%%%%%%%%%%%%%%%%%%%%%%%%%%
%%%%%%%%%%%%%%%%%%%%%%%%%%%%%%%%%%%%%%%%%%%%%%%%%%%%%%%%%%%%%%%%%%%%%%%%%%%%%%%%%%%%%%%%%%%%%%%%%%%%%%%%%%%%%%%%%%%%%%%%%

The above discussion suggests that the proposed methodology already on the third lifting level might produce fairly good approximations of algorithmic thresholds for any $\kappa>0$. Below we uncover that the same tendency continues on the fourth level.

We start by recognizing that  a combination of (\ref{eq:algnegprac24a6a0}) and (\ref{eq:algalgnegprac24a6a0}) allows to write their fourth level analogue 
\begin{align}\label{eq:4skap1}
&  \hspace{-.3in} \bar{\psi}_{rd}^{(4,s)}(\p,\q,\c,\gamma_{sq})
\hspace{-.0in}   =  \hspace{-.0in}
  \frac{1}{2}
(1-\p_2)\q_2^{(s)}
 +  \frac{1}{2}
\lp \p_2\q_2^{(s)}-\p_3\q_3^{(s)} \rp\c_3^{(s)}
+  \frac{1}{2}
\lp \p_3\q_3^{(s)}-\p_4\q_4^{(s)} \rp\c_4^{(s)}
\nonumber \\
&    - \frac{1}{\c_4^{(s)}}\mE_{{\mathcal U}_5}\log  \lp  \mE_{{\mathcal U}_4}  \lp \mE_{{\mathcal U}_3}   \lp
  2\cosh \lp \zeta_4^{(s)} \rp
     \rp^{\c_3^{(s)}}  \rp^{\frac{\c_4^{(s)}}{\c_3^{(s)}}}    \rp
  \nonumber \\
\hspace{-.0in} & \hspace{-.0in}
 -\alpha\frac{1}{\c_4^{(s)}}  \mE_{{\mathcal U}_4} \log  \lp \mE_{{\mathcal U}_4}  \lp   \mE_{{\mathcal U}_3}  \lp  
  \frac{1}{2}\erfc\lp -\frac{\eta_4+\kappa}{\sqrt{2}\sqrt{1-\p_2}} \rp  
  -
    \frac{1}{2}\erfc\lp -\frac{\eta_4-\kappa}{\sqrt{2}\sqrt{1-\p_2}} \rp 
  \rp^{\c_3^{(s)}} \rp^{\frac{ \c_4^{(s)}  }{\c_3^{(s)}}}    \rp , \nonumber \\
    \end{align}
where similarly to  (\ref{eq:algnegprac24a4a0}) and (\ref{eq:algnegprac24a2}) (or (\ref{eq:skap2}) and  (\ref{eq:skap3})) 
\begin{eqnarray}\label{eq:4skap2}
\zeta_4^{(s)}  = \sqrt{\q_2^{(s)} - \q_3^{(s)} }\h_1^{(3)}   +  \sqrt{\q_3^{(s)}  - \q_4^{(s)}  }\h_1^{(4)}  +  \sqrt{ \q_4^{(s)}  }\h_1^{(5)} .
\end{eqnarray}
and
\begin{eqnarray}\label{eq:4skap3}
\eta_4 & = & \sqrt{\p_2-\p_3}\u_1^{(2,3)} + \sqrt{\p_3 - \p_4 }\u_1^{(2,4)}  + \sqrt{\p_5 }\u_1^{(2,5)}  .
\end{eqnarray}
Since we again operate in the partial lifting mode, $\p_4=\q_4=0$ and (\ref{eq:4skap1})-(\ref{eq:4skap3}) can be rewritten as 
\begin{eqnarray}\label{eq:4skap4} \hspace{-.0in} \bar{\psi}_{rd}^{(4,s)}(\p,\q,\c,\gamma_{sq})
  & =  &
     \frac{1}{2}
(1-\p_2)\q_2^{(s)}
 +  \frac{1}{2}
\lp \p_2\q_2^{(s)}-\p_3\q_3^{(s)} \rp\c_3^{(s)}
+  \frac{1}{2}
\lp \p_3\q_3^{(s)} \rp\c_4^{(s)}
\nonumber \\
&  &  - \frac{1}{\c_4^{(s)}} \log  \lp  \mE_{{\mathcal U}_4}  \lp \mE_{{\mathcal U}_3}   \lp
  2\cosh \lp \sqrt{\q_2^{(s)} - \q_3^{(s)} }\h_1^{(3)}   +  \sqrt{\q_3^{(s)}  }\h_1^{(4)}        \rp
     \rp^{\c_3^{(s)}}  \rp^{\frac{\c_4^{(s)}}{\c_3^{(s)}}}    \rp
  \nonumber \\
  & & 
 -\alpha\frac{1}{\c_4^{(s)}}    \log \Bigg .\Bigg   ( \mE_{{\mathcal U}_4}   \Bigg .\Bigg  (    \mE_{{\mathcal U}_3}  \Bigg .\Bigg  (  
  \frac{1}{2}\erfc\lp -\frac{ \sqrt{\p_2-\p_3}\u_1^{(2,3)} + \sqrt{\p_3 }\u_1^{(2,4)}   +\kappa}{\sqrt{2}\sqrt{1-\p_2}} \rp  
  \nonumber \\
  & & 
 -
    \frac{1}{2}\erfc\lp -\frac{ \sqrt{\p_2-\p_3}\u_1^{(2,3)} + \sqrt{\p_3 }\u_1^{(2,4)}   -\kappa}{\sqrt{2}\sqrt{1-\p_2}} \rp 
 \Bigg .\Bigg     ) ^{\c_3^{(s)}}    \Bigg .\Bigg  )  ^{\frac{ \c_4^{(s)}  }{\c_3^{(s)}}}  \Bigg .\Bigg   )  .
    \end{eqnarray}
As $\alpha,\kappa\rightarrow 0$ we have $\p_2,\p_3\rightarrow 1$, $\q_3^{(s)}\rightarrow 0$, $\q_4^{(s)}\rightarrow 0$ which with scaling $1-\p_2,\p_2-\p_3\sim\frac{\alpha}{-\log \lp \alpha \rp}$, $\q_2^{(s)} \sim -\log \lp \alpha \rp$, $\frac{c_4^{(3)}}{c_4^{(3)}}\rightarrow \infty$,  $\q_3^{(s)}\c_4^{(s)}=c_x$, and $\kappa\sim \sqrt{\frac{\alpha}{-\log \lp \alpha \rp}}$ allows to rewrite (\ref{eq:4skap4}) as
\begin{eqnarray}\label{eq:4skap5} 
\hspace{-.0in} \bar{\psi}_{rd}^{(4,s)}(\p,\q,\c,\gamma_{sq})
  & =  &
     \frac{1}{2}
(1-\p_2)\q_2^{(s)}
 +  \frac{1}{2}
\p_2\q_2^{(s)} \c_3^{(s)}
+  \frac{1}{2}
\p_3 c_x
\nonumber \\
&  &  - \frac{1}{\c_4^{(s)}} \log  \lp  \mE_{{\mathcal U}_4}  \lp \omega_q \lp \sqrt{\q_3^{(s)}  }\h_1^{(4)}   \rp  \rp^{\frac{\c_4^{(s)}}{\c_3^{(s)}}}    \rp
 - \frac{\alpha}{\c_4^{(s)}} \log  \lp  \mE_{{\mathcal U}_4}  \lp \omega_p \lp  \sqrt{\p_3 }\u_1^{(2,4)}  \rp  \rp^{\frac{\c_4^{(s)}}{\c_3^{(s)}}}    \rp , \nonumber \\
    \end{eqnarray}
where
\begin{eqnarray}
\label{eq:4skap6} 
 \omega_q(x) &=& \mE_{{\mathcal U}_3}   \lp
  2\cosh \lp \sqrt{\q_2^{(s)} - \q_3^{(s)} }\h_1^{(3)}   +  x       \rp  \rp^{\c_3^{(s)}} 
\rightarrow 
\mE_{{\mathcal U}_3}   \lp
  2\cosh \lp \sqrt{\q_2^{(s)} }\h_1^{(3)}   +  x       \rp  \rp^{\c_3^{(s)}} 
     \nonumber  \\
   \omega_p(x) &=&  \mE_{{\mathcal U}_3}  \Bigg .\Bigg  (  
  \frac{1}{2}\erfc\lp -\frac{ \sqrt{\p_2-\p_3}\u_1^{(2,3)} + x   +\kappa}{\sqrt{2}\sqrt{1-\p_2}} \rp  
  -
    \frac{1}{2}\erfc\lp -\frac{ \sqrt{\p_2-\p_3}\u_1^{(2,3)} + x   -\kappa}{\sqrt{2}\sqrt{1-\p_2}} \rp 
 \Bigg .\Bigg     ) ^{\c_3^{(s)}} .  
\end{eqnarray}
Utilizing the above scaling further we also find $\u_1^{(2,4)}=0$ for the saddle point of the last term outer integration which then gives
\begin{eqnarray}
\label{eq:4skap7} 
\frac{\alpha}{\c_4^{(s)}} \log \lp \mE_{{\mathcal U}_4}  \lp \omega_p \lp  \sqrt{\p_3 }\u_1^{(2,4)}  \rp  \rp^{\frac{\c_4^{(s)}}{\c_3^{(s)}}}   \rp
 &  =  &
\frac{\alpha}{\c_4^{(s)}} \log 
\lp
\frac{1}{\sqrt{2\pi}} \int e^{\frac{\c_4^{(s)}}{\c_3^{(s)}} \log\lp  \omega_p \lp  \sqrt{\p_3 }\u_1^{(2,4)}  \rp   \rp   
- \frac{1}{2}\lp \u_1^{(2,4)}\rp^2} d\u_1^{(2,4)} \rp
\nonumber \\
&  \rightarrow  & \log 
\frac{\alpha}{\c_4^{(s)}} \lp
  e^{\frac{\c_4^{(s)}}{\c_3^{(s)}} \log\lp  \omega_p \lp 0  \rp   \rp   
} \rp
\nonumber \\
&  \rightarrow  & 
\frac{\alpha}{\c_3^{(s)}} \log \lp   \omega_p \lp 0 \rp \rp.
\end{eqnarray}
Writing further
\begin{eqnarray}
\label{eq:4skap8} 
 \log  \lp  \mE_{{\mathcal U}_4}  \lp \omega_q \lp \sqrt{\q_3^{(s)}  }\h_1^{(4)}   \rp  \rp^{\frac{\c_4^{(s)}}{\c_3^{(s)}}}    \rp
&  =  &
  \log  \lp  \frac{1}{\sqrt{2\pi}}  \int e^{\frac{\c_4^{(s)}}{\c_3^{(s)}}\log \lp \omega_q \lp \sqrt{\q_3^{(s)}  }\h_1^{(4)}   \rp  \rp  -\frac{1}{2}\lp \h_1^{(4)} \rp^2 } \d\h_1^{(4)}  \rp  
  \nonumber   \\
&  =  &
  \log  \lp  \frac{1}{\sqrt{2\pi}}  \int e^{\frac{\c_x}{\q_3^{(s)}\c_3^{(s)}}\log \lp \omega_q \lp \sqrt{\q_3^{(s)}  }\h_1^{(4)}   \rp  \rp  -\frac{1}{2\q_3^{(s)}}\lp \sqrt{\q_3^{(s)}}\h_1^{(4)} \rp^2 } \d\h_1^{(4)}  \rp   \nonumber   \\
&  =  &
  \log  \lp  \frac{1}{\sqrt{2\pi \q_3^{(s)}  } }  \int e^{\frac{\c_x}{\q_3^{(s)}\c_3^{(s)}}\log \lp \omega_q \lp h  \rp  \rp  -\frac{1}{2\q_3^{(s)}}  h^2 } \d h  \rp ,
\end{eqnarray}
taking the exponent's derivative with respect to $h$, and equalling to zero gives
\begin{eqnarray}
\label{eq:4skap9} 
\frac{d\lp \frac{\c_c}{ \c_3^{(s)}}\log \lp \omega_q \lp h  \rp  \rp  -\frac{1}{2 }  h^2  \rp }{dh}
 & = &
  \frac{\c_x}{ \c_3^{(s)}}\frac{\frac{d\omega_q \lp h  \rp }{dh}}{ \lp \omega_q \lp h  \rp  \rp }  - h 
\nonumber \\
 & \rightarrow  &
   \c_x   \frac{  \mE_{{\mathcal U}_3}   \lp
  2\cosh \lp \sqrt{\q_2^{(s)} }\h_1^{(3)}   +  h \rp
  \tanh \lp \sqrt{\q_2^{(s)}   }\h_1^{(3)}   +  h       \rp
    \rp^{\c_3^{(s)}    }     }{ \lp \omega_q \lp h  \rp  \rp }  - h = 0 , \nonumber \\
\end{eqnarray}
A combination of (\ref{eq:4skap5}) and  (\ref{eq:4skap7})- (\ref{eq:4skap9}) constitutes the approximation utilized to obtain  the lower portion of $\alpha_c^{(4)}(\kappa) $ (4-sfl-RDT) curve in Figure \ref{fig:fig1}. As was the case on the third level, when $\alpha,\kappa\rightarrow 0$ the above approximation becomes fully accurate. However, as Figure \ref{fig:fig1} indicates, even for $\kappa$ away from zero the approximative continuation of $\alpha_c^{(4)}(\kappa) $ curve seems to smoothly extend the non-approximative part.

%%%%%%%%%%%%%%%%%%%%%%%%%%%%%%%%%%%%%%%%%%%%%%%%%%%%%%%%%%%%%%%%%%%%%%%%%%%%%%%%%%%%%%%%%%%%%%%%%%%%%%%%%%%%%%%%%%%%%%%%%
%%%%%%%%%%%%%%%%%%%%%%%%%%%%%%%%%%%%%%%%%%%%%%%%%%%%%%%%%%%%%%%%%%%%%%%%%%%%%%%%%%%%%%%%%%%%%%%%%%%%%%%%%%%%%%%%%%%%%%%%%
\subsubsection{$\alpha,\kappa\rightarrow 0$ -- $r=4$}
\label{sec:numericalskap40}
%%%%%%%%%%%%%%%%%%%%%%%%%%%%%%%%%%%%%%%%%%%%%%%%%%%%%%%%%%%%%%%%%%%%%%%%%%%%%%%%%%%%%%%%%%%%%%%%%%%%%%%%%%%%%%%%%%%%%%%%%
%%%%%%%%%%%%%%%%%%%%%%%%%%%%%%%%%%%%%%%%%%%%%%%%%%%%%%%%%%%%%%%%%%%%%%%%%%%%%%%%%%%%%%%%%%%%%%%%%%%%%%%%%%%%%%%%%%%%%%%%%

Given the simplifications that occurred on the third level when $\alpha,\kappa\rightarrow 0$, it is interesting to see whether something similar happens here. To that end, we first note that for $\q_2^{(s)}\sim -\log\lp \alpha\rp$ and $c_x\sim -\log\lp \alpha\rp$   (\ref{eq:4skap9}) gives $h = c_x$ for the saddle point. From (\ref{eq:4skap8}) we then obtain
\begin{eqnarray}
\label{eq:4skap10} 
 \frac{1}{\c_4^{(s)}}  \log  \lp  \mE_{{\mathcal U}_4}  \lp \omega_q \lp \sqrt{\q_3^{(s)}  }\h_1^{(4)}   \rp  \rp^{\frac{\c_4^{(s)}}{\c_3^{(s)}}}    \rp
 &  \rightarrow  &
 \frac{1}{\c_4^{(s)}}   \log  \lp   e^{\frac{\c_x}{\q_3^{(s)}\c_3^{(s)}}\log \lp \omega_q \lp c_x  \rp  \rp  -\frac{1}{2\q_3^{(s)}}  c_x^2 }   \rp 
 \nonumber \\
  &  \rightarrow  &
 \frac{1}{\c_3^{(s)}}    \log \lp \omega_q \lp c_x  \rp  \rp  -\frac{1}{2}  c_x  . 
\end{eqnarray}
%Moreover,
%\begin{eqnarray}
%\label{eq:4skap10} 
%\mE_{{\mathcal U}_3}   \lp
%  2\cosh \lp \sqrt{\q_2^{(s)} }\h_1^{(3)}   +  c_x       \rp  \rp^{\c_3^{(s)}} 
% &  = &
%\mE_{{\mathcal U}_3}   \lp
%  e^{ \sqrt{\q_2^{(s)} }\h_1^{(3)}   +  c_x       } 
%  +
%    e^{ -\lp \sqrt{\q_2^{(s)} }\h_1^{(3)}   +  c_x \rp      } \rp^{\c_3^{(s)}} 
%\nonumber \\
% &  = &
%\mE_{{\mathcal U}_3}   \lp
%  e^{ \c_3^{(s)} \lp \sqrt{\q_2^{(s)} }\h_1^{(3)}   +  c_x     \rp  } 
% \lp 1 +
%    e^{ -2\lp \sqrt{\q_2^{(s)} }\h_1^{(3)}   +  c_x \rp      } \rp^{\c_3^{(s)}} \rp
%\nonumber \\
% &  = &
%\mE_{{\mathcal U}_3}   \lp
%  e^{ \c_3^{(s)} \lp \sqrt{\q_2^{(s)} }\h_1^{(3)}   +  c_x     \rp  } 
% \lp 1 +
%    \c_3^{(s)} e^{ -2\lp \sqrt{\q_2^{(s)} }\h_1^{(3)}   +  c_x \rp      }  \rp \rp
%\nonumber \\
% &  = &
%\mE_{{\mathcal U}_3}  
%  e^{ \c_3^{(s)} \lp \sqrt{\q_2^{(s)} }\h_1^{(3)}   +  c_x     \rp  } 
%\mE_{{\mathcal U}_3}  
% \lp 1 +
%    \c_3^{(s)} e^{ -2\lp \sqrt{\q_2^{(s)} }\h_1^{(3)}   +  c_x \rp      }  \rp
%\nonumber \\
% &  = &
%   e^{ \frac{\lp\c_3^{(s)} \rp*2 \q_2^{(s)}  }{2}    + \c_3^{(s)}c_x     } 
%  +
%    \c_3^{(s)} 
%       e^{ \frac{\lp\c_3^{(s)}-2\rp^2\q_2^{(s)}  }{2}    + (\c_3^{(s)}-2)c_x     } 
% .
%   \end{eqnarray}
Also, we have
\begin{equation}
\label{eq:4skap11}
   \omega_q \lp c_x  \rp   
 = 
\mE_{{\mathcal U}_3}   \lp
  2\cosh \lp \sqrt{\q_2^{(s)} }\h_1^{(3)}   +  c_x       \rp  \rp^{\c_3^{(s)}} 
   = 
\frac{1}{\sqrt{2\pi}} \int  e^{   \c_3^{(s)}  \log\lp
  2\cosh \lp \sqrt{\q_2^{(s)} }\h_1^{(3)}   +  c_x       \rp  \rp  -\frac{1}{2} \lp \h_1^{(3)} \rp ^2   } d\h_1^{(3)}.
   \end{equation}
After taking the derivative with respect to $\h_1^{(3)}$ one finds
\begin{eqnarray}
\label{eq:4skap12} 
\frac{d \lp   \c_3^{(s)}  \log\lp
  2\cosh \lp \sqrt{\q_2^{(s)} }\h_1^{(3)}   +  c_x       \rp  \rp  -\frac{1}{2} \lp \h_1^{(3)} \rp ^2   \rp }{d\h_1^{(3)}}
=
\sqrt{\q_2^{(s)} }     \c_3^{(s)}  
  \tanh \lp \sqrt{\q_2^{(s)} }\h_1^{(3)}   +  c_x       \rp  -  \h_1^{(3)}    .
\end{eqnarray}
Equalling the above derivative to zero gives for the saddle point $ \h_1^{(3)} = \sqrt{\q_2^{(s)} }     \c_3^{(s)}$. One can then rewrite (\ref{eq:4skap11}) as
\begin{eqnarray}
\label{eq:4skap13} 
   \omega_q \lp c_x  \rp   
 = 
\mE_{{\mathcal U}_3}   \lp
  2\cosh \lp \sqrt{\q_2^{(s)} }\h_1^{(3)}   +  c_x       \rp  \rp^{\c_3^{(s)}} 
 &  = &
  e^{   \c_3^{(s)}  \log\lp
  2\cosh \lp \q_2^{(s)}      \c_3^{(s)}   +  c_x       \rp  \rp  -\frac{1}{2}   \q_2^{(s)}    \lp  \c_3^{(s)}\rp ^2   }.
   \end{eqnarray}
A combination of (\ref{eq:4skap10}) and (\ref{eq:4skap13}) then gives
\begin{eqnarray}
\label{eq:4skap14} 
 \frac{1}{\c_4^{(s)}}  \log  \lp  \mE_{{\mathcal U}_4}  \lp \omega_q \lp \sqrt{\q_3^{(s)}  }\h_1^{(4)}   \rp  \rp^{\frac{\c_4^{(s)}}{\c_3^{(s)}}}    \rp
   &  \rightarrow  &
 \frac{1}{\c_3^{(s)}}    \log \lp \omega_q \lp c_x  \rp  \rp  -\frac{1}{2}  c_x  
 \nonumber \\
   &  \rightarrow  &
\frac{1}{\c_3^{(s)}}    \log \lp   e^{   \c_3^{(s)}  \log\lp
  2\cosh \lp \q_2^{(s)}      \c_3^{(s)}   +  c_x       \rp  \rp  -\frac{1}{2}  \q_2^{(s)}   \lp     \c_3^{(s)}\rp ^2   }   \rp  -\frac{1}{2}  c_x  
 \nonumber \\
   &  \rightarrow  &
     \log\lp
  2\cosh \lp \q_2^{(s)}      \c_3^{(s)}   +  c_x       \rp  \rp  -\frac{1}{2}  \q_2^{(s)}   \c_3^{(s)}      -\frac{1}{2}  c_x  
 \nonumber \\
   &  \rightarrow  &
     \log\lp
 1+ e^{-2 \lp \q_2^{(s)}      \c_3^{(s)}   +  c_x       \rp } \rp  +\frac{1}{2}  \q_2^{(s)}   \c_3^{(s)}      +\frac{1}{2}  c_x  
 . 
\end{eqnarray}
Combining (\ref{eq:4skap5}), (\ref{eq:4skap7}), (\ref{eq:4skap8}), and (\ref{eq:4skap14}) we find
\begin{eqnarray}\label{eq:4skap15} 
\hspace{-.0in} \bar{\psi}_{rd}^{(4,s)}(\p,\q,\c,\gamma_{sq})
  & \rightarrow  &
     \frac{1}{2}
(1-\p_2)\q_2^{(s)}
 +  \frac{1}{2}
\p_2\q_2^{(s)} \c_3^{(s)}
+  \frac{1}{2}
\p_3 c_x
\nonumber \\
&  &  
   -  \log\lp
 1+ e^{-2 \lp \q_2^{(s)}      \c_3^{(s)}   +  c_x       \rp } \rp  -\frac{1}{2}  \q_2^{(s)}   \c_3^{(s)}      -\frac{1}{2}  c_x  
 - \frac{\alpha}{\c_3^{(s)}} \log \lp   \omega_p \lp 0 \rp \rp
  \nonumber \\
  & \rightarrow  &
   \frac{1}{2}
(\p_2-1)\q_2^{(s)} \lp \c_3^{(s)}-1\rp
+  \frac{1}{2}
(\p_3-1) c_x
\nonumber \\
&  &  
   -  \log\lp
 1+ e^{-2 \lp \q_2^{(s)}      \c_3^{(s)}   +  c_x       \rp } \rp    
 - \frac{\alpha}{\c_3^{(s)}} \log \lp   \omega_p \lp 0 \rp \rp. 
    \end{eqnarray}
We now look at stationary points equations obtained after taking $ \q_2^{(s)} $ and $c_x$ derivatives. First we find
\begin{eqnarray}\label{eq:4skap16} 
\hspace{-.0in} \frac{d\bar{\psi}_{rd}^{(4,s)}(\p,\q,\c,\gamma_{sq})}{d \q_2^{(s)}}
   & \rightarrow  &
   \frac{1}{2}
(\p_2-1)  \lp \c_3^{(s)}-1\rp
 +
   2   \c_3^{(s)}   \frac{ e^{-2 \lp \q_2^{(s)}      \c_3^{(s)}   +  c_x       \rp } }{\lp
 1+ e^{-2 \lp \q_2^{(s)}      \c_3^{(s)}   +  c_x       \rp } \rp  }    
   \nonumber \\   
&   \rightarrow  &
   \frac{1}{2}
(\p_2-1)  \lp \c_3^{(s)}-1\rp
 +
   2   \c_3^{(s)}   e^{-2 \lp \q_2^{(s)}      \c_3^{(s)}   +  c_x       \rp  }   , 
    \end{eqnarray}
and then
\begin{eqnarray}\label{eq:4skap17} 
\hspace{-.0in} \frac{d\bar{\psi}_{rd}^{(4,s)}(\p,\q,\c,\gamma_{sq})}{d c_x}
   & \rightarrow  &
   \frac{1}{2}
(\p_3-1) 
 +
   2     \frac{ e^{-2 \lp \q_2^{(s)}      \c_3^{(s)}   +  c_x       \rp } }{\lp
 1+ e^{-2 \lp \q_2^{(s)}      \c_3^{(s)}   +  c_x       \rp } \rp  }  
 \nonumber \\
    & \rightarrow  &
   \frac{1}{2}
(\p_3-1) 
 +
   2      e^{-2 \lp \q_2^{(s)}      \c_3^{(s)}   +  c_x       \rp }   .
    \end{eqnarray}
Keeping in mind the earlier mentioned scaling, we set $1-\p_2\rightarrow\frac{p_2\alpha}{-\log\lp \alpha \rp}$, $\p_2-\p_4\rightarrow \frac{p_3\alpha}{-\log\lp \alpha \rp}$, $\q_2 \rightarrow -q_2\log\lp \alpha \rp $, and  $\c_x \rightarrow -\bar{c}_x\log\lp \alpha \rp $. After equaling the derivative in (\ref{eq:4skap16} )  to zero  we have
\begin{eqnarray}\label{eq:4skap18} 
    \q_2^{(s)}      \c_3^{(s)}   +  c_x      
    \rightarrow
-  \frac{1}{2}
 \log \lp   \frac{1}{4}
(1-\p_2)  \frac{\c_3^{(s)}-1} {\c_3^{(s)} } \rp
\rightarrow  - \frac{1}{2}
 \log \lp   \frac{1}{4}
\frac{p_2\alpha}{-\log\lp \alpha \rp} \frac{\c_3^{(s)}-1} {\c_3^{(s)} } \rp
\rightarrow 
-\frac{1}{2}
 \log \lp 
\alpha \rp,
    \end{eqnarray}
    and
\begin{eqnarray}\label{eq:4skap19} 
   -\lp  q_2   \c_3^{(s)}   +  \bar{c}_x \rp   \log \lp 
\alpha \rp  
 \rightarrow 
-\frac{1}{2}
 \log \lp 
\alpha \rp,
    \end{eqnarray}
Also, after equaling the derivative in (\ref{eq:4skap17} )  to zero  we have
\begin{eqnarray}\label{eq:4skap20} 
    \q_2^{(s)}      \c_3^{(s)}   +  c_x      
    \rightarrow
 - \frac{1}{2}
 \log \lp   \frac{1}{4}
(1-\p_3)  \rp
  \rightarrow
 -  \frac{1}{2}
 \log \lp   \frac{1}{4}
\frac{(p_2-p_3)\alpha}{-\log\lp \alpha \rp}   \rp
\rightarrow 
-\frac{1}{2}
 \log \lp 
\alpha \rp,
    \end{eqnarray}
    and
\begin{eqnarray}\label{eq:4skap21} 
   -\lp  q_2     \c_3^{(s)}   +  \bar{c}_x \rp       \log \lp 
\alpha \rp
 \rightarrow 
-\frac{1}{2}
 \log \lp 
\alpha \rp,
    \end{eqnarray}
In other words, on leading terms level, the above derivatives produce redundant stationary points equations. While we have no way of knowing if the utilized scaling is appropriate or if other scalings and/or finer analyses of higher order terms might change things, the above could indicate that not much of a substantial improvement is expected as one moves from the third to the fourth level. It is also worth noting that the above  seems to be in an agreement with \cite{BarbAKZ23} where for $\alpha\rightarrow 0$ the 1RSB local entropy approach (which, as demonstrated in previous section, produces the same results as our third level of lifting) was shown to match rigorously proven contiguity results of the planted model.

%%%%%%%%%%%%%%%%%%%%%%%%%%%%%%%%%%%%%%%%%%%%%%%%%%%%%%%%%%%%%%%%%%%%%%%%%%%%%%%%%%%%%%%%%%%%%%%%%%%%%%%%%%%%%%%%%%%%%%%%%
%%%%%%%%%%%%%%%%%%%%%%%%%%%%%%%%%%%%%%%%%%%%%%%%%%%%%%%%%%%%%%%%%%%%%%%%%%%%%%%%%%%%%%%%%%%%%%%%%%%%%%%%%%%%%%%%%%%%%%%%%
%%%%%%%%%%%%%%%%%%%%%%%%%%%%%%%%%%%%%%%%%%%%%%%%%%%%%%%%%%%%%%%%%%%%%%%%%%%%%%%%%%%%%%%%%%%%%%%%%%%%%%%%%%%%%%%%%%%%%%%%%
%%%%%%%%%%%%%%%%%%%%%%%%%%%%%%%%%%%%%%%%%%%%%%%%%%%%%%%%%%%%%%%%%%%%%%%%%%%%%%%%%%%%%%%%%%%%%%%%%%%%%%%%%%%%%%%%%%%%%%%%%
%%%%%%%%%%%%%%%%%%%%%%%%%%%%%%%%%%%%%%%%%%%%%%%%%%%%%%%%%%%%%%%%%%%%%%%%%%%%%%%%%%%%%%%%%%%%%%%%%%%%%%%%%%%%%%%%%%%%%%%%%
\section{Computational and practical aspects}
\label{sec:comp}
%%%%%%%%%%%%%%%%%%%%%%%%%%%%%%%%%%%%%%%%%%%%%%%%%%%%%%%%%%%%%%%%%%%%%%%%%%%%%%%%%%%%%%%%%%%%%%%%%%%%%%%%%%%%%%%%%%%%%%%%%
%%%%%%%%%%%%%%%%%%%%%%%%%%%%%%%%%%%%%%%%%%%%%%%%%%%%%%%%%%%%%%%%%%%%%%%%%%%%%%%%%%%%%%%%%%%%%%%%%%%%%%%%%%%%%%%%%%%%%%%%%
%%%%%%%%%%%%%%%%%%%%%%%%%%%%%%%%%%%%%%%%%%%%%%%%%%%%%%%%%%%%%%%%%%%%%%%%%%%%%%%%%%%%%%%%%%%%%%%%%%%%%%%%%%%%%%%%%%%%%%%%%
%%%%%%%%%%%%%%%%%%%%%%%%%%%%%%%%%%%%%%%%%%%%%%%%%%%%%%%%%%%%%%%%%%%%%%%%%%%%%%%%%%%%%%%%%%%%%%%%%%%%%%%%%%%%%%%%%%%%%%%%%
%%%%%%%%%%%%%%%%%%%%%%%%%%%%%%%%%%%%%%%%%%%%%%%%%%%%%%%%%%%%%%%%%%%%%%%%%%%%%%%%%%%%%%%%%%%%%%%%%%%%%%%%%%%%%%%%%%%%%%%%%

%%%%%%%%%%%%%%%%%%%%%%%%%%%%%%%%%%%%%%%%%%%%%%%%%%%%%%%%%%%%%%%%%%%%%%%%%%%%%%%%%%%%%%%%%%%%%%%%%%%%%%%%%%%%%%%%%%%%%%%%%
%%%%%%%%%%%%%%%%%%%%%%%%%%%%%%%%%%%%%%%%%%%%%%%%%%%%%%%%%%%%%%%%%%%%%%%%%%%%%%%%%%%%%%%%%%%%%%%%%%%%%%%%%%%%%%%%%%%%%%%%%
\subsection{Practicality of numerical evaluations}
\label{sec:pracnumeval}
%%%%%%%%%%%%%%%%%%%%%%%%%%%%%%%%%%%%%%%%%%%%%%%%%%%%%%%%%%%%%%%%%%%%%%%%%%%%%%%%%%%%%%%%%%%%%%%%%%%%%%%%%%%%%%%%%%%%%%%%%
%%%%%%%%%%%%%%%%%%%%%%%%%%%%%%%%%%%%%%%%%%%%%%%%%%%%%%%%%%%%%%%%%%%%%%%%%%%%%%%%%%%%%%%%%%%%%%%%%%%%%%%%%%%%%%%%%%%%%%%%%

We here point out some important aspects related to numerical evaluations that may not be easily visible through the above results. First, the results of both previous sections indeed provide a generic strategy to evaluate potential estimates of algorithmic thresholds. However, even if one is equipped with all the needed analytical characterizations, a gigantic task might still await as numerical solving of the obtained stationary points equations is not guaranteed to be easy. A couple of potential obstacles immediately come to mind: (\textbf{\emph{i}}) increasing $r$ increases the number of nested integrations which all but assures memory problems; and  (\textbf{\emph{ii}}) opting for iterative schemes more often than not might result in procedures being stuck in local optima which coupled with potentially fairly slow convergence could be unsurpassable. With all of this in mind, we found as most convenient to circumvent  systematic solving approaches and instead focus on ad-hoc heuristics. Similarly to what was done in \cite{Stojnicalgbp25}, we used a semi-manual approach. First we  randomly chose a subset of variables to optimize and then depending on what such a procedure gives we manually move some of the other ones and choose another set candidates for the following optimization. When doing the subset optimization we often avoid reaching full precision to potentially prevent being stuck in unfavorable local optima. While no structural formalization of this procedure is possible, it still offers a useful advantage. Extensive repetitions and interactive tryouts allow one to gain a bit of a feeling regarding the behavior of optimizing parameters and their impact on the overall flow of the procedure. This helps avoiding local optima and enables a more accurate understanding of stationary points. For example, we believe that even if multiple stationary points are present their associated  $\alpha_a$'s are either close to the ones given in our tables or sufficiently  far away that can easily be discarded. Also, potential parametric inaccuracies might have negligible effects on $\alpha_c^{(r)}$ (this is particularly useful as convergence is often slow which renders achieving full parametric accuracy infeasible).

It is also important to note that at present conducting numerical evaluations to high precision is not the most essential component of the proposed machinery but rather a helping tool that allows its practical realization. Right now the proposed ideas are of conjectural nature which renders results obtained through numerical procedures more relevant than the procedures themselves. However, if the rigorous confirmations are eventually achieved  then development of provably efficient numerical procedures will become of utmost importance. Whether creating a systematic approach guaranteed to solve the problem for any $r$ is possible remains an open problem. The memory requirements together with convergence might pose a serious challenge in that direction.

%%%%%%%%%%%%%%%%%%%%%%%%%%%%%%%%%%%%%%%%%%%%%%%%%%%%%%%%%%%%%%%%%%%%%%%%%%%%%%%%%%%%%%%%%%%%%%%%%%%%%%%%%%%%%%%%%%%%%%%%%
%%%%%%%%%%%%%%%%%%%%%%%%%%%%%%%%%%%%%%%%%%%%%%%%%%%%%%%%%%%%%%%%%%%%%%%%%%%%%%%%%%%%%%%%%%%%%%%%%%%%%%%%%%%%%%%%%%%%%%%%%
%%%%%%%%%%%%%%%%%%%%%%%%%%%%%%%%%%%%%%%%%%%%%%%%%%%%%%%%%%%%%%%%%%%%%%%%%%%%%%%%%%%%%%%%%%%%%%%%%%%%%%%%%%%%%%%%%%%%%%%%%
%%%%%%%%%%%%%%%%%%%%%%%%%%%%%%%%%%%%%%%%%%%%%%%%%%%%%%%%%%%%%%%%%%%%%%%%%%%%%%%%%%%%%%%%%%%%%%%%%%%%%%%%%%%%%%%%%%%%%%%%%
%%%%%%%%%%%%%%%%%%%%%%%%%%%%%%%%%%%%%%%%%%%%%%%%%%%%%%%%%%%%%%%%%%%%%%%%%%%%%%%%%%%%%%%%%%%%%%%%%%%%%%%%%%%%%%%%%%%%%%%%%
\subsection{A concrete SBP algorithm}
\label{sec:algimp}
%%%%%%%%%%%%%%%%%%%%%%%%%%%%%%%%%%%%%%%%%%%%%%%%%%%%%%%%%%%%%%%%%%%%%%%%%%%%%%%%%%%%%%%%%%%%%%%%%%%%%%%%%%%%%%%%%%%%%%%%%
%%%%%%%%%%%%%%%%%%%%%%%%%%%%%%%%%%%%%%%%%%%%%%%%%%%%%%%%%%%%%%%%%%%%%%%%%%%%%%%%%%%%%%%%%%%%%%%%%%%%%%%%%%%%%%%%%%%%%%%%%
%%%%%%%%%%%%%%%%%%%%%%%%%%%%%%%%%%%%%%%%%%%%%%%%%%%%%%%%%%%%%%%%%%%%%%%%%%%%%%%%%%%%%%%%%%%%%%%%%%%%%%%%%%%%%%%%%%%%%%%%%
%%%%%%%%%%%%%%%%%%%%%%%%%%%%%%%%%%%%%%%%%%%%%%%%%%%%%%%%%%%%%%%%%%%%%%%%%%%%%%%%%%%%%%%%%%%%%%%%%%%%%%%%%%%%%%%%%%%%%%%%%
%%%%%%%%%%%%%%%%%%%%%%%%%%%%%%%%%%%%%%%%%%%%%%%%%%%%%%%%%%%%%%%%%%%%%%%%%%%%%%%%%%%%%%%%%%%%%%%%%%%%%%%%%%%%%%%%%%%%%%%%%

In an ideal scenario, the ultimate goal of any SBP consideration is to eventually have an efficient (say polynomial) algorithm that can solve the problem for constraints densities as close to satisfiability threshold as possible. As the above analytical machinery and previous results available in the literature strongly suggest that SBP likely exhibits a computational gap, approaching algorithmic threshold seems the most one can expect. At present we are unaware of any algorithmic dynamics that can fully match the analytical predictions discussed in previous sections. On the other hand, these predictions seem fairly connected to the ones obtained for ABP in \cite{Stojnicalgbp25} and for the negative Hopfield model in \cite{Stojniccluphop25}. As   \cite{Stojnicalgbp25,Stojniccluphop25} also designed CLuP like algorithms with performances fairly closely approaching  proposed theoretical results, one may wonder whether analogous SBP designs are possible as well.  With the introduction of a CLuP-SBP algorithmic variant, we below provide a positive answer to this question.

%%%%%%%%%%%%%%%%%%%%%%%%%%%%%%%%%%%%%%%%%%%%%%%%%%%%%%%%%%%%%%%%%%%%%%%%%%%%%%%%%%%%%%%%%%%%%%%%%
%%%%%%%%%%%%%%%%%%%%%%%%%%%%%%%%%%%%%%%%%%%%%%%%%%%%%%%%%%%%%%%%%%%%%%%%%%%%%%%%%%%%%%%%%%%%%%%%%
\subsubsection{CLuP-SBP}
\label{sec:clupabp}
%%%%%%%%%%%%%%%%%%%%%%%%%%%%%%%%%%%%%%%%%%%%%%%%%%%%%%%%%%%%%%%%%%%%%%%%%%%%%%%%%%%%%%%%%%%%%%%%%
%%%%%%%%%%%%%%%%%%%%%%%%%%%%%%%%%%%%%%%%%%%%%%%%%%%%%%%%%%%%%%%%%%%%%%%%%%%%%%%%%%%%%%%%%%%%%%%%%

CLuP  (\emph{controlled loosening-up})  algorithms were introduced   for solving \emph{planted} models in \cite{Stojnicclupint19,Stojnicclupspreg20}. Considering two classical  regression problems (with binary and sparse unknown planted vectors), \cite{Stojnicclupint19,Stojnicclupspreg20}   showed that CLuP retains high accuracy while avoiding some of the potentially unwanted  features (excessive reliance of planted signal's a priori available knowledge, statistical sensitivity,  and super large dimensions) of AMP and other existing alternatives. In \cite{Stojniccluphop25,Stojnicclupsk25,Stojnicalgbp25} CLuP variants particularly tailored to fit positive/negative Hopfield models and ABPs widened the range of applicability to include the  non-planted scenarios as well. We here show that a successful CLuP variant tailored for SBP can be designed as well. To that end, we first note that, as in \cite{Stojnicalgbp25}, we are now facing the following feasibility problem (this is in a stark contrast with classical optimal objective seeking problems considered  in \cite{Stojnicclupint19,Stojnicclupspreg20,Stojniccluphop25,Stojnicclupsk25})
\begin{eqnarray}
 \mbox{find} & & \x\nonumber \\
\mbox{subject to}
& & |G\x|  \leq \kappa \1 \nonumber \\
& & \x\in\left \{-\frac{1}{\sqrt{n}}, \frac{1}{\sqrt{n}} \right \}^n . \label{eq:clupex1}
\end{eqnarray}
Following into the footsteps of \cite{Stojnicalgbp25}, we propose an iterative procedure that we call
\begin{eqnarray}\label{eq:algimpeq2}
 \hspace{-.55in} \mbox{\bl{\textbf{\emph{CLuP-SBP algorithm:}}}}  \hspace{.65in} \x^{(t+1)} & \rightarrow  &
\mbox{\textbf{gradbar}}\lp\bar{f}_{b,x} \lp \x;\bar{t}_{0x}^{(t)} \rp ;\x^{(t)},\bar{t}_{0x}^{(t)} \rp
 \nonumber \\
\bar{t}_{0x}^{(t+1)}  &  \rightarrow  & \bar{c}^{(t)}\bar{t}_{0x}^{(t)}.
\end{eqnarray}
We take $\hat{\kappa} = \max  \left | G \frac{1}{\sqrt{n}} \mbox{sign}\lp \x^{(t)}\rp  \right |$ and the corresponding $\frac{1}{\sqrt{n}} \mbox{sign}\lp \x^{(t)}\rp$ as the output of the algorithm. For function $\bar{f}_{b,x} \lp \x;\bar{t}_{0x}^{(t)}  \rp $ specified by argument $\bar{t}_{0x}^{(t)}$
\begin{eqnarray}\label{eq:algimpeq3}
\bar{f}_{b,x} \lp\x;\bar{t}_{0x}^{(t)}\rp = - \bar{t}_{0x}^{(t)} \|\x\|_2 - \frac{1}{m}\sum_{j=1}^{m}\log\lp \kappa_0^2 -\lp G_{j,1:n} \x\rp^2 \rp
-\frac{1}{n}\sum_{i=1}^{n}  \log(1-n\x_i^2),
\end{eqnarray}
procedure $\mbox{\textbf{gradbar}}\lp\bar{f}_{b,x} \lp \x;\bar{t}_{0x}^{(t)} \rp ;\x^{(t)},\bar{t}_{0x}^{(t)} \rp$ applies a form of gradient descent starting from $\x^{(t)}$.  The optimization in (\ref{eq:algimpeq3}) is conceived as the ABP analogous one in \cite{Stojnicalgbp25}.  Not much of a restriction is now imposed on $\kappa_0$ and  parameters $t_{0x}^{(0)}$, $c^{(t)}$. An ideal choice would ensure that the objective in (\ref{eq:algimpeq3}) has a favorable  optimizing landscape with minimal number of local optima. However, practically reaching such a choice is delicate and no generic recipe seems to guarantee success (change of ambient dimension $n$ and/or  constraints density $\alpha$ might be among the reasons). Similarly to what was the case with CLuP-ABP, our experience here suggests retuning these parameters on fly as best strategy. Even though overall universality may be unreachable, a couple of quick pointers might be useful: (\textbf{\emph{i}})  favorable $\kappa_0$ seems to increase as $\alpha$ increases; and  (\textbf{\emph{ii}}) taking $c^{(t)} =1.1$ and $t_{0x}^{(0)}=1.3$ could serve as a good starting choice. Also, starting $\x^{(0)}$ is any $\x$ that fits under $\log$s. One can start with a random choice from $\left \{\pm\frac{1}{\sqrt{n}}\right \}^n$ and then scale down by two until feasibility is reached. Restarting for different $\x^{(0)}$ is critically important. A large number of restarts usually prevents applying algorithm to large dimensional problems accros the entire $\kappa\in[0,3]$ range. However, we were able to apply it to moderate dimensions in higher $\kappa$ regime and then to successively increase dimensions as $\kappa$ decreases. The obtained results are shown in Figure  \ref{fig:fig1}.  To ensure faithful polynomiality emulation, the number of restarts never went over $2000$ (typically, it actually remained in the range of a few hundreds). We should also point out that dealing with feasibility problems inevitably brings the difficulty of adequate success emulation. For fixed $\alpha$ we plot $\mE \hat{\kappa}$ which allows that these inadequacies are circumvented. Nonetheless,  simulated results closely mimic the proposed theoretical threshold curve over a wide range of $\kappa$. The algorithm may not be very practical without further upgrades (particularly so in the higher $\kappa$ regime). However, it does strengthen belief that the proposed algorithmic thresholds might indeed be very close (if not identical) to the true ones.

 \begin{figure}[h]
%\begin{minipage}[b]{.5\linewidth}
\centering
\centerline{\includegraphics[width=1.00\linewidth]{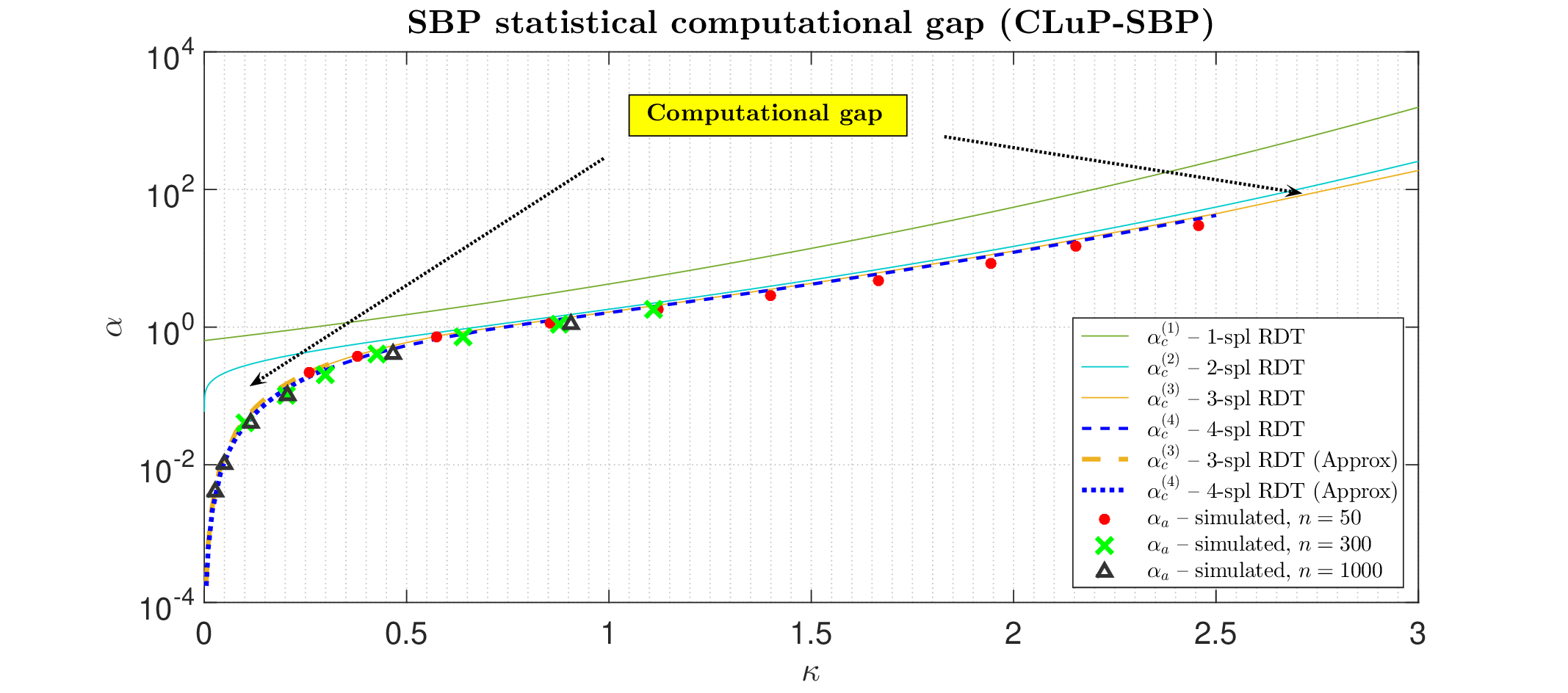}}
%\end{minipage}
%\begin{minipage}[b]{.5\linewidth}
%\centering
%\centerline{\epsfig{figure=finprerral08.eps,width=9cm,height=6.5cm}}
%\end{minipage}
\caption{SBP algorithmic threshold: fl-RDT theoretical prediction and CLuP-SBP algorithmic simulation}
\label{fig:fig1}
\end{figure}

%%%%%%%%%%%%%%%%%%%%%%%%%%%%%%%%%%%%%%%%%%%%%%%%%%%%%%%%%%%%%%%%%%%%%%%%%%%%%%%%%%%%%%%%%%%%%%%%%%%%%%%%%%%%%%%%%%%%%%%%%
%%%%%%%%%%%%%%%%%%%%%%%%%%%%%%%%%%%%%%%%%%%%%%%%%%%%%%%%%%%%%%%%%%%%%%%%%%%%%%%%%%%%%%%%%%%%%%%%%%%%%%%%%%%%%%%%%%%%%%%%%
%%%%%%%%%%%%%%%%%%%%%%%%%%%%%%%%%%%%%%%%%%%%%%%%%%%%%%%%%%%%%%%%%%%%%%%%%%%%%%%%%%%%%%%%%%%%%%%%%%%%%%%%%%%%%%%%%%%%%%%%%
\section{Algorithmic implications}
\label{sec:algimp}
%%%%%%%%%%%%%%%%%%%%%%%%%%%%%%%%%%%%%%%%%%%%%%%%%%%%%%%%%%%%%%%%%%%%%%%%%%%%%%%%%%%%%%%%%%%%%%%%%%%%%%%%%%%%%%%%%%%%%%%%%
%%%%%%%%%%%%%%%%%%%%%%%%%%%%%%%%%%%%%%%%%%%%%%%%%%%%%%%%%%%%%%%%%%%%%%%%%%%%%%%%%%%%%%%%%%%%%%%%%%%%%%%%%%%%%%%%%%%%%%%%%
%%%%%%%%%%%%%%%%%%%%%%%%%%%%%%%%%%%%%%%%%%%%%%%%%%%%%%%%%%%%%%%%%%%%%%%%%%%%%%%%%%%%%%%%%%%%%%%%%%%%%%%%%%%%%%%%%%%%%%%%%

Whether local entropy indeed relates to algorithmic thresholds and SCGs remains a mystery at present. Equally mysterious remains the connection between our approach and SCGs. Below we emphasize and summarize two types of striking similarities observed so far that may motivate a very generic applicability of the methods discussed in earlier sections.

%%%%%%%%%%%%%%%%%%%%%%%%%%%%%%%%%%%%%%%%%%%%%%%%%%%%%%%%%%%%%%%%%%%%%%%%%%%%%%%%%%%%%%%%%%%%%%%%%%%%%%%%%%%%%%%%%%%%%%%%%
%%%%%%%%%%%%%%%%%%%%%%%%%%%%%%%%%%%%%%%%%%%%%%%%%%%%%%%%%%%%%%%%%%%%%%%%%%%%%%%%%%%%%%%%%%%%%%%%%%%%%%%%%%%%%%%%%%%%%%%%%
\subsection{Striking feasibility problems similarities}
\label{sec:algimpfeas}
%%%%%%%%%%%%%%%%%%%%%%%%%%%%%%%%%%%%%%%%%%%%%%%%%%%%%%%%%%%%%%%%%%%%%%%%%%%%%%%%%%%%%%%%%%%%%%%%%%%%%%%%%%%%%%%%%%%%%%%%%
%%%%%%%%%%%%%%%%%%%%%%%%%%%%%%%%%%%%%%%%%%%%%%%%%%%%%%%%%%%%%%%%%%%%%%%%%%%%%%%%%%%%%%%%%%%%%%%%%%%%%%%%%%%%%%%%%%%%%%%%%

Remarkable agreements between our results and local entropy considerations combined with existence of polynomial algorithms in ranges close to predicted algorithmic thresholds seem as a bit more than just a mere coincidence. As stated earlier on multiple occasions, the very same behavior that we are seeing here was observed when studying ABP in \cite{Stojnicalgbp25} as well. To make the parallel
clearer, we briefly summarize key aspects of both ABP and corresponding SBP considerations (for concreteness, we focus on canonical instances $\kappa=0$ for ABP and $\kappa=1$ for SBP). ABP is believed  to exhibit a statistical-computational gap as the best known polynomial algorithms can solve (canonical $\kappa=0$) ABP instances for constraints densities up to $\sim 0.75-0.77$ which falls short of the satisfiability threshold $\approx 0.8331$ \cite{KraMez89,DingSun19,NakSun23,BoltNakSunXu22,Huang24,Stojnicbinperflrdt23}. Moreover, local entropy studies of atypical ABP solutions \cite{Bald15,Bald16,Bald21,Stojnicabple25}   point towards $\alpha$ interval $(0.77,0.78)$ as the range where clustering defragmentation (envisioned as a likely cause for failure of locally improving algorithms) happens. The parametric fl-RDT approach proposed in \cite{Stojnicalgbp25} gives $\approx 0.7764$ as algorithmic threshold estimate on the fifth level of lifting  and predicts that the true value is somewhere in $0.775-0.776$ interval. In light of these ABP considerations, a completely analogous scenario seems to be happening for SBP as well. The local entropy study of SBP's atypical solutions (basically SBP analogue to ABP related  \cite{Bald15,Bald16,Bald21}) is presented in  \cite{Bald20,BarbAKZ23}. Utilizing replica methods (with 1RSB ansatz),  \cite{Bald20}  obtains (for canonical $\kappa=1$ scenario) $\alpha\sim 1.58$ as the critical constraints  density estimate where clustering defragmentation happens (this is well short of SBP's satisfiability threshold $\approx
1.8159$ obtained earlier in  \cite{AubPerZde19,GamKizPerXu22,PerkXu21,AbbLiSly21a,AbbLiSly21b,Bald20,Barb24,BarbAKZ23} and matched here on the second level of lifting with $\alpha_c^{(2)}(1)$). As shown in Table \ref{tab:tab3}, our parametric fl-RDT approach gives on the seventh level of lifting $\alpha_c^{(7)}(1)\approx 1.6021$. Given somewhat slower convergence than in ABP case, we believe that one has for the converging value  $\alpha_c^{(\infty)}(1) \sim 1.59-1.60$ which remarkably closely approaches the local entropy prediction of  \cite{Bald20}.  The results of  \cite{Bald20} are further extended in  \cite{BarbAKZ23} through both  mathematically rigorous and replica approaches. Particular attention is devoted to $\alpha,\kappa\rightarrow 0$ regime where earlier OGP studies \cite{GamKizPerXu22} hinted at the existence of a computational gap with a $\alpha,\kappa$ dependence ($\kappa\sim\sqrt{\frac{\alpha}{-\log\lp \alpha\rp}}$) that qualitatively deviates from the one associated with satisfiability thresholds. Modulo types of clustering, both contiguity approach of a planted model and 1RSB analysis of local entropies gave in  \cite{BarbAKZ23} $\kappa\approx 1.2385\sqrt{\frac{\alpha}{-\log\lp \alpha\rp}}$ as the functional dependence at the so-called energetic algorithmic threshold for $\alpha,\kappa\rightarrow 0$ regime (which indeed substantially deviates from the satisfiability threshold, qualitatively matches \cite{GamKizPerXu22} scaling, and is also away from $\kappa\sim \sqrt{\alpha} $ dependence where best known algorithms succeed \cite{BanSpen20}). Our results on the third lifting level fully match $1.2385\sqrt{\frac{\alpha}{-\log\lp \alpha\rp}}$ (moreover, our fourth level considerations might suggest that not much of a significant improving happens beyond the third level (possibly in agreement with a conjecture of \cite{BarbAKZ23}  that in $\alpha,\kappa\rightarrow 0$ regime their 1RSB analyses might suffice)).

On top of the above analytical similarities, the algorithmic ones seem to be present as well. In \cite{Stojnicalgbp25} the performance of CLuP-ABP  was shown to closely approach the ABP's algorithmic threshold predictions. As Figure \ref{fig:fig1} shows, performance of analogous CLuP-SBP introduced here closely approaches the fourth lifting level  estimate $\alpha_c^{(4)}(\kappa)$ over a wide range of $\kappa$.

%%%%%%%%%%%%%%%%%%%%%%%%%%%%%%%%%%%%%%%%%%%%%%%%%%%%%%%%%%%%%%%%%%%%%%%%%%%%%%%%%%%%%%%%%%%%%%%%%%%%%%%%%%%%%%%%%%%%%%%%%
%%%%%%%%%%%%%%%%%%%%%%%%%%%%%%%%%%%%%%%%%%%%%%%%%%%%%%%%%%%%%%%%%%%%%%%%%%%%%%%%%%%%%%%%%%%%%%%%%%%%%%%%%%%%%%%%%%%%%%%%%
\subsection{Striking optimal objective problems similarities}
\label{sec:algimpopt}
%%%%%%%%%%%%%%%%%%%%%%%%%%%%%%%%%%%%%%%%%%%%%%%%%%%%%%%%%%%%%%%%%%%%%%%%%%%%%%%%%%%%%%%%%%%%%%%%%%%%%%%%%%%%%%%%%%%%%%%%%
%%%%%%%%%%%%%%%%%%%%%%%%%%%%%%%%%%%%%%%%%%%%%%%%%%%%%%%%%%%%%%%%%%%%%%%%%%%%%%%%%%%%%%%%%%%%%%%%%%%%%%%%%%%%%%%%%%%%%%%%%

Both SBP and ABP discussed above are feasibility problems. However, the trend observed for them seems to extend to standard optimal objective seeking optimizations as well. In  \cite{Stojniccluphop25} positive and negative Hopfield models (Hop+ and Hop-) were studied and CLuP variants are designed to study their GSEs.  In both scenarios excellent algorithmic performance that closely approaches theoretical predictions is observed. For Hop+ the absence of SCG is expected and \cite{Stojniccluphop25}'s algorithmic results indicate that such expectation might indeed be correct. Similar proximity of practical performance and theoretical predictions is observed for Hop- as well. However, differently from Hop+, the corresponding Hop- theoretical predictions assume arbitrary (non-necessarily decreasing) $\c$ sequence ordering which is not physical and in a way resembles what happens here in the SBP context and in \cite{Stojnicalgbp25} in the ABP context (the associated Gibbs measures (with or without $\c$ sequence ordering) of Hop- are discontinuous and the SCG is expected). While  the higher lifting level results of \cite{Stojniccluphop25}  seem to be either precisely the algorithmic thresholds or their close approximates, the 2-sfl-RDT results (the highest RDT lifting level where decreasing property of $\c$ sequence still holds) are likely the theoretical GSE values (reachable with an infinite computational power). As numerical difference between the second and higher level estimates is insignificant, predicated existence of Hop- SCG seems more as a formality than a feature of practical relevance. However, as such difference for both ABP and (particularly) SBP is more significant, the postulated existence of SCG in these problems represents a very relevant algorithmic obstacle.

In addition to the above discussed Hopfield models, similar parallels seem to appear in SK models. A strong progress in algorithmic studying of SK models was recently achieved in \cite{Montanari19} where Montanari showed that the GSE of the pure $2$-spin SK Ising model can be computed  in polynomial time  via IAMP (an incremental modification of AMP) provided that the associated Parisi functional is continuously increasing (several other fast algorithms achieved similar performance and effectively reaffirmed that the classical SK's GSE is computable in polynomial time \cite{Dandietal25,Erba24,Boet05,Das25,Stojnicclupsk25}). In contrast  with the $2$-spin SK model (which is expected to be solvable in polynomial time), higher $p$-spin variants might not be. Employing the same IAMP  \cite{ElAlMont20} demonstrated that an SCG is indeed likely to happen already for $p=3$ (practically speaking, the gap is rather small, but its existence brings an important theoretical value). Furthermore, \cite{ElAlMont20} also demonstrated that the IAMP's performance seems to fairly well match  the virtual GSE estimate obtained with the restrictive  nondecreasing (physical) nature of Parisi functional removed. Despite the fact that here (and in \cite{Stojniccluphop25,Stojnicalgbp25})  the realm of operation is different (instead of PDE and associated functionals, fl-RDT and sequences of parameters are considered and ABP and SBP are feasibility problems rather than optimal objective seeking ones), we believe that the observed parallels are noteworthy.

%%%%%%%%%%%%%%%%%%%%%%%%%%%%%%%%%%%%%%%%%%%%%%%%%%%%%%%%%%%%%%%%%%%%%%%%%%%%%%%%%%%%%%%%%%%%%%%%%%%%%%%%%%%%%%%%%%%%%%%%%
%%%%%%%%%%%%%%%%%%%%%%%%%%%%%%%%%%%%%%%%%%%%%%%%%%%%%%%%%%%%%%%%%%%%%%%%%%%%%%%%%%%%%%%%%%%%%%%%%%%%%%%%%%%%%%%%%%%%%%%%%
\subsection{Possible consequences}
\label{sec:csq}
%%%%%%%%%%%%%%%%%%%%%%%%%%%%%%%%%%%%%%%%%%%%%%%%%%%%%%%%%%%%%%%%%%%%%%%%%%%%%%%%%%%%%%%%%%%%%%%%%%%%%%%%%%%%%%%%%%%%%%%%%
%%%%%%%%%%%%%%%%%%%%%%%%%%%%%%%%%%%%%%%%%%%%%%%%%%%%%%%%%%%%%%%%%%%%%%%%%%%%%%%%%%%%%%%%%%%%%%%%%%%%%%%%%%%%%%%%%%%%%%%%%

While it remains possible that our propositions coincidentally happen to be in agrement with local entropy SCG predictions, the above similarities seem to point in a different direction. Morever, not only do we believe that the agreements are not coincidental, we actually believe that our propositions might extend far beyond SBP, ABP, and Hopfild examples discussed so far.  We formalize all of the above in the following SBP algorithmic threshold conjecture.

\begin{conjecture}[SBP algorithmic threshold] 
\label{thm:conj1} 
Let $G\in\mR^{m\times n}$ be comprised of independent standard normals and for constraints density $\alpha=\lim_{n\rightarrow \infty } \frac{m}{n} $ and $\kappa\in\mR_+$
  consider a statistical SBP $\mathbf{\mathcal S} \lp G,\kappa,\alpha \rp$ from (\ref{eq:ex1a0}). Define its \underline{algorithmic} threshold as
\begin{eqnarray}\label{eq:alphaa}
  \alpha_a (\kappa) \triangleq   \max   \left \{\alpha |\hspace{.05in}  \lim_{n\rightarrow \infty}\mP\left (   \mbox{$\mathbf{\mathcal S} \lp G,\kappa,\alpha \rp$  is solvable in polynomial time}   \right ) =1 \right \}.
 \end{eqnarray}
Let $ \bar{\psi}_{rd}^{(r)} (\cdot)\triangleq  \bar{\psi}_{rd} (\cdot)$  with $\bar{\psi}_{rd} (\cdot)$ as in  (\ref{eq:negprac13}) and $\hat{\p}$, $\hat{\q}$, $\hat{\c}$, and $\hat{\gamma}_{sq}$ as in Theorem   \ref{thme:negthmprac1}. Then there is a non-increasing converging $r$-sequence
\begin{eqnarray}\label{eq:alphacr}
  \alpha_c^{(r)} (\kappa) =  \left \{ \alpha  \hspace{.0in} |  \hspace{.05in}  \bar{\psi}_{rd}^{(r)} (\hat{\p},\hat{\q},\hat{\c},\hat{\gamma}_{sq}) = 0 \right \},
 \end{eqnarray}
such that
 \begin{eqnarray}\label{eq:conjub}
 \alpha_c^{(r)} (\kappa) \geq \alpha_a (\kappa) \quad \mbox{and} \quad    \lim_{r\rightarrow \infty}\alpha_c^{(r)} (\kappa) \triangleq \hat{\alpha}_a (\kappa) = \alpha_a (\kappa).
 \end{eqnarray}
Since  $\alpha_c(\kappa)=\alpha_c^{(2)}(\kappa)$ one then has for the SBP's statistical computational gap (SCG)
\begin{eqnarray}\label{eq:scg}
 SCG = \alpha_c (\kappa) - \alpha_a (\kappa) = \alpha_c^{(2)} (\kappa) -  \lim_{r\rightarrow \infty}\alpha_c^{(r)} (\kappa) .
 \end{eqnarray}
\end{conjecture}

The conjecture essentially states that within the fl-RDT for any $\kappa\in\mR_+$ there is a non-increasing $r$-sequence, $ \alpha_c^{(r)} (\kappa)$,  that converges towards algorithmic threshold (the discussions of the preceding sections suggest that in certain $\kappa$ regimes the limiting $r$ could even be finite). On a more practical level the conjecture effectively states that $ \alpha_c^{(r)} (\kappa)$ numerical results from Tables \ref{tab:tab1}-\ref{tab:tab3} establish a series of approximative estimates for $ \hat{\alpha}_a (\kappa)$ which itself is recognized as an algorithmic threshold ($\alpha_a (\kappa)$) candidate.

As stated above, we also believe that there are many other random problems where our propositions might produce results similar to those obtained here and in \cite{Stojniccluphop25,Stojnicalgbp25}. If true, this would
suggest that our results could likely be consequences of more universal principles common for a variety of random problems types. In fact, we actually believe that they might be a consequence of a generic parametric fl-RDT property. The following conjecture provides a possible summary of these developments.

\begin{conjecture}[Parametric fl-RDT algorithmic conjecture]
\label{thm:conj2} 
Assume the setup of Theorem \ref{thm:thmsflrdt1}. Define \underline{algorithmically} achievable $\psi_{rp,a}$ as
\begin{eqnarray}\label{eq:cjgen1}
\psi_{rp,a} \triangleq   \max   \left \{\psi_{rp} |\hspace{.05in}  \lim_{n\rightarrow \infty}\mP\left ( \mbox{$\psi_{rp}$ can be achieved in polynomial time}   \right ) =1 \right \}.
 \end{eqnarray}
Let $  \psi_{rd}^{(r)} (\cdot)\triangleq  \psi_{rd} (\cdot)$  with $\psi_{rd} (\cdot)$ as in  (\ref{eq:thmsflrdt2eq1a0}) and $\hat{\p}$, $\hat{\q}$, and $\hat{\c}$ as in Theorem \ref{thm:thmsflrdt1}. 
Then there is an $r$-sequence $  \psi_{rd}^{(r)} (\cdot)$  with \underline{decreasing} $\hat{\c}^>$ sequence such that
\begin{eqnarray}\label{eq:cjgen2}
\psi_{rp} = \lim_{r\rightarrow \infty}\psi_{rd}^{(r)}(\hat{\p},\hat{\q},\hat{\c}^>).
 \end{eqnarray}
\begin{itemize}
  \item If the sequence is such that there are no subintervals of $[0,1]$ with no elements of $\hat{\p}$ and $\hat{\q}$   then there is no computational gap, i.e.,
\begin{eqnarray}\label{eq:cjgen2}
\psi_{rp} = \lim_{r\rightarrow \infty}\psi_{rd}^{(r)}(\hat{\p},\hat{\q},\hat{\c}^>)=\psi_{rp,a} \quad \mbox{and}\quad
SCG =  \psi_{rp} -\psi_{rp,a}  =0 .
 \end{eqnarray}
  \item If the above does not hold then there is an $r$-sequence $  \psi_{rd}^{(r)} (\cdot)$  with \underline{arbitrarily ordered} (not necessarily decreasing) $\hat{\c}^{\sim}$ sequence such that
\begin{eqnarray}\label{eq:cjgen3}
\psi_{rd}^{(r)}(\hat{\p},\hat{\q},\hat{\c}^{\sim})\geq \psi_{rp,a} \quad \mbox{and} \quad  \lim_{r\rightarrow \infty}\psi_{rd}^{(r)}(\hat{\p},\hat{\q},\hat{\c}^{\sim}) = \hat{\psi}_{rd}=\psi_{rp,a} ,
 \end{eqnarray}
 and computational gap
\begin{eqnarray}\label{eq:cjgen3}
SCG = \psi_{rp} -\psi_{rp,a}  =   \lim_{r\rightarrow \infty}\psi_{rd}^{(r)}(\hat{\p},\hat{\q},\hat{\c}^{>}) 
- 
\lim_{r\rightarrow \infty}\psi_{rd}^{(r)}(\hat{\p},\hat{\q},\hat{\c}^{\sim}) .
 \end{eqnarray}

\end{itemize}
\end{conjecture}

%%%%%%%%%%%%%%%%%%%%%%%%%%%%%%%%%%%%%%%%%%%%%%%%%%%%%%%%%%%%%%%%%%%%%%%%%%%%%%%%
%%%%%%%%%%%%%%%%%%%%%%%%%%%%%%%%%%%%%%%%%%%%%%%%%%%%%%%%%%%%%%%%%%%%%%%%%%%%%%%%
%%%%%%%%%%%%%%%%%%%%%%%%%%%%%%%%%%%%%%%%%%%%%%%%%%%%%%%%%%%%%%%%%%%%%%%%%%%%%%%%
%%%%%%%%%%%%%%%%%%%%%%%%%%%%%%%%%%%%%%%%%%%%%%%%%%%%%%%%%%%%%%%%%%%%%%%%%%%%%%%%
%%%%%%%%%%%%%%%%%%%%%%%%%%%%%%%%%%%%%%%%%%%%%%%%%%%%%%%%%%%%%%%%%%%%%%%%%%%%%%%%
\section{Conclusion}
\label{sec:conc}
%%%%%%%%%%%%%%%%%%%%%%%%%%%%%%%%%%%%%%%%%%%%%%%%%%%%%%%%%%%%%%%%%%%%%%%%%%%%%%%%
%%%%%%%%%%%%%%%%%%%%%%%%%%%%%%%%%%%%%%%%%%%%%%%%%%%%%%%%%%%%%%%%%%%%%%%%%%%%%%%%
%%%%%%%%%%%%%%%%%%%%%%%%%%%%%%%%%%%%%%%%%%%%%%%%%%%%%%%%%%%%%%%%%%%%%%%%%%%%%%%%
%%%%%%%%%%%%%%%%%%%%%%%%%%%%%%%%%%%%%%%%%%%%%%%%%%%%%%%%%%%%%%%%%%%%%%%%%%%%%%%%
%%%%%%%%%%%%%%%%%%%%%%%%%%%%%%%%%%%%%%%%%%%%%%%%%%%%%%%%%%%%%%%%%%%%%%%%%%%%%%%%

We study SBP storage capacity via  a powerful mathematical engine called \emph{fully lifted random duality theory} (fl-RDT) \cite{Stojnicflrdt23}. Potential algorithmic implications analogous to those observed for ABP in \cite{Stojnicalgbp25} are uncovered. In particular,  a structural parametric change in  fl-RDT  is demonstrated as one progresses through lifting levels. A key parametric fl-RDT component, so-called $\c$ sequence, is shown to have natural (physical) decreasing ordering on the first two lifting levels. However, the phenomenology changes on higher levels and a perfect $\c$ sequence ordering disappears. The very same behavior is observed for ABP in \cite{Stojnicalgbp25} and connected to the change from satisfiability to algorithmic threshold. Following such a strategy here, we first  find that (for canonical $\kappa=1$ scenario) the second lifting level constraints density precisely matches the satisfiability threshold $\alpha_c\approx 1.8159$. The estimate then decreases as one progresses through higher lifting levels and on the seventh level reaches $\alpha\approx 1.6021$ (with predicted convergence in the interval $1.59-1.60$). This is then observed to closely agree with the local entropy clustering defragmentation estimate ($\approx 1.58$ ) obtained via replica methods in \cite{Bald20} and believed to be responsible for failure of fast locally improving algorithms.

In addition to canonical $\kappa=1$ scenario (which allows to discuss concrete numerical values), the results are also extended to a wide range of other $\kappa$ values. A particular attention is then devoted to the so-called low $\alpha,\kappa$ regime where strong progress has been made in recent years through studies of both OGP and local entropy. Our third lifting level results for $\alpha,\kappa\rightarrow 0$ regime are shown to qualitatively match the $\kappa\sim \sqrt{\frac{\alpha}{-\log\lp \alpha\rp}}$ OGP based estimate of \cite{GamKizPerXu22} and to precisely match the corresponding local entropy based estimate  $\kappa\approx 1.2385 \sqrt{\frac{\alpha}{-\log\lp \alpha\rp}}$ of \cite{BarbAKZ23}. Drawing parallel with algorithmic studies of ABPs and negative Hopfield models \cite{Stojnicclupsk25,Stojnicalgbp25}, a CLuP-SBP algorithm is designed and its performance is shown to closely approach theoretical predictions in a wide $\kappa$ range including both small and large $\kappa$ sub-regimes.

Obtained results suggest that the proposed parametric algorithmic phenomenology might be a generic fl-RDT feature and well worth of further exploration. In addition to identification of other optimization problems where the proposed methodologies might be applicable, we single out the follow directions as avenues where progress is preciously needed and likely to provide key insights regarding overall SCGs magic: (\textbf{\emph{i}}) Searching for potential physics related driving force that might be behind computational gaps phenomena;  (\textbf{\emph{ii}}) Developing  mathematically rigorous confirmation that would fully substantiate proposed concepts; (\textbf{\emph{iii}})  Designing algorithmic dynamics that would match the proposed fl-RDT parametrization and achieve the predicated algorithmic thresholds; and (\textbf{\emph{iv}}) Studying numerical aspects  (existence/uniqueness/ convergence, etc.) associated with stationary points equations and their universal algorithmic solvers.

%\newpage1
%\setcounter{page}{1}
\begin{singlespace}
\bibliographystyle{plain}
\bibliography{nflgscompyxRefs}
\end{singlespace}

\end{document}